\documentclass{article}

\usepackage[main, final]{neurips_2026}

\usepackage[utf8]{inputenc} 
\usepackage[T1]{fontenc}    
\usepackage{hyperref}       
\usepackage{url}            
\usepackage{booktabs}       
\usepackage{amsfonts}       
\usepackage{nicefrac}       
\usepackage{microtype}      
\usepackage{xcolor}         
\usepackage{graphicx}
\usepackage{subcaption}

\usepackage{amsmath}
\usepackage{amsthm}
\usepackage{thmtools, thm-restate}
\usepackage{cleveref}
\usepackage{enumitem}

\newtheorem{theorem}{Theorem}[section]
\newtheorem{definition}[theorem]{Definition}
\newtheorem{lemma}[theorem]{Lemma}
\newtheorem{fact}{Fact}[section]
\newtheorem{assumption}[theorem]{Assumption}
\newtheorem{corollary}[theorem]{Corollary}

\title{Reward Bias Substitution: Single-Axis Bias Mitigations Redirect Optimization Pressure}

\author{%
  Max Lamparth \\
  Stanford University \\
  \And
  Daniel Fein \\
  Stanford University \\
  \AND
  Andreas Haupt \\
  Stanford University \\
  \And
  Marcel Hussing \\
  University of Pennsylvania \\
  \And
  Mykel J. Kochenderfer \\
  Stanford University \\
}

\newcommand{\muaud}{\mu_{\mathrm{aud}}}

\begin{document}

\maketitle

\begin{abstract}
    Single-axis mitigations of reward-model biases, such as reducing proxy reliance on length, sycophancy, or style, can redirect optimization pressure onto correlated proxies rather than eliminate it, a new type of reward hacking we call reward bias substitution. 
    The failure is enabled by a measurement-versus-optimization gap between mitigation evaluation on the audit distribution and policy training on the policy-induced distribution.
    Prior reward-hacking analyses assume a fixed reward proxy, whereas we study the mitigation operator that removes one bias from it, the basic step of iterative reward refinement.
    We formalize mitigation outcomes into a regime taxonomy and prove that successful mitigation, bias substitution, and overcorrection can produce identical observables under every audit-distribution score, including ranking accuracy and win-rate, even when granted oracle access to the true reward.
    Across 30 published preference-learning mitigation works, no method we survey reports the evidence needed to certify successful mitigation.
    Augmenting evaluation with policy-induced distributions while tracking multiple biases provably resolves this distributional gap, and we translate this into actionable prescriptions for mitigation methods and benchmarks. 
    We demonstrate bias substitution in RLHF, where a length penalty in GRPO shortens responses as intended yet makes the policy overconfident and less accurate on open-ended questions, show that the same distributional gap hides failures of published mitigations, and find the necessary preconditions for reward bias substitution documented in published results and in our study of length-sycophancy coupling in LLM responses.
\end{abstract}





\section{Introduction}
\label{sec:intro}

Reinforcement Learning from Human Feedback (RLHF) and related preference-learning methods use human annotations to shape model behavior on non-verifiable objectives, but the resulting reward models and policies commonly learn spurious correlations.
A recurring example is reward-length correlation in learned rewards and verbosity in the resulting policies, across both RLHF and DPO \citep{singhal2024long, park2024disentangling}.
Many mitigations target such features directly, e.g., by separating length from quality
in reward models 
\citep{shen2023loose, chenicml2024odin, bu2025beyond} or controlling for length in evaluations \citep{dubois2024length, liblog2024style}.
Such biases often surface after training \citep{perez2023discovering},
so iteratively mitigating individual or subsets of biases with RLHF or DPO is 
common, and we argue unavoidable for learned rewards on non-verifiable objectives.
However, \emph{removing a feature from the reward does not remove the pressure to exploit it.}
Spurious features are partially informative about quality and correlated with one another, so suppressing one redirects optimization onto the rest, similar to shortcuts in supervised learning \citep{licvpr2023whac}.
\Cref{fig:smoking_gun} illustrates the symptom for language model RLHF. 
A length penalty during GRPO \citep{shao2024deepseekmathpushinglimitsmathematical} training compresses responses as intended, yet the freed optimization pressure is redirected onto an untargeted feature axis, driving the policy into overconfidence and degrading free-form factual accuracy while multiple-choice accuracy holds (see \Cref{app:gun}). 
This phenomenon is not specific to length, nor RLHF, but a general type of reward hacking that we call \textbf{reward bias substitution}, in which a single-axis mitigation applied to a reward model or the implicit rewards of a policy relocates optimization pressure onto correlated proxies (see  \Cref{fig:r1_schematic}).

\begin{figure}
  \centering
  \includegraphics[width=0.99\linewidth]{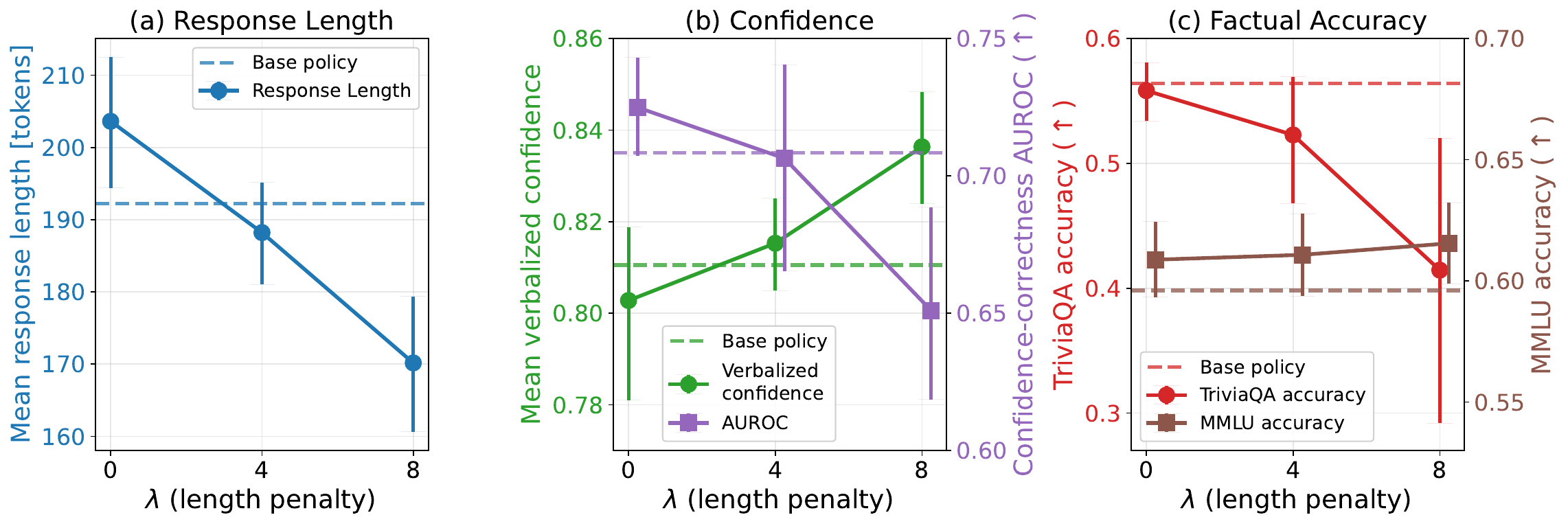}
  \caption{
  Language model example of reward bias substitution. A length penalty shortens responses, but optimization pressure is redirected onto verbalized confidence, while free-form accuracy (TriviaQA) drops and MMLU is unchanged. RLHF of \textit{Llama-3.2-3B-Instruct} with \textit{Skywork-Reward-V2-Llama-3.1-8B} using GRPO ($\beta = 2\mathrm{e}{-2}$, LR $= 3\mathrm{e}{-5}$, 600 steps) on \textit{UltraFeedback} with four random seeds for each $\lambda \in \{0, 4, 8\}$, 95\% bootstrapped confidence intervals. To mitigate length, we modify the reward $\tilde R$ in the RLHF objective as $ \tilde R(x, y) = R_\text{RM}(x, y) - \lambda\, n_\text{tok} / 100$. Confidence and accuracy stay at base level at $\lambda = 0$, so the penalty, not RLHF, causes the shift. See \Cref{app:gun}.
  }
  \label{fig:smoking_gun}
\end{figure}
Crucially, \emph{reward bias substitution is invisible} to how mitigations are evaluated, a \emph{measurement-versus-optimization gap} between audit and policy-induced distributions.
On the audit distribution, a successful mitigation, a relocation of the bias onto another feature, and an overcorrection lowering true reward can all produce the same observable scores.
Across the published preference-learning mitigation methods we survey, none report the evidence needed to certify that they removed a bias rather than relocated or overcorrected it, and the resulting substitution is occasionally miscategorized as an isolated anomaly, a capability trade-off, or judge noise (\Cref{sec:4_evidence}, \Cref{app:regimes}).
Existing reward-hacking definitions \citep{skalseneurips2022defining, laidlaw2025correlated} 
are stated for a given proxy and not for an operator editing it, so they cannot separate removing a bias from relocating it.
Left unaddressed, bias substitution risks silently trading one bias for another in deployment.

We prove this blindspot is structural, as no audit-distribution score can tell successful mitigation, bias substitution, and overcorrection apart, no matter how the benchmark is enriched and even when granted oracle access to the true reward. 
Similarly, multi-axis mitigation cannot help while still validated away from the policy-induced distribution, and causal identification would demand assumptions that preference data cannot grant. 
However, we prove that evaluating on the policy-induced distribution while tracking off-target features is sufficient to separate the three outcomes for explicit-reward RLHF and direct preference optimization.
In practice, we recommend validating a reward bias mitigation against at least two reference policies under best-of-N or policy optimization, reporting the induced drift in length, confidence, and sycophancy alongside the usual audit score and correctness.
\textbf{This paper contributes the following:}
\begin{itemize}
    \item \textbf{We formalize reward bias substitution}, with the measurement-versus-optimization gap as the structural mechanism making it invisible to prior reward-hacking definitions.
    \item \textbf{We introduce a regime taxonomy} classifying all distinct single-axis mitigation outcomes: successful mitigation (R0, with a contaminated sub-case R0\textsubscript{cont}), bias substitution (R1), overcorrection (R2), silent non-op (R3), and audit-distribution sensitivity (R4).
    \item \textbf{We prove a matched impossibility-sufficiency pair}. No benchmark functional over audit-distribution observables can separate successful mitigation from substitution or overcorrection, while augmenting with policy-induced distributions provably can, and we turn the sufficient condition into prescriptions for methods and benchmarks.
    \item \textbf{We provide empirical evidence for our framework.} 
    We validate substitution end-to-end in GRPO RLHF (length traded for overconfidence, R1), the measurement-versus-optimization gap on two published length-debiasing operators across five reward models,  
    and we measure length-sycophancy coupling across eight LLMs.
    We further classify 30 published mitigation works and find that none reports the evidence needed to certify R0.
\end{itemize}

Our code for the experiments is available at \href{https://github.com/maxlampe/bias\_substitution}{github.com/maxlampe/bias\_substitution} (MIT license).

\begin{figure}
    \centering
    \includegraphics[width=0.99\linewidth]{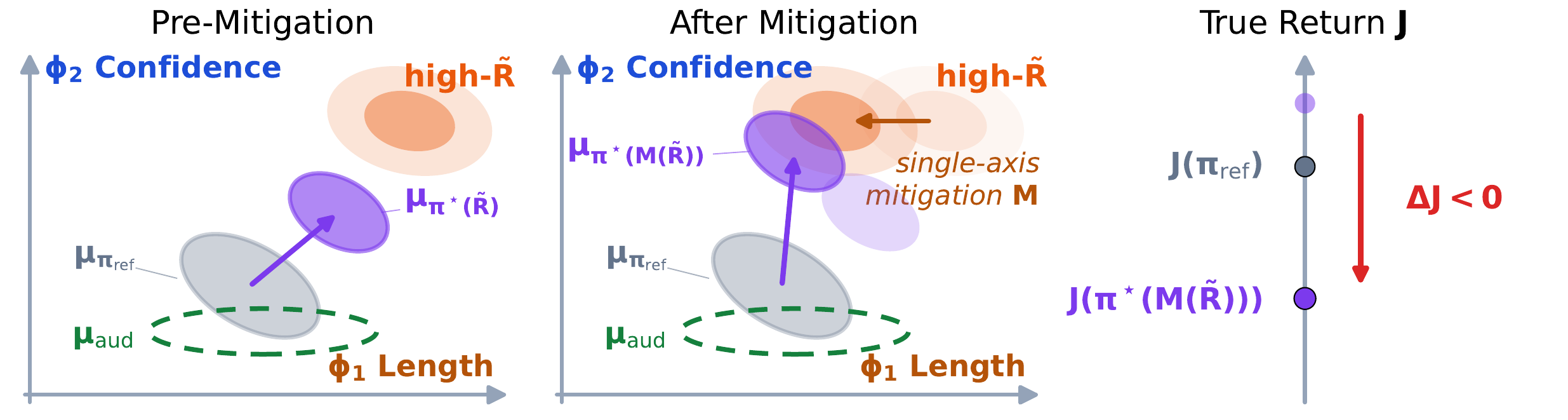}
    \caption{Reward bias substitution mechanism (schematic). 
    The correlated spurious axes are $\phi_1$ (length) and $\phi_2$ (confidence). 
    Before mitigation (left), the optimized policy induces $\mu_{\pi^\star(\tilde R)}$ in a high proxy reward region. 
    A single-axis length mitigation $M$ (middle) cuts length reliance and redirects optimization pressure onto confidence. 
    The rerouting leaves no trace at $\muaud$ and appears only at $\mu_{\pi^\star(M(\tilde R))}$. 
    For \Cref{fig:smoking_gun}, true return $J$ falls even though the length bias looks removed (right). 
    This gap is why audit-distribution scores alone cannot certify a real mitigation, formalized in \Cref{thm:audit-insufficiency}.
    }
    \label{fig:r1_schematic}
\end{figure}

\section{Two Invariances Enable Reward Bias Substitution}
\label{sec:background}

We work in the single-turn contextual-bandit reduction of RLHF \citep{christiano2017deep, stiennon2020learning, rafailov2023direct, kaufmann2025survey} with prompts $x \in \mathcal{X}$ drawn i.i.d. from a fixed context distribution $\mathcal{D}$, responses $y \in \mathcal{Y}$ sampled $y \sim \pi(\cdot \mid x)$, true reward $R(x,y)$, learned proxy $\tilde R(x,y)$, and the \textit{plain return} $J(\pi, R) = \mathbb{E}_{x \sim \mathcal{D},\, y \sim \pi(\cdot \mid x)}[R(x,y)]$ as the quantity we want to improve. 
The KL-regularized RLHF training objective is
\begin{equation} 
J_{\text{RLHF}}(\pi; \tilde R) = \mathbb{E}_{x \sim \mathcal{D}}\!\left[\mathbb{E}_{y \sim \pi(\cdot\mid x)}[\tilde R(x, y)] - \beta\, D_{\mathrm{KL}}\!\big(\pi(\cdot\mid x)\,\|\,\pi_{\mathrm{ref}}(\cdot\mid x)\big)\right] .
\label{equ:j_rlhf}
\end{equation}
For any reward $\bar R$ the KL-regularized optimum is 
\citep{peters2007reinforcement, rafailov2023direct}
\begin{align}
    \pi^\star_\beta(\bar R)(y \mid x) \propto \pi_\mathrm{ref}(y \mid x) \exp\!\big(\bar R(x,y)/\beta\big).
    \label{equ:softmaxsolution}
\end{align}

Let $\mu_\pi(x,y) = \mathcal{D}(x)\pi(y \mid x)$ denote the policy-induced distribution on $\mathcal{X} \times \mathcal{Y}$ given a policy $\pi$ (the one-step occupancy measure under $\mathcal{D}$). 
In addition to the $\pi_\mathrm{ref}$-defined measure $\mu_{\pi_\mathrm{ref}}$, 
we fix an \emph{audit distribution} $\muaud$ on $\mathcal{X} \times \mathcal{Y}$, a probability measure used for reliance estimation and correlation measurement.
Natural choices include $\mu_{\pi_\mathrm{ref}}$ itself (coupling measurement and optimization)
and annotator-conditioned audit distributions $\muaud^\text{human}, \muaud^\text{LLM}$, reflecting that reliance estimates depend on who labels (see also \Cref{app:A2} for human-versus-LLM-judge gap).
We invoke a standing regularity Assumption~\ref{asm:reg} throughout for well-defined softmax policies and policy-level expectations. 

Reward modeling in RLHF involves invariances at two levels. 
First, what preference data identifies about the reward (data-level), and second, how the KL-regularized policy responds to reward transformations (policy-level). 
At the \emph{data level}, the reward identifiability problem leaves the learned reward underdetermined \citep{kimicml2021reward,cao2021identifiability, skalseicml2023invariance, tien2023causal, casper2023open}, independent of sample size \citep{siththaranjan2024distributional}.  
Many reward functions explain the same preference data and comparisons identify the true reward only up to a prompt-only additive shift under the Bradley-Terry likelihood \citep{bradleyterrt1952pairwisecomparisons, wu2023pairwise, hejna2023inverse,ethayarajh2024kto, knox2024models}.
Under this underdetermination, a spurious feature can be indistinguishable from a structurally relevant one.

At the \emph{policy level}, KL-regularized optimization is sensitive to cardinal reward values, since $\tilde R \rightarrow c \tilde R$ at fixed $\beta$ is equivalent to $\beta \rightarrow \beta / c$ \citep{haarnoja2018softactorcriticoffpolicymaximum}, so two reward models agreeing on the ranking of every response pair can still induce different policies \citep{skalseicml2023invariance}. 
In our one-step setting, potential-based shaping \citep{ng1999policy} is exactly this
prompt-only shift, which single-axis mitigations break.
Combined with data-level identifiability, this means \emph{confounded features can be consistent with the observed rankings while still distorting the cardinal values that determine the optimized policy}.
A single-axis mitigation acts on exactly such a feature, so removing a bias, relocating it onto a correlated feature, and overcorrecting can all present the same audit-side evidence (\Cref{fig:r1_schematic}). 
Whether any audit-distribution benchmark can separate them is a claim about all benchmarks at once, not something one experiment can settle, so we prove it in \Cref{sec:bias_sub}.

\section{Formalizing Reward Bias Substitution}
\label{sec:bias_sub}

We turn these invariances into a taxonomy of single-axis mitigation outcomes, characterizing what audit-distribution evaluation can and cannot certify, and deriving prescriptions for mitigation methods and benchmarks.
We state our framework for any scalar proxy $\tilde R$ satisfying Assumption~\ref{asm:reg}, including Bradley-Terry, generative, rubric-scored, and verifiable. See \Cref{app:dpo} for implicit rewards.

\subsection{Feature Map and Spurious vs. Structurally Relevant Features}
\label{sec:3_1_features}

In this work, we are interested in categorizing and analyzing the effects of bias mitigation strategies given a multiplicity of axes that we may care to optimize for.
We will refer to these axes as (surface) features.
So, before we can disentangle their interactions with rewards, we need to define features.
\begin{definition}[Feature map]
    \label{def:feature_map}
    A \emph{feature map} is a measurable function $\phi : \mathcal{X} \times \mathcal{Y} \to \mathbb{R}$ capturing an interpretable surface attribute of a response. 
    To form a set of features that we care about, we fix a finite ordered tuple of feature maps $\phi_1, \ldots, \phi_K$ and collect them into the vector-valued map $\Phi = (\phi_1, \ldots, \phi_K)^\top : \mathcal{X} \times \mathcal{Y} \to \mathbb{R}^K$.
\end{definition}
We use $\Phi$ both as the column-vector-valued map $(\phi_1, \ldots, \phi_K)^\top$ 
and as the set $\{\phi_1, \ldots, \phi_K\}$ depending on when membership ($\phi_i \in \Phi$) or partitions ($\Phi = \Phi_{\mathrm{sp}} \sqcup \Phi_{\mathrm{struct}}$) are more natural.

To instantiate our feature sets, we will make the natural assumption that the chosen feature sets capture distinct response attributes under the audit distribution. 
For example, features such as length, sycophancy, politeness, or hedging may be correlated, but they should not collapse into the same measurement across responses in the audit distribution.
This rules out duplicate or exactly redundant features and ensures that each feature represents a separately identifiable axis of response behavior.
\begin{assumption}[Non-degeneracy]\label{asm:nondeg}
The Gram matrix $G := \mathbb{E}_{\muaud}[\Phi\Phi^\top] \in \mathbb{R}^{K \times K}$, with entries $G_{ij} = \mathbb{E}_{\muaud}[\phi_i \phi_j]$, is positive definite.
\end{assumption}

With these distinct axes in place, we can ask which attributes the reward actually depends on.

\begin{definition}[Spurious vs. structurally relevant at $\muaud$]
    \label{def:spur_v_struct}
    Treat $R$ as a function on $\mathcal{X} \times \mathcal{Y}$ and assume feature realizability (Assumption~\ref{asm:featreal}).
    Feature $\phi_i$ is \emph{spurious with respect to $R$ at $\muaud$} if
    \begin{equation}
        \mathbb{E}_{\muaud}\!\bigl[R(x,y) \,\big|\, x,\, (\phi_j(x,y))_{j \neq i} \bigr] \text{ does not depend on } \phi_i(x,y) \text{ for } \muaud\text{-a.e.\ }
    \end{equation}
    Otherwise $\phi_i$ is \emph{structurally relevant at $\muaud$}.
    Let $\Phi_\mathrm{sp} = \{\phi_i \in \Phi : \phi_i \text{ spurious w.r.t.}\ R\ \text{at}\ \muaud\}$ and $\Phi_\mathrm{struct} = \Phi \setminus \Phi_\mathrm{sp}$ for the induced partition of $\Phi$.
\end{definition}

See \Cref{app:local-spuriousness} for the equivalent $o(\varepsilon)$ formulation under $\varepsilon$-mixture perturbations, capturing the intuition that $R$ has no first-order dependence on $\phi_i$ at $\muaud$.
The corresponding causal reading and partial-identifiability relationship are given in \Cref{app:causal-scope}.
The underlying causal partition is not fully identifiable from preference data alone 
and \Cref{def:spur_v_struct} is its conservative observational projection.
Under partial informativeness, natural features like length act as mediators of multiple mechanisms.
The binary partition of $\Phi$ from \Cref{def:spur_v_struct} may therefore conservatively classify them as $\Phi_\mathrm{struct}$ when they correlate with $R$ under $\muaud$ (\Cref{app:A2}).
This partition of $\Phi$ also matches existing single-axis mitigations acting on whole features rather than within-feature decompositions \citep[e.g.][]{fein2026one, papadatos2024linear, chenicml2024odin, huang2025posthoc}.
Finer partitions require stronger assumptions (interventions, gold labels, distributional invariance) outside our regime.

\subsection{Single-Axis Mitigation and Measurement-Versus-Optimization Gap}
\label{sec:3_2_linearfix}

Several prominent reward-model mitigations identify a feature direction, estimate the reliance of $\tilde R$ on it, and subtract its contribution as, e.g., linear probes \citep{fein2026one, papadatos2024linear},  explicit verbosity decorrelation
\citep{wang2024interpretablepreferences}, non-linear calibration \citep{huang2025posthoc}, or architectural disentanglement \citep{chenicml2024odin}. 
To abstract away from these implementation details, we first define a population-level notion of reliance, given by the coefficients of the best linear approximation of $\tilde R$ in $\Phi$ under the audit distribution.

\begin{definition}[Linear reliance]
\label{def:linear_reliance}
The \emph{linear reliance} of $\tilde{R}$ on $\Phi$ at $\muaud$ is
\begin{equation}
    g(\tilde{R};\, \muaud) = \big(\mathbb{E}_{\muaud}[\Phi \Phi^\top]\big)^{-1} \mathbb{E}_{\muaud}[\Phi\, \tilde{R}] \in \mathbb{R}^K,
\end{equation}
yielding the $L^2(\muaud)$-orthogonal decomposition $\tilde{R} = \sum_i g_i \phi_i + \tilde{R}_\perp$ with $\mathbb{E}_{\muaud}[\phi_i \tilde{R}_\perp] = 0$. 
\end{definition}

Note $g$ is a \emph{statistic} of the triple $(\tilde R, \Phi, \muaud)$, not a property of $\tilde R$ alone and evaluating at $\mu_{\pi^\star} \neq \muaud$ yields a different vector, driving later regime distinctions and enabling reward bias substitution.

\begin{definition}[Single-axis mitigation]
\label{def:single_axis_mitigation}
A \emph{single-axis mitigation targeting feature $i$} is an operator $M_i : \tilde{R} \mapsto \tilde{R}'$ with $|g_i(\tilde{R}';\, \muaud)| < |g_i(\tilde{R};\, \muaud)|$. The canonical projection instance is
\begin{equation}
    M_i(\tilde{R})(x,y) \;=\; \tilde{R}(x,y) - g_i(\tilde{R};\, \muaud)\, \phi_i(x,y).
\end{equation}
\end{definition}

We single out the projection instance above as canonical, because it zeros $g_i$ at $\muaud$ exactly.
The canonical $M_i$ depends on $\muaud$ through $g_i(\tilde R; \muaud)$. 
We write $M_i^{\muaud}$ when this dependence matters.

\begin{lemma}[Single-axis identity at $\muaud$]
\label{lem:single_axis_identity}
For the canonical $M_i$, $g_i(M_i(\tilde{R});\, \muaud) = 0$ and $g_j(M_i(\tilde{R});\, \muaud) = g_j(\tilde{R};\, \muaud)$ for $j \neq i$.
\end{lemma}
\begin{proof}
Direct computation using $\mathbb{E}_{\muaud}[\Phi\phi_i]$ as the $i$-th column of $\mathbb{E}_{\muaud}[\Phi\Phi^\top]$.
\end{proof}

The identity justifies calling $M_i$ ``single-axis'', as at $\muaud$, mitigation moves $\tilde R$ along the $i$-th coordinate of $g$-space.\footnote{The canonical $M_i$ can also be seen as the population Frisch-Waugh-Lovell partial-out of $\phi_i$ from $\tilde R$ at $\muaud$.} 
Note that $M_i$ is an associational operator, as it removes $\phi_i$-reliance at $\muaud$ in the projection sense, not the causal contribution of $\phi_i$ to $R$ (see Appendix~\ref{app:causal-scope}).
Our taxonomy classifies the resulting outcomes by what the optimizer does at $\mu_{\pi^\star}$, not by what is observable at $\muaud$.
 
\paragraph{Measurement-versus-optimization gap.}
Single-axis diagnostics evaluate $M_i$ at $\muaud$, where $|g_i| = 0$ by construction (\Cref{lem:single_axis_identity}).
The optimizing policy realizes the effects of $M_i$ at $\mu_{\pi^\star}$, where in general $g_i(M_i(\tilde R); \mu_{\pi^\star}) \neq 0$ and $g_j(M_i(\tilde R); \mu_{\pi^\star}) \neq g_j(\tilde R; \mu_{\pi^\star})$ for $j \neq i$.
Mitigation moves the proxy along an axis defined at the audit distribution, but optimization responds to a vector defined at a different distribution (see \Cref{fig:r1_schematic}).
\emph{This gap is the structural mechanism enabling bias substitution and is invisible to any diagnostic that operates at $\muaud$ alone whenever $\muaud \neq \mu_{\pi_{\mathrm{ref}}}$}, as is standard for reward model benchmarks (see also \Cref{app:audit-insufficiency}).
Versions of this gap are known for a fixed proxy or learned model \citep{ross2011reductionofimitationlearning,gleave2021quantifying,skalse2024starc,fluri25a,wen2025rethinking,rame2024warm,chaudhari2025rlhfdeciphered}, see \Cref{sec:related_work}. 
Here, the gap arises from the mitigation operator $M_i$ estimated at $\muaud$, so an audit of the targeted feature passes by construction.
\Cref{sec:impossibility-sufficiency} shows the gap cannot be closed by any audit-distribution-only evaluation and a closed-form instance is given in \Cref{app:nonsuck,app:phasediagrams}.

\paragraph{Gauge invariance.}
The linear reliance $g$ is not invariant under prompt-only reward shifts $\tilde R(x,y) \mapsto \tilde R(x,y) + b(x)$, while the KL-regularized optimum is, which matters for PPO-style reward whitening \citep{lambertrlhf2025reinforcement}.
\Cref{app:A4} gives a gauge-invariant $g_\mathrm{cent}$ and verifies that our taxonomy transfers. 
\textbf{Scale invariance.}
Applying $M_i$ also changes $\|\tilde R\|_{L^2(\muaud)}$, and KL-regularized policies are not invariant to reward scale (see \Cref{equ:softmaxsolution}), so this scale change interacts with optimization non-trivially.
\Cref{app:A3} gives the corrected scale identity and a scale-invariant variant $M_i^\mathrm{norm}$.

\subsection{Induced Feature Drift and Reward Bias Substitution}
\label{sec:3_3_bias_substitution}

Under optimization, the measurement-versus-optimization gap produces several distinct failure modes, which we classify along two axes, the change in true reward at the optimized policy and whether optimization pressure redirects onto another spurious feature.
Throughout, fix $\beta > 0$, let $\tilde R$ be a proxy reward, let $\phi_i \in \Phi_\mathrm{sp}$ be the 
targeted feature, and let $M_i$ be a single-axis mitigation (Definition~\ref{def:single_axis_mitigation}) 
targeting $\phi_i$, with $\tilde R $ and $M_i(\tilde R)$ satisfying the regularity Assumption~\ref{asm:reg}. 

\begin{definition}[Regime Taxonomy]
\label{def:all_regimes}
Let $\pi = \pi^\star_\beta(\tilde R)$ and $\pi' = \pi^\star_\beta(M_i(\tilde R))$ be the pre- and post-mitigation KL-regularized optima, and write $\Delta_j = \mathbb{E}_{\mu_{\pi'}}[\phi_j] - \mathbb{E}_{\mu_{\pi}}[\phi_j]$ for the induced drift in feature $\phi_j$ and $\Delta J = J(\pi', R) - J(\pi, R)$ for the change in true reward. 
Then $M_i$ is in regime \textbf{R0}, \textbf{R0}\textsubscript{cont}, \textbf{R1}, \textbf{R2}, or \textbf{R3} according to the redirection predicate on $\Phi_\mathrm{sp} \setminus \{\phi_i\}$ and the sign of $\Delta J$,
\begin{center}
\small
\begin{tabular}{lll}
\toprule
 & No redirection & Redirection \\
 True reward change & ($\Delta_j = 0$ on $\Phi_\mathrm{sp} \setminus \{\phi_i\}$)
 & ($\Delta_j \neq 0$ for some $\phi_j \in \Phi_\mathrm{sp} \setminus \{\phi_i\}$) \\
\midrule
$\Delta J > 0$ & R0 Successful mitigation & R0\textsubscript{cont} Contaminated success \\
$\Delta J = 0$ & R3 Silent non-op & R1 Bias substitution (neutral) \\
$\Delta J < 0$ & R2 Overcorrection & R1 Bias substitution (harmful) \\
\bottomrule
\end{tabular}
\end{center}
where R1 spans both redirection cells with $\Delta J \le 0$, neutral when $\Delta J = 0$ and harmful when $\Delta J < 0$.
\end{definition}

In the GRPO example in \Cref{fig:smoking_gun} and \Cref{app:gun}, penalizing length (the targeted $\phi_i$) redirects pressure onto expressed confidence while factual accuracy drops, placing the mitigation in harmful bias substitution (R1).
\Cref{app:A2} verifies the taxonomy is exhaustive and \Cref{app:epsilonbands} gives the $\varepsilon$-banded versions of $\Delta_j$ and $\Delta J$ for finite-sample studies 
which we use in \Cref{sec:4_evidence} to classify published methods wherever the cited work reports uncertainty, and sign-only otherwise.
The regimes classify mitigations by quantities depending on $R$ and $\Phi_\mathrm{sp}$, neither fully identifiable from preference data alone, and the redirection predicate references only $\Phi_\mathrm{sp}$ because redirection onto structurally relevant features is not penalized.

\paragraph{Remarks.} 
R0\textsubscript{cont} is the outcome audit-distribution evaluation most easily mistakes for R0, where the redirection is present but its cost is outweighed by the gain from reducing $\phi_i$, so evaluation registers the improvement but not the redirection. 
R1 is the regime the measurement-versus-optimization gap specifically enables, because $|g_i(M_i(\tilde R); \muaud)| = 0$ holds at the audit distribution by construction (Lemma~\ref{lem:single_axis_identity}) while $\Delta J \le 0$ at the optimized policy, so audit diagnostics targeting $\phi_i$ register apparent success regardless of the sign of $\Delta J$. 
R2 has two origins: scale overshoot, fixable by rescaling, and target misspecification, which removes a genuinely informative component of a partially informative feature and generically is not \citep[see also][]{tsipras2018robustness}.
\Cref{app:a8-takeaways} distinguishes them via $\Delta J(c)$ across partial mitigations $M^c_i$. 
R3 is silent because the mitigation alters the proxy in ways the audit distribution may register while the optimized policy does not express them, the generic outcome when the optimization footprint of the targeted feature was already small, leaving $D_\mathrm{KL}(\pi'\,\|\,\pi)$ small at the relevant $\beta$.

The taxonomy in \Cref{def:all_regimes} fixes a single audit distribution $\muaud$, but cannot cover failure modes where the regime label changes when $\muaud$ does, 
e.g., across human and LLM-judge audits or across prompt topics \citep{wang2026automatically}.
Thus:

\begin{definition}[R4 Audit-distribution sensitivity]
\label{def:R4}
A mitigation construction $M_i$ exhibits \emph{audit-distribution sensitivity} if, 
holding $(\tilde R, R, \Phi_\mathrm{sp}, \beta, \pi_{\mathrm{ref}})$ fixed, there exist 
$\muaud^{(1)} \neq \muaud^{(2)}$ such that
\begin{equation*}
    \pi^\star_\beta\bigl(M_i^{(1)}(\tilde R)\bigr) 
    \quad \text{and} \quad
    \pi^\star_\beta\bigl(M_i^{(2)}(\tilde R)\bigr)
    \quad \text{fall into different R0-R3 regimes, where} ~ M_i^{(\ell)} := M_i^{\muaud^{(\ell)}}.
\end{equation*}
\end{definition}
 
R4 is regime-labeled (rather than treated as a property of the deployment setting, like $\pi_{\mathrm{ref}}$-sensitivity) because $\muaud$ is a designer-controlled input to the mitigation\footnote{$\muaud$ is non-performative in the sense of~\citet{perdomo20a}.}
and formalizes the observation that failure-mode classification is 
a property of $(R, \tilde R, M_i, \muaud)$.
See \Cref{fig:headline_and_r4} for an example transition.

\paragraph{Instantiations.} 
Definitions~\ref{def:all_regimes} and~\ref{def:R4} cover the mechanistic outcome space,
with closed-form instantiations of all regimes given in the linear-Gaussian and quadratic-nonlinear settings of \Cref{app:nonsuck,app:phasediagrams} with visualizations. 
The R0-R3 conditions reference only $\Delta_j$ and $\Delta J$, both invariant under prompt-only shifts. 
The canonical $M_i$ is not gauge-invariant, but $M_i^{\mathrm{cent}}$ (\Cref{app:A4}) restores classification invariance.
In \Cref{sec:4_evidence}, we map published mitigations onto these regimes.

\subsection{Impossibility and Sufficiency for Audit-Distribution Evaluation}
\label{sec:impossibility-sufficiency}

The measurement-versus-optimization gap raises the question whether any benchmark functional defined at the audit distribution can distinguish successful mitigation (R0) from contaminated success (R0\textsubscript{cont}), substitution (R1), or overcorrection (R2).
We prove that they cannot and show that augmenting evaluation with policy-induced distributions suffices (see also \Cref{app:audit-insufficiency,app:audit-sufficiency}).

\paragraph{Benchmark class.} 
Let $\mathcal{B}$ denote the class of benchmark functionals depending only on the joint distribution of $(\tilde{R}, M_i(\tilde{R}), R, \Phi)$ under $\muaud$, with $\muaud$ taken as the empirical evaluation measure. 
This class subsumes ranking accuracy, pairwise win-rate, style-controlled accuracy, preference-prediction calibration, reward-target correlation, the linear-reliance statistic $g$ of Definition~\ref{def:linear_reliance},
reward-distance metrics computed at $\muaud$ \citep{gleave2021quantifying, skalse2024starc}, and any composites.
RewardBench~\citep{lambertnaacl2025rewardbench}, RewardBench2 (headline accuracy)~\citep{malik2026rewardbench2}, and RM-Bench (style-controlled accuracy)~\citep{liu2025rmbench} all lie in $\mathcal{B}$.
Policy-level evaluations such as length-controlled AlpacaEval~\citep{dubois2024length} and Chatbot Arena~\citep{chiang2024chatbot} lie outside $\mathcal{B}$.
While providing a proxy for $\Delta J$, none of them reports off-target $\Delta_j$, which is needed to separate R0 from R0\textsubscript{cont} and R1 from R2.
Including $R$ as an input strengthens our impossibility result, since $R$ appears identically across the four instances and standard benchmarks lack oracle access to it.

\begin{restatable}[Audit-distribution insufficiency]{theorem}{auditinsufficiency}
\label{thm:audit-insufficiency}
Fix $\beta > 0$. 
There exist four choices of $\pi_\mathrm{ref}$ with $\muaud$, $\tilde R$, $M_i$, $R$, $\Phi$, $\Phi_\mathrm{sp}$, and $\beta$ held fixed s.t.
\begin{enumerate}[label=\textup{(\roman*)}]
    \item the joint distribution of $(\tilde R, M_i(\tilde R), R, \Phi)$ under $\muaud$ is identical across the four instances, so $B$ takes the same value on all four for every $B \in \mathcal{B}$,
    \item the optimized policies $\pi^\star_\beta(M_i(\tilde R))$ fall into regimes R0, R0\textsubscript{cont}, R1, and R2 respectively, in the sense of Definition~\ref{def:all_regimes}.
\end{enumerate}
\end{restatable}

The proof (\Cref{app:audit-insufficiency}) constructs four reference
policies in a linear-Gaussian setting with fixed Gram, so audit-side
observables coincide while policy-side first moments place the optimized policies in distinct regimes.
Since only $\pi_{\mathrm{ref}}$ varies, the same reward model can land in any of R0, R0\textsubscript{cont}, R1, R2 depending on the reference policy it is deployed against, 
consistent with the observations of \citet{malik2026rewardbench2} that high audit scores can fail to transfer to PPO under RM-policy lineage mismatch. 
Even with $\pi_{\mathrm{ref}}$ fixed, audit scores fail to predict policy outcomes in \citet{wen2025rethinking, shen2023trickledownimpactrewardinconsistency} and the reward model scoring highest on RewardBench trains the weakest PPO policy in \citet{ivison2024unpacking} (also see \Cref{app:a8-takeaways}).

Theorem~\ref{thm:audit-insufficiency} establishes the \emph{distributional blindspot}.
A second, the \emph{functional blindspot}, arises from a distinct mechanism and persists even when $\pi_{\mathrm{ref}}$ is held fixed, since $\mathcal{B}_{\text{ord}} \subseteq \mathcal{B}$ (functionals invariant under
monotone transformations of $\tilde{R}$) is blind to cardinal-scale shifts that KL-regularized optimization tracks (Corollary~\ref{cor:functional-blindspot}).
Since mitigation generically induces rescaling (\Cref{app:A3}), ordinal-only evaluation cannot rule out scale-driven regime shifts even when the proxy is unchanged in ordinal content.

We show that augmenting the benchmark input with the policy-induced distributions provably suffices.
Let $\mathcal{B}^+$ denote the class of benchmark functionals depending on the joint distribution of $(\tilde{R}, M_i(\tilde{R}), R, \Phi)$ under each of $\muaud, \mu_{\pi^\star_\beta(\tilde{R})}, \mu_{\pi^\star_\beta(M_i(\tilde{R}))}$.

\begin{restatable}[Audit-distribution sufficiency]{theorem}{auditsufficiency}
\label{thm:audit-sufficiency}
Fix $\beta > 0$ and let $\pi = \pi^\star_\beta(\tilde R)$, $\pi' = \pi^\star_\beta(M_i(\tilde R))$. 
With $\Delta_j$ and $\Delta J$ as in Definition~\ref{def:all_regimes} and 
\Cref{sec:3_3_bias_substitution}, define $B^\star \in \mathcal{B}^+$ by
\begin{equation}
\label{equ:bstar}
    B^\star = \begin{cases}
        \mathrm{R0}        & \text{if } \Delta J > 0 \text{ and } \Delta_j = 0 \text{ for all } \phi_j \in \Phi_\mathrm{sp} \setminus \{\phi_i\}, \\
        \mathrm{R0}_{\mathrm{cont}} & \text{if } \Delta J > 0 \text{ and } \Delta_j \neq 0 \text{ for some } \phi_j \in \Phi_\mathrm{sp} \setminus \{\phi_i\}, \\
        \mathrm{R1}        & \text{if } \Delta J \leq 0 \text{ and } \Delta_j \neq 0 \text{ for some } \phi_j \in \Phi_\mathrm{sp} \setminus \{\phi_i\}, \\
        \mathrm{R2}        & \text{if } \Delta J < 0 \text{ and } \Delta_j = 0 \text{ for all } \phi_j \in \Phi_\mathrm{sp} \setminus \{\phi_i\}.
    \end{cases}
\end{equation}
Under Assumptions~\ref{asm:reg}, \ref{asm:nondeg}, and \ref{asm:featreal}, with $M_i(\tilde R)$ likewise satisfying Assumption~\ref{asm:reg}, for every tuple $(\pi_\mathrm{ref}, \tilde R, M_i, R, \Phi, \Phi_\mathrm{sp})$ with $(\pi, \pi') \in \mathrm{R0} \cup \mathrm{R0}_{\mathrm{cont}} \cup \mathrm{R1} \cup \mathrm{R2}$ in the sense of Definition~\ref{def:all_regimes}, $B^\star$ returns the correct regime label.
\end{restatable}

Proof (\Cref{app:audit-sufficiency}): $\Delta_j$ and $\Delta J$ are first moments under $\mu_\pi$ and $\mu_{\pi'}$, both in $\mathcal{B}^+$. 
The four branches are mutually exclusive by Definition~\ref{def:all_regimes}, where the sign of $\Delta J$ separates R0 and R0\textsubscript{cont} from R1 and R2, and the redirection condition separates within each pair. 
A fifth branch ($\Delta J = 0 \land \Delta_j = 0$) extends the classifier to R3, with $D_{\mathrm{KL}}(\pi' \| \pi)$ separating policy-relevant from policy-irrelevant R3. 
A finite-sample version under noise-floor $\varepsilon$-bands follows from the same construction (\Cref{app:epsilonbands}).

\paragraph{What this pair establishes.} 
Theorems~\ref{thm:audit-insufficiency} and~\ref{thm:audit-sufficiency} tell us that audit-only inputs cannot separate R0, R0\textsubscript{cont}, R1, and R2, \textbf{even when $R$ itself is granted as an oracle input. }
Augmenting with policy-induced distributions and the spurious-feature partition closes the gap. 

\subsection{Implications for Mitigation Methods and Evaluations}
\label{sec:3_4_evalimps}
 
We translate the findings of Theorems~\ref{thm:audit-insufficiency} and~\ref{thm:audit-sufficiency} into prescriptions for new reward bias mitigation works. 
We state the corresponding prescriptions for benchmark developers, regime-specific gaps not closed by $B^\star$, multi-axis recommendations, and operator-variant caveats in \Cref{app:a8-takeaways}.

\paragraph{Method paper prescriptions.} A paper claiming successful reward bias mitigation should:
\begin{enumerate}
    \item \emph{Evaluate at policy-induced distributions and report $(\Delta_j, \Delta J)$.} By Theorem~\ref{thm:audit-sufficiency} this input separates R0, R0\textsubscript{cont}, R1, and R2, as no audit-only input does (Theorem~\ref{thm:audit-insufficiency}). 
    Concretely: Run the mitigated reward model in BoN or short PPO/GRPO against a fixed reference policy.
    \item \emph{Instrument off-target $\Delta_j$ on every $\phi_j \in \Phi_{\mathrm{sp}} \setminus \{\phi_i\}$ at $\mu_{\pi^\star}$}. 
    Reporting only on-target $|g_i| \approx 0$ leaves R0 structurally indistinguishable from R0\textsubscript{cont} and R1, regardless of audit-side strength.
    Concretely: At least track length, confidence, and sycophancy, as instantiated in \Cref{sec:4_evidence}.
    \item \emph{Report cardinal scale $\|\tilde R\|_{L^2(\muaud)}$ pre/post mitigation.} 
    By Corollary~\ref{cor:functional-blindspot}, ordinal scores are invariant under $\tilde R \mapsto c\tilde R$ ($c > 0$) while $J(\pi^\star_\beta(c\tilde R), R)$ varies with $c$. Since mitigation induces rescaling (\Cref{app:A3}), ordinal-only evaluation misses scale-driven regime shifts.
    \item \emph{Document $\pi_\mathrm{ref}$-sensitivity and gauge dependence.} Theorem~\ref{thm:audit-insufficiency} produces four distinct regimes by varying only $\pi_\mathrm{ref}$ for fixed $(\tilde{R}, M_i, \muaud)$. 
    The canonical $M_i$ is also sensitive to prompt-only reward shifts. 
    Validate across at least two reference policies of different lineages and report under $M_i^{\mathrm{cent}}$ (\Cref{app:A4}) to remove the gauge degree of freedom.
\end{enumerate}

\paragraph{Joint adoption.}
Successful mitigation (R0) certification is a joint property of the mitigation operator and the evaluation protocol. Method papers cannot unilaterally provide policy-distribution evaluation without supporting benchmark infrastructure and benchmarks cannot unilaterally provide off-target $\Delta_j$ measurement without method papers specifying $\Phi_{\mathrm{sp}}$. 
Joint adoption moves the field from R0 claims audit-side scores cannot certify (Theorem~\ref{thm:audit-insufficiency}) to R0 claims this framework can certify (Theorem~\ref{thm:audit-sufficiency}). 
Extended prescriptions, including benchmark developer recommendations and multi-axis operator guidance, are in \Cref{app:a8-takeaways}.
We map published mitigation work onto these regimes in \Cref{sec:4_evidence}.

\section{Reward Bias Substitution in the Wild}
\label{sec:4_evidence}

Bias substitution related failures are already present and occasionally miscategorized as an anomaly, a capability trade-off, or judge noise, because no prior framework separates substitution from overcorrection or successful mitigation.
No method we survey in the published RLHF and DPO mitigation literature reports the evidence \Cref{thm:audit-sufficiency} requires to certify successful mitigation (R0).

\paragraph{Reward models exhibit correlated, partially informative spurious features.}
Bias substitution requires correlated, partially informative spurious features under $\muaud$, both hold in deployed reward models (Fact~\ref{fact:length_bias}-\ref{fact:dpo_length_bias}, \ref{fact:evaluator_weaknesses_learnable}). 
Partial informativeness anchors the target-misspecification origin of R2, and correlated axes provide directions for pressure to redirect onto.
We extend this to policy outputs with length-sycophancy measurements on the \textit{r/AITA} dataset of \citet{cheng2026sycophantic} (Fact~\ref{fact:sycophancy_length_dependence}, see \Cref{app:AITA_syclength}). 
On 14375 open-ended responses of eight models to posts where Reddit users judged the poster at fault, responses that an LLM judge labels as endorsing the poster are shorter by 124.5 characters (95\% CI $[108.9, 140.0]$) after controlling for model and prompt, while separately elicited binary verdicts show no reliable length association ($-15.0$ characters, 95\% CI $[-30.3, 0.3]$). 
Sycophancy and length are therefore coupled, and \Cref{app:phasediagrams} gives closed-form phase diagrams of regime transitions driven by such couplings.

\paragraph{Reward bias substitution arises end-to-end under RLHF training.}
We run GRPO \citep{shao2024deepseekmathpushinglimitsmathematical} 
on \textit{Llama-3.2-3B-Instruct} \citep{grattafiori2024llama3herdmodels} with the Skywork-\textit{Reward-V2-Llama-3.1-8B} reward model \citep{liu2025skyworkrewardv2scalingpreferencedata}, shaping the reward as $\tilde{R}(x,y) = R_{\mathrm{RM}}(x,y) - \lambda\,n_{\mathrm{tok}}(y)/100$ for $\lambda \in \{0,4,8\}$ with four seeds per value. 
This shaping has the form of the canonical single-axis mitigation $M_i$ of Definition~\ref{def:single_axis_mitigation}, with a fixed coefficient in place of the estimated reliance.

As $\lambda$ rises, mean response length falls from 204 to 170 tokens and multiple-choice MMLU accuracy is unchanged, yet mean verbalized confidence rises from 0.80 to 0.84, confidence-correctness AUROC drops from 0.72 to 0.65, and free-form accuracy on TriviaQA falls from 0.56 to 0.41 (\Cref{fig:smoking_gun}, \Cref{tab:gunz}). 
Sycophancy flip rates show no statistically significant change when increasing $\lambda$.
The $\lambda=0$ runs leave verbalized confidence, AUROC, and TriviaQA accuracy at base level, so the degradation is caused by the penalty, not by RLHF itself. 
Optimization pressure redirects from length to verbalized confidence and the true reward proxy of free-form factual accuracy degrades, instantiating harmful R1 in standard RLHF. 
Full setup, training curves, and further per-metric results are in \Cref{app:gun}.

\paragraph{Bias mitigations that pass audit can fail under optimization.}
We show that the measurement-versus-optimization gap fails to close with previously published mitigation operators in \Cref{app:bon}.
We re-evaluate the post-hoc LOESS calibration of \citet{huang2025posthoc} and the linear-probe operator of \citet{fein2026one} across five reward models under Best-of-N selection.
On AlpacaEval, the calibration drives mean absolute pooled reward-length correlation from $0.316$ to $0.037$, an audit-side success by construction. 
Within prompts, all four SOTA reward models end with negative correlations on AlpacaEval, three of them larger in absolute value than at baseline. 
On GSM8K all five end positive, and the four SOTA reward models are larger than their baseline.
AlpacaEval LC win rate degrades below baseline on two cells, improves on two others, and GSM8K BoN accuracy drops $3.6$ points on one.
Under the $\varepsilon$-banded reading of \Cref{app:epsilonbands}, two cells satisfy $\Delta J < -\varepsilon_J$ and fall in R2$_\varepsilon$ or harmful R1$_\varepsilon$ depending on the unmeasured off-target axes.
Either way, both are mitigation failures that the pooled audit statistic reports as successes.

\paragraph{Published methods cannot certify successful mitigations.}
Of the 30 surveyed works in \Cref{app:regimes}, five show true reward falling in some reported setting \citep{singhal2024long,park2024disentangling,bu2025beyond,huang2025posthoc,fein2026one}, the last two in our re-evaluation in \Cref{app:bon}, including the operator of \citet{fein2026one} that still retains $\Delta J>0$ on four of five reward models.
Twelve of the other 25 do evaluate at the policy distribution and report $\Delta J > 0$, but do not report the needed off-target $\Delta_j$ across a stated $\Phi_{\mathrm{sp}}$, so R0, R0\textsubscript{cont}, and harmful R1 all remain consistent with their evidence (\Cref{tab:survey}), and the remaining do not measure a policy-side $\Delta J$ at all.
The omission is structural, since every reward-model benchmark we surveyed lies in the class $\mathcal{B}$ of \Cref{thm:audit-insufficiency}, and it falls on the features a length mitigation is likely to move.
In particular, penalizing length discourages hedging, which reward models score as confidence (\Cref{fact:length_epistemic_uncertainty}, \Cref{fact:confidence_quality_proxy}), so redirections like in \Cref{fig:smoking_gun} can be expected, yet no surveyed work measures confidence at the policy distribution.
Overall, reporting off-target feature changes $\Delta_j$, especially on length, confidence, and sycophancy, is missing for \Cref{thm:audit-sufficiency} to certify R0.

\section{Related Work}
\label{sec:related_work}

\paragraph{Prior definitions of reward hacking.}
Reward bias substitution is a Goodhart-type failure that survives mitigation, as suppressing the reliance of the proxy reward on a spurious feature relocates optimization pressure onto correlated proxies rather than removing it.
This differs from reward overoptimization, which concerns optimizing a fixed proxy \citep{gao2023scaling,mahmoud2026rewardhackingrubricbasedreinforcement}, whereas our failure mode is introduced by a mitigation operator and arises even under moderate optimization.
Existing definitions ask whether optimizing a given proxy \citep{skalseneurips2022defining, kwa2024catastrophic, laidlaw2025correlated} or the worst case over an admissible set of proxies \citep{fluri25a, liu2026orpo} produces an acceptable policy. 
Where mitigation enters, it acts on the policy, shown to fail in \citet{kwa2024catastrophic} and \citet{fluri25a} and to carry a guarantee in \citet{laidlaw2025correlated}, or at the reward level \citep{liu2026orpo}.
None defines an operator that edits proxy reliance on a named feature (\Cref{def:single_axis_mitigation}), so prior definitions cannot distinguish a mitigation that reduces reward hacking from one that relocates it, overshoots, or is ineffective.
Substitution has appeared unnamed as residual hacking after a length penalty \citep{miao2025theenergyloss}, off-target misalignment after consistency training \citep{africa2026consistencytrainingentrenchmisalignment}, and covert re-encoding under a monitor \citep{farquhar2025mona}.
The first two attribute it to incomplete coverage or distribution shift, the last detects relocation by a counterfactual intervention on the known feature in one task and responds by changing the optimizer, and none defines a feature-editing operator, which off-target features to track, or what audit-side scores can certify.
The redirection-onto-correlated-proxies mechanism also has precedents in reward design, where
design-time feature correlations break at deployment \citep{hadfieldmenell2017inverserewarddesign},
and in supervised learning, where suppressing one cue transfers reliance to the next
\citep{geirhos2020shortcut, Shahnature2020, scimeca2022which, licvpr2023whac}, bias mitigation methods
\citep{shrestha2022biasmitigation}, and group-level guarantees \citep{kearns2018preventing, sagawa2020distributionally}.

We provide the first formal taxonomy of these failures, with KL-regularized RLHF and DPO-like methods (\Cref{app:dpo}) as the primary instantiations.
The nearest formal precedents study a fixed proxy whose true return eventually decreases with optimization, with remedies that replace the proxy using the full utility \citep{zhuang2020consequences} or stop early under a proxy-error bound \citep{karwowski2024goodharts}. 
For us, the proxy is edited by a correction to a known bias at fixed budget and we ask what an audit can certify. 
\citet{sorg2010internalrewardsmitigate} change the reward to compensate for agent limitations rather than to remove bias.
\citet{fluri25a} show that bounding expected reward-model error cannot generally certify low regret. 
Theorem~\ref{thm:audit-insufficiency} makes the analogous point for every functional of the audit distribution, even with oracle access to $R$, and for regime labels rather than regret. 
Other analogues are partial certificates \citep{brown2021valuealignment} and off-distribution divergence under i.i.d.\ equivalence \citep{damour2022underspecification}. 
Our cardinal-scale sensitivity (Corollary~\ref{cor:functional-blindspot}) and $\pi_{\mathrm{ref}}$ dependence (Theorem~\ref{thm:audit-insufficiency}) are specific to the KL-regularized setting.

\paragraph{Reward bias mitigation methods and benchmarks.}
We position our work with respect to prior mitigation methods and benchmarks in \Cref{sec:4_evidence} and \Cref{app:regimes}. 
Our framework classifies what single-axis operators can certify rather than proposing a new one. 
Concurrent work shows that the broader failure is already observed. 
Narrow training changes behavior off target unpredictable by evaluations on the training task \citep{betley2026training,khan2026safetydriftfinetuningevidence}, and a safety-training mitigation can pass on its own prompt distribution while misalignment persists on agentic tasks \citep{macdiarmid2025naturalemergentmisalignmentreward}. 
Adjacent formal work addresses probabilistic reward modeling \citep{chen2026learning}, identifiability under correlated probit models \citep{cherapanamjeri2026learning}, and a KL-minimal sycophancy correction \citep{shapira2026rlhfamplifiessycophancy}. None formalize how single-axis mitigations redistribute optimization pressure onto correlated proxies.

\section{Discussion}
\label{sec:6_discussion}

\paragraph{Limitations.}
Our regime classification uses first-moment drift, which can 
miss distributional changes that leave the mean unchanged. 
Our survey in \Cref{app:survey} limits the practical severity of this gap, as the biases targeted by the operators we classify, primarily length, confidence, and sycophancy, are mean-level biases that published corrections partly remove.
Definition~\ref{def:spur_v_struct} partitions $\Phi$ into spurious and structurally relevant features rather than decomposing within-feature.
However, our approach is conservative under partial informativeness, as it under-flags substitution rather than over-flagging clean mitigations (\Cref{app:A2}). 
Finally, whenever someone targets a feature for debiasing, e.g., after judging policy rollouts, they declare it a member of $\Phi_{\mathrm{sp}}$, so our taxonomy and impossibility result apply regardless of whether $\Phi_{\mathrm{sp}}$ is ground-truth identifiable from preference data.

\paragraph{Impact.}
Current improvement metrics for reward bias mitigation cannot distinguish a model that has become less reward-hacky from one that has shifted exploitation modes, and this ambiguity propagates into post-deployment alignment claims. 
Our taxonomy and impossibility-sufficiency pair apply to any scalar proxy reward under single-axis mitigation, and we develop language model RLHF in detail because it is the primary lever for non-verifiable behavior, where substitution is demonstrated end-to-end in GRPO and consistent with failures of published operators on SOTA reward models.
Learned rewards on non-verifiable objectives cannot be specified in one shot (\Cref{sec:background}), so refinement is iterative, and each uncertified round can relocate a bias that the next round inherits as its target.
However, the gap is closable, as our checklist turns R0 claims that audit-side scores cannot certify into claims our framework certifies, once jointly adopted by method and benchmark developers. 
In high-compute settings the same prescriptions scale to multi-axis mitigation, but only when evaluated at the policy-induced distribution rather than the audit distribution, and even then they certify ``no substitution within the measured panel,'' not ``no substitution,'' since off-panel axes remain open to redirection and an audit that passes on the targeted axis does not certify the rest \citep{denison2024sycophancysubterfugeinvestigatingrewardtampering}.
\citet{macdiarmid2025naturalemergentmisalignmentreward} caution that passing evaluations cannot rule out a subtler version of the same failure.
Length is the natural anchor, the central axis through which sycophancy and confidence couple (Facts~\ref{fact:sycophancy_length_dependence},~\ref{fact:length_epistemic_uncertainty},~\ref{fact:confidence_quality_proxy}), and identifying the rest is the direction we most want pursued next.

\section*{Acknowledgments}

Max Lamparth is supported through a grant from Coefficient Giving (formerly Open Philanthropy), Stanford's Hoover Institution Tech Policy Accelerator, the Stanford Center for AI Safety, and the Stanford Intelligent Systems Laboratory.
Andreas Haupt has been supported by Schmidt Sciences.
Marcel Hussing was partially supported by DARPA grant \#HR00112420305. Any opinions, findings, and conclusions or recommendations expressed in this material are those of the author(s) and do not necessarily reflect the views, position, or policy of DARPA or the US Government.

\bibliography{mybib}
\bibliographystyle{plainnat}

\appendix

\section{Additional Formalizations}
\label{app:more_math}

\subsection{Causal Interpretation and Scope}
\label{app:causal-scope}

In this appendix, we outline the causal frame of our formalization in \Cref{sec:bias_sub} and delimit what the framework does and does not claim about the underlying causal structure of reward.

\paragraph{Underlying structural causal model.}
We posit, descriptively, that the true reward $R$ and the features $\Phi$ are related through an underlying structural causal model $\mathcal{M}$ on $(X, Y, R, \Phi)$. 
We do not assume that $\mathcal{M}$ is identified from preference data $\mathcal{D}$. 
Preference data identifies $R$ only up to a prompt-only additive shift under the Bradley-Terry likelihood \citep{rafailov2023direct}, leaving the causal structure of $\mathcal{M}$ partially identified at best. 
The framework of \Cref{sec:bias_sub} therefore operates within this partial-identification regime and it adopts causal vocabulary to state failure modes precisely, but its definitions and results do not require $\mathcal{M}$ to be recovered from $\mathcal{D}$.

\paragraph{Causal reading of Definition~\ref{def:spur_v_struct}.}
The corresponding causal reading of Definition~\ref{def:spur_v_struct} is: 
$\phi_i$ is \emph{causally spurious} with respect to $R$ iff $\phi_i$ has no direct edge into $R$ in $\mathcal{M}$, i.e., every directed path from $\phi_i$ to $R$ passes through other features in $\Phi$. 
Under positivity-style richness on the support of $\muaud$ (strengthening Assumption~\ref{asm:featreal}), faithfulness, and no unobserved common causes among $(\Phi, R)$, the causal and observational readings coincide on the partition $\Phi_{\mathrm{sp}}$ vs. $\Phi_{\mathrm{struct}}$.
The causal reading is the primitive notion and the observational form its identifiable shadow.
Thus, $\Phi_\mathrm{sp}$ vs.\ $\Phi_{\text{struct}}$ is best read as the conservative observational projection of an underlying causal partition preference data does not fully identify.

\paragraph{Partial informativeness as causal mediation.}
Partial informativeness (\Cref{sec:3_1_features,app:A2}) corresponds causally to $\phi_i$ having both a direct path to $R$ in $\mathcal{M}$ and one or more mediated paths through other features. Length is the canonical instance: a direct contribution via comprehensiveness, plus mediated effects of hedging, sycophancy, and verbosity compensation (see \Cref{sec:4_evidence,app:survey}). 
\Cref{def:spur_v_struct} generically places such features in $\Phi_{\mathrm{struct}}$, since the direct path makes the conditional mean of $R$ depend on $\phi_i$ at $\muaud$, so redirection onto them is not counted by the redirection predicate of \Cref{def:all_regimes} (the under-flagging noted in \Cref{sec:3_1_features}).

\paragraph{Associational mitigation operators.}
The single-axis mitigations $M_i$ of Definition~\ref{def:single_axis_mitigation} are associational, as they target the linear reliance statistic $g_i$, which is an $L^2(\muaud)$ regression coefficient, not a controlled direct effect. 
A causal counterpart would target the controlled direct effect of $\phi_i$ on $R$, identifiable only under interventions on $\mathcal{M}$ or strong conditional-independence assumptions that preference data does not warrant. 
The data augmentation of \citet{liu2025rrm} imposes such assumptions by construction. 
It pairs each prompt with responses written for other prompts and labels the pairs by rule so that the augmented data satisfy two d-separation constraints, and its guarantee that the artifact-to-label edge is removed holds under the assumed DAG, faithfulness, and the rule that a response written for another prompt loses.
Existing single-axis mitigations in the language reward modeling literature \citep[e.g.,][]{fein2026one,papadatos2024linear,huang2025posthoc} are associational in this sense, and the regime taxonomy of \Cref{sec:3_3_bias_substitution} characterizes their failure modes precisely because of the gap between associational construction and causal response of the optimized policy.

\paragraph{Overcorrection (R2) origins in causal language.}
The two origins of R2 in Definition~\ref{def:all_regimes} restate causally as: scale overshoot (projection coefficient too large at $\mu_{\pi^\star}$ while $\phi_i$ remains causally spurious) versus target misspecification ($\phi_i$ has a non-zero direct path to $R$ in $\mathcal{M}$ that the projection has zeroed at $\muaud$).
The misspecification case has a supervised-learning precedent in \citet{tsipras2018robustness}.
In their Gaussian model, the classifier that relies solely on the robust feature attains accuracy $p<1$ while the unconstrained classifier attains accuracy arbitrarily close to one, and adversarial training provably produces the former, so the cost of enforcing invariance to a predictive but non-robust feature is a property of the data distribution and persists in the infinite-data limit.
The same cost appears inside the projection operator of \Cref{def:single_axis_mitigation} itself in HEX \citep{wang2018learning}, whose authors report that their texture representation also absorbs semantic information that the projection then removes, with a corresponding accuracy drop. 
This is misspecification under empirical risk minimization, with no optimizer responding to the projected object and no partial-projection sweep, so it evidences the origin rather than the R2 regime or its non-recoverability.
The rescalability test of \Cref{app:a8-takeaways} discriminates these origins observationally, without identifying $\mathcal{M}$.

\paragraph{Audit-distribution sensitivity (R4) and transportability.}
The audit-distribution sensitivity of Definition~\ref{def:R4} is structurally analogous to the transportability question \citep[e.g.][]{pearl_externalvalidity}. 
An identified mitigation at one audit distribution does not need to transport to another. 
R4 names this phenomenon under partial identification, without formally invoking transportability machinery.

\subsection{Standing Assumptions and Optimization Setup}
\label{app:A1}

\begin{assumption}[Regularity]\label{asm:reg}
$R, \tilde{R} \in L^2(\muaud) \cap L^\infty(\mu_{\pi_{\mathrm{ref}}})$ 
and $\phi_k \in L^2(\muaud) \cap L^\infty(\mu_{\pi_{\mathrm{ref}}})$ 
for each $k = 1, \ldots, K$, where $\Phi = (\phi_1, \ldots, \phi_K)^\top$ is the 
feature map of Definition~\ref{def:feature_map}.
\end{assumption}

The $L^2(\muaud)$ inclusion makes $g(\cdot;\muaud)$ of Definition~\ref{def:linear_reliance} well-defined as a Gram-inner-product statistic at the audit distribution. 
The $L^\infty(\mu_{\pi_\mathrm{ref}})$ conditions are natural both for bounded-output reward models and for the interpretable surface features this paper considers (length, formatting counts, style indicators). 
The conditions imply that any $\bar R$ obtained as a finite $\mathbb{R}$-linear combination of $\tilde R$, $R$, and the $\phi_k$ also lies in $L^\infty(\mu_{\pi_\mathrm{ref}})$, which covers every $\bar R$ used in this paper, including $M_i(\tilde R)$ of Definition~\ref{def:single_axis_mitigation} and the scale-normalized variant $M_i^\text{norm}$ introduced below. 
Since $\mu_{\pi_\mathrm{ref}}$ is a probability measure, $L^\infty(\mu_{\pi_\mathrm{ref}}) \subset L^2(\mu_{\pi_\mathrm{ref}})$ holds automatically and no separate $L^2(\mu_{\pi_\mathrm{ref}})$ clause is needed.

\textbf{Optimization setup.} Fix $\beta > 0$. For any $\bar R$ satisfying Assumption~\ref{asm:reg}, define the KL-regularized optimizer
\begin{equation*}
    \pi^\star_\beta(\bar R) \;=\; \arg\max_\pi \; \mathbb{E}_{x \sim \mathcal{D},\, y \sim \pi}\!\big[\bar R(x,y)\big] \;-\; \beta\, D_\mathrm{KL}\!\big(\pi \,\|\, \pi_\mathrm{ref}\big),
\end{equation*}
which in the single-turn contextual-bandit regime admits the unique softmax closed form $\pi^\star_\beta(\bar R)(y \mid x) \propto \pi_\mathrm{ref}(y \mid x)\,\exp(\bar R(x,y)/\beta)$ recalled in the main text. Essential boundedness of $\bar R$ gives two-sided bounds on both $\exp(\bar R/\beta)$ and the normalizer $Z(x) = \mathbb{E}_{y \sim \pi_\mathrm{ref}(\cdot \mid x)}[\exp(\bar R(x,y)/\beta)]$, so the density ratio $d\mu_{\pi^\star_\beta(\bar R)}/d\mu_{\pi_\mathrm{ref}} = \exp(\bar R/\beta)/Z(x)$ is essentially bounded. 
The chain $\bar R \in L^\infty(\mu_{\pi_\mathrm{ref}}) \Rightarrow L^\infty(\mu_{\pi^\star_\beta(\bar R)}) \subset L^2(\mu_{\pi^\star_\beta(\bar R)})$ then follows immediately, and analogously for each $R$ and each $\phi_k$. 
The policy-level expectations $\mathbb{E}_{\mu_{\pi^\star}}[\phi_j]$ entering $\Delta_j$ in Definition~\ref{def:all_regimes} are then finite. 
Well-definedness of $g(\bar R; \mu_{\pi^\star_\beta(\bar R)})$ additionally requires Gram non-degeneracy at $\mu_{\pi^\star}$. This condition is handled in \Cref{app:A2}, where the relevant audit distributions are introduced.

\paragraph{Unbounded response spaces.} 
The closed-form constructions in later \Cref{app:nonsuck,app:phasediagrams,app:audit-insufficiency} and \Cref{cor:dpo_impossibility} use $\mathcal{Y} = \mathbb{R}^K$ with Gaussian $\pi_\mathrm{ref}$ and polynomial rewards, which are not essentially bounded.
However, our constructions meet the $L^\infty(\mu_{\pi_\mathrm{ref}})$ clause of Assumption~\ref{asm:reg} by restricting $\pi_\mathrm{ref}$ to the box $\mathcal{Y}_T = [-T, T]^K$. 
So for completeness, we write $\mathcal{N}_T(0, \Sigma)$ for $\mathcal{N}(0, \Sigma)$ conditioned on $\mathcal{Y}_T$. 
The following lemma shows that the restriction is free and unproblematic for those constructions. 
Every closed form reported there is the $T \to \infty$ limit of the truncated quantity, every exact zero used for a regime label holds for every $T$, and every strict sign holds for all sufficiently large $T$, so each construction is carried out at one fixed finite $T$.

\begin{lemma}[Box truncation for Gaussian distributions]
\label{lem:box_truncation}
Let $\Sigma \in \mathbb{R}^{K \times K}$ be positive definite, $\beta > 0$, $\varphi_\Sigma$ the density of $\mathcal{N}(0, \Sigma)$, and $\phi_k(y) = y_k$ the coordinate features. 
Let $\bar R$ be a continuous function on $\mathbb{R}^K$ with $\mathbb{E}_{\mathcal{N}(0, \Sigma)}\big[(1 + \|y\|)\exp(\bar R(y)/\beta)\big] < \infty$. 
Write $\pi^\star_{\beta, T}(\bar R)$ for the KL-regularized optimum of \Cref{equ:softmaxsolution} against $\pi_\mathrm{ref} = \mathcal{N}_T(0, \Sigma)$, and $\pi^\star_\beta(\bar R)$ for the probability measure on $\mathbb{R}^K$ with density proportional to $\exp(\bar R/\beta)\,\varphi_\Sigma$.
\begin{enumerate}
    \item[(i)] (Regularity.) For every $T > 0$, $\bar R$ and each $\phi_k$ lie in $L^\infty(\mathcal{N}_T(0, \Sigma))$, so Assumption~\ref{asm:reg} holds with $\pi_\mathrm{ref} = \mathcal{N}_T(0, \Sigma)$, and the Gram matrices $\mathbb{E}[\Phi\Phi^\top]$ at $\mathcal{N}_T(0, \Sigma)$ and at $\pi^\star_{\beta, T}(\bar R)$ are positive definite.
    \item[(ii)] (Convergence.) For every continuous $f$ with $\mathbb{E}_{\mathcal{N}(0, \Sigma)}[|f|\exp(\bar R/\beta)] < \infty$, $\mathbb{E}_{\pi^\star_{\beta, T}(\bar R)}[f] \to \mathbb{E}_{\pi^\star_\beta(\bar R)}[f]$ as $T \to \infty$.
    \item[(iii)] (Exact zeros.) Suppose $\Sigma_{jk} = 0$ for all $k \neq j$, and $\bar R(y) = h(y_{-j}) + u(y_j)$ and $\bar R'(y) = h'(y_{-j}) + u(y_j)$ for continuous $h, h', u$, where $y_{-j}$ collects the coordinates other than $y_j$. 
    Then $\mathbb{E}_{\pi^\star_{\beta, T}(\bar R)}[\phi_j] = \mathbb{E}_{\pi^\star_{\beta, T}(\bar R')}[\phi_j]$ for every $T > 0$.
\end{enumerate}
\end{lemma}

\begin{proof}
(i) A continuous function is bounded on the compact set $\mathcal{Y}_T$, and $\mathcal{N}_T(0, \Sigma)$ is supported on $\mathcal{Y}_T$. 
The density ratio $d\pi^\star_{\beta, T}(\bar R)/d\mathcal{N}_T(0, \Sigma) = \exp(\bar R/\beta)/Z_T$ is bounded above and below on $\mathcal{Y}_T$, so the chain above applies. 
Both measures have a density that is strictly positive on $\mathcal{Y}_T$, a set with non-empty interior, so no non-zero linear combination of the coordinates vanishes almost surely and both Gram matrices are positive definite.

(ii) Write $Z_T = \int_{\mathcal{Y}_T} \exp(\bar R/\beta)\,\varphi_\Sigma$ and $S_T(f) = \int_{\mathcal{Y}_T} f \exp(\bar R/\beta)\,\varphi_\Sigma$, so that $\mathbb{E}_{\pi^\star_{\beta, T}(\bar R)}[f] = S_T(f)/Z_T$. 
As $T \to \infty$ the indicator $\mathbf{1}_{\mathcal{Y}_T}$ increases to $1$ pointwise, and both $|f|\exp(\bar R/\beta)\,\varphi_\Sigma$ and $\exp(\bar R/\beta)\,\varphi_\Sigma$ are integrable, so dominated convergence gives $S_T(f) \to S_\infty(f)$ and $Z_T \to Z_\infty \in (0, \infty)$, hence $S_T(f)/Z_T \to S_\infty(f)/Z_\infty$.

(iii) Since $\Sigma_{jk} = 0$ for $k \neq j$, the density factorizes as $\varphi_\Sigma(y) = \varphi_{\Sigma_{-j}}(y_{-j})\,\varphi_{\Sigma_{jj}}(y_j)$, and $\mathcal{Y}_T$ is a product of intervals. 
The density of $\pi^\star_{\beta, T}(\bar R)$ is therefore proportional to $\exp(h(y_{-j})/\beta)\,\varphi_{\Sigma_{-j}}(y_{-j})\,\mathbf{1}_{[-T, T]^{K-1}}(y_{-j})$ times $\exp(u(y_j)/\beta)\,\varphi_{\Sigma_{jj}}(y_j)\,\mathbf{1}_{[-T, T]}(y_j)$, so $y_j$ is independent of $y_{-j}$ with a marginal that depends only on $u$, $\Sigma_{jj}$, and $T$. 
The same holds for $\bar R'$ with the same $u$, so the two expectations of $\phi_j$ coincide.
\end{proof}

For linear $\bar R$ the integrability hypothesis holds for every positive-definite $\Sigma$, since a Gaussian distribution has finite exponential moments of every linear function, and then $\pi^\star_\beta(\bar R) = \mathcal{N}(\Sigma c/\beta, \Sigma)$ holds by completing the square. 
For the quadratic rewards used as non-linear instantiation in \Cref{app:phasediagrams}, the integrability hypothesis holds exactly when $2\gamma/\beta < 1$, which is stated there.

\subsection{Feature Realizability and the Audit Distribution}
\label{app:A2}

This section collects four items referenced from \Cref{sec:bias_sub}. 
A discussion of the choices available for $\muaud$ (pointed to from the preamble of \Cref{sec:bias_sub}), the Gram non-degeneracy condition for the decoupled case, the structural Assumption~\ref{asm:featreal} (feature realizability) on the feature map (invoked by Definition~\ref{def:spur_v_struct}), and the expressiveness remark that motivates reading the regime language of \Cref{sec:3_3_bias_substitution} conservatively under partially informative features (pointed to immediately after Definition~\ref{def:spur_v_struct}). 
The order below matches the narrative order of \Cref{sec:bias_sub}.

\textbf{Remark (choice of audit distribution).} The main text introduces $\muaud$ as an auxiliary distribution on $\mathcal{X} \times \mathcal{Y}$ used for reliance estimation and correlation measurement, distinct from the KL anchor $\pi_\mathrm{ref}$. Two choices are load-bearing for this paper:
\begin{enumerate}
    \item \textit{The coupled case $\muaud = \mu_{\pi_\mathrm{ref}}$.} Recovers the setting of \citet{laidlaw2025correlated} and makes diagnostic statistics directly comparable to their bounds, at the cost of coupling measurement and optimization reference points.
    \item \textit{An annotator-conditioned distribution $\muaud^\text{human}$ or $\muaud^\text{LLM}$}, reflecting that the same proxy $\tilde R$ induces different $g$-profiles depending on which annotator pool generated the audit data \citep{movva2026whats}. 
    This formalizes the human-vs-LLM-judge gap documented in \citet{saito2023verbosity, zheng2023judging, chenemnlp2024humans}, and is what R4 (Definition~\ref{def:R4}) operationalizes.
    On the label side, \citet{findeis2025inverse} show that which biases a human preference dataset encodes depends on its source, as a principle preferring neutrality over information reconstructs 58.0\% of human labels on the relevant subset of the PRISM preference dataset but only 40.8\% on the relevant subset of Chatbot Arena data, which the authors describe as dataset-specific bias.
    To illustrate these points grounded in domain expertise, in mental healthcare, three psychiatrists rating identical responses under a shared rubric reach single-rater ICC 0.09 to 0.30, with directional offsets that annotator interviews trace to incompatible clinical frameworks rather than noise \citep{jafari2026expert}, and expert preference annotations fit with hierarchical Bradley-Terry models carry per-annotator slope and offset parameters that are fit as nuisance parameters and omitted from the released labels \citep{lamparth2026moving}, 
    so the annotator-dependent reward scale (see \Cref{sec:background}) cannot be recovered downstream.
\end{enumerate}
Other valid choices, e.g., the preference-data distribution used to train $\tilde R$, curated audit sets, deployment distributions, and round-anchored distributions in iterative RLHF \citep{yeilr2025iterative}, yield different $g$-profiles and shift substitution diagnostics accordingly. 
Substitution claims are always relative to the chosen audit distribution, and the choice should be stated explicitly in applications.

Where $g(\bar R; \mu_{\pi^\star_\beta(\bar R)})$ is referenced in \Cref{sec:bias_sub} (e.g., the measurement-vs-optimization gap following Lemma~\ref{lem:single_axis_identity}), well-definedness requires Gram non-degeneracy at $\mu_{\pi^\star}$, which follows from Assumption~\ref{asm:nondeg} (non-degeneracy) whenever $d\mu_{\pi^\star}/d\muaud$ is bounded above and below. In the coupled case ($\muaud = \mu_{\pi_\mathrm{ref}}$), the bound follows from Assumption~\ref{asm:reg} (regularity). 
In the decoupled case ($\muaud^\text{human}, \muaud^\text{LLM}$), it is a separate regularity condition assumed in the relevant statements.

\begin{assumption}[Feature Realizability]\label{asm:featreal}
For $\muaud$-almost every $(x, y) \in \mathcal{X} \times \mathcal{Y}$ and every feature index $i \in \{1, \ldots, K\}$, there exists at least one $y' \in \mathcal{Y}$ with $\phi_j(x, y') = \phi_j(x, y)$ for all $j \neq i$ and $\phi_i(x, y') \neq \phi_i(x, y)$.
\end{assumption}

Assumption~\ref{asm:featreal} is the non-vacuity condition for Definition~\ref{def:spur_v_struct}.
Without it, $\phi_i$ can be functionally determined by the prompt and the other features on $\mathcal{Y}$, so further conditioning on $\phi_i$ leaves the conditional expectation in Definition~\ref{def:spur_v_struct} unchanged and every feature is trivially classified as spurious. 
Assumption~\ref{asm:featreal} is therefore a richness condition on the response space \textit{under $\muaud$}, as $\mathcal{Y}$ must be rich enough where the audit distribution places mass that each feature axis admits independent perturbation.

We treat $g$ and $\Delta_j$ as observational statistics under $\muaud$ and $\mu_{\pi^\star}$ throughout, as an interventional reading along $\phi_i$ would require strengthening Assumption~\ref{asm:featreal} to a richness condition on counterfactual realization, which we do not invoke.

Even after weakening to the $\muaud$-a.e.\ form, Assumption~\ref{asm:featreal} is fragile for natural-language features even within $\Phi$. 
Consider $\Phi = \{\phi_\text{length}, \phi_\text{exclamation\_count}\}$. 
Assumption~\ref{asm:featreal} demands that $\muaud$-a.e.\ there exist a response $y'$ matching $y$ on exclamation count but differing in length, and conversely. 
The first direction is plausible, as long and short variants with the same exclamation count are typically realizable. 
The second is genuinely constrained, as at fixed length, the realizable range of exclamation counts is bounded above by length itself, and at very short lengths the counterfactual set may be empty. 
Even within $\Phi$, the response space does not cleanly factor into independent feature axes. 
This failure mode is specific to features that count items within the response (exclamation marks, bullet points, headers), since their range is mechanically tied to length. Features operationalized as binary indicators or response-level scalars do not have this bound (e.g., sycophancy and confidence in \Cref{sec:4_evidence,app:AITA_syclength}), and Assumption~\ref{asm:featreal} holds approximately for these pairs.

Feature misspecification is itself a major failure mode in practice.
Under Assumption~\ref{asm:featreal}, Definition~\ref{def:spur_v_struct} cleanly separates spurious from structurally relevant features, and the canonical mitigation $M_i$ of Definition~\ref{def:single_axis_mitigation} targets a well-defined coordinate of $g$-space. 
When Assumption~\ref{asm:featreal} fails, the definitions of \Cref{sec:3_1_features} should be read as approximations, and the binary spurious/structurally-relevant partition as the best dichotomous summary of a genuinely continuous informativeness spectrum.

\paragraph{Remark (scope of the binary partition under partial informativeness).} 
Definition~\ref{def:spur_v_struct} partitions $\Phi$ into spurious and structurally relevant features. 
The partition is scope-matched to the operator class deployed in the literature (whole-feature projections, see \Cref{sec:3_1_features}) and to what preference data identifies, but it does not separate within-feature components for partially informative features in the empirical setting surveyed in \Cref{sec:4_evidence}, where many natural features (foremost length) act as mediators for multiple mechanisms (sycophancy, epistemic uncertainty, informativeness) rather than as pure spurious proxies or pure structural drivers.

Under partial informativeness, the regime distinction of \Cref{sec:3_3_bias_substitution} still applies, but with a scoping caveat: bias substitution becomes a \textit{conservative classification} of the ways a single-axis mitigation can fail, because real features admit additional substitution modes that the binary partition cannot express. 
Concretely, the best $L^2(\muaud)$ approximation of $\phi_i$ by a function of $R$ alone (writing $\phi_i = \mathbb{E}_{\muaud}[\phi_i \mid R] + \phi_i^\perp$, i.e., decomposing $\phi_i$ into ``the part that tracks $R$'' and a residual) does not in general coincide with the causal-versus-spurious decomposition of $\phi_i$ under the underlying structural causal model $\mathcal{M}$ (\Cref{app:causal-scope}). 
The first is a $\muaud$-level regression decomposition, the second a structural statement under $\mathcal{M}$.
Reconciling the two requires additional structure, e.g., a causal model, conditional-independence assumptions, or auxiliary interventions, and we defer partial-informativeness extensions to future work.

\paragraph{Exhaustiveness of the regime taxonomy.}
The two axes underlying \Cref{sec:3_3_bias_substitution},
whether mitigation redirects first-moment exploitation onto another spurious feature 
($\Delta_j(\pi, \pi') \neq 0$ for some $\phi_j \in \Phi_\mathrm{sp} \setminus \{\phi_i\}$) 
and the sign of the change in true reward 
($\Delta J = J(\pi', R) - J(\pi, R)$), 
partition the outcome space into six cells, summarized in \Cref{def:all_regimes}. 
Five regime labels (R0, R0\textsubscript{cont}, R1, R2, R3) suffice because R1 absorbs both $\Delta J = 0$ (neutral substitution) and $\Delta J < 0$ (harmful substitution) under Definition~\ref{def:all_regimes}, while R0 and R0\textsubscript{cont} receive separate labels because the redirection-with-improvement corner is precisely the case audit-distribution evaluation most readily mistakes for clean R0 (see \Cref{thm:audit-insufficiency}). 
The asymmetry is therefore deliberate rather than ad hoc: $\mathrm{R0}_{\mathrm{cont}}$ earns a separate label for downstream emphasis in the impossibility result, whereas neutral and harmful R1 are both already failures of $\Delta J$ and the sub-case treatment suffices.
The R4 axis (Definition~\ref{def:R4}) does not appear in \Cref{def:all_regimes} because R4 is not an outcome cell but a property of how cell membership shifts across $\muaud$ choices, as discussed in \Cref{sec:3_3_bias_substitution} 
and instantiated in closed form in \Cref{app:phasediagrams}.

\paragraph{Spurious-redirection-with-improvement corner.} R1 in Definition~\ref{def:all_regimes} (bias substitution) requires $\Delta J \leq 0$. 
The complementary corner $\Delta_j \neq 0$ for some $\phi_j \in \Phi_\mathrm{sp} \setminus \{\phi_i\}$ with $\Delta J > 0$ is regime R0\textsubscript{cont} per Definition~\ref{def:all_regimes}. 
$R0_{\text{cont}}$ arises when redirection onto a genuinely spurious $\phi_j$ is outweighed by structural-feature gains from reducing $\phi_i$. 
Redirection onto a partially informative feature does not produce $R0_{\mathrm{cont}}$, since \Cref{def:spur_v_struct} places it in $\Phi_{\mathrm{struct}}$ and the redirection predicate ignores it.

The practical consequence for the rest of the paper is that substitution claims proved under Assumption~\ref{asm:featreal} (feature realizability) should be read as a floor rather than a complete characterization: when Assumption~\ref{asm:featreal} fails, at least the substitution mode the binary partition captures remains available to the optimizing policy, and generically additional modes are available too.

\subsection{Equivalent Epsilon-Formulation for Local Spuriousness}
\label{app:local-spuriousness}

\Cref{def:spur_v_struct} states spuriousness as conditional mean independence at $\muaud$. 
We derive the equivalent $o(\varepsilon)$ formulation for completeness.

Let $\mathcal{Y}_{-i}(x,y) := \{y' : \phi_j(x,y') = \phi_j(x,y)\ \forall j \neq i\}$ denote the $i$-fiber through $(x,y)$, and decompose $\muaud = \mathcal{D}(x)\,\pi_\mathrm{aud}(y \mid x)$. 
For a Markov kernel $\kappa$ with $\kappa(\mathcal{Y}_{-i}(x,y) \mid x,y) = 1$ for $\muaud$-a.e.\ $(x,y)$, define the mixture-perturbed conditional
\begin{equation}
    \pi_\varepsilon^\kappa(y' \mid x) \;:=\; (1-\varepsilon)\,\pi_\mathrm{aud}(y' \mid x) \,+\, \varepsilon \int \kappa(y' \mid x,y)\,\pi_\mathrm{aud}(dy \mid x), \qquad \varepsilon \in [0,1].
\end{equation}
Consider the perturbation condition
\begin{equation}
    \mathbb{E}_{x \sim \mathcal{D},\, y' \sim \pi_\varepsilon^\kappa(\cdot \mid x)}\!\bigl[R(x,y')\bigr] \,-\, \mathbb{E}_{\muaud}\!\bigl[R(x,y)\bigr] \;\in\; o(\varepsilon) \quad \text{as } \varepsilon \downarrow 0.
\end{equation}
The left-hand side equals $\varepsilon \cdot \bigl(\mathbb{E}_{\kappa \otimes \muaud}[R] - \mathbb{E}_{\muaud}[R]\bigr)$ for mixture perturbations, so the $o(\varepsilon)$ condition is equivalent to $\mathbb{E}_{\kappa \otimes \muaud}[R] = \mathbb{E}_{\muaud}[R]$.
Requiring this for every kernel $\kappa$ implies the conditional mean independence of \Cref{def:spur_v_struct}, and the two are equivalent when $R$ is $\sigma(\Phi)$-measurable (so that within-fiber variation of $R$ collapses to variation in $\phi_i$).

\paragraph{Verification of downstream results.}
\Cref{thm:audit-insufficiency} and the closed-form regime instantiations in \Cref{app:nonsuck,app:phasediagrams} use a true reward of the form $R = w \phi_3$, which is $\sigma(\Phi)$-measurable. 
Thus, the conditional mean independence and strict point-wise invariance on $i$-fibers coincide in this regime.
The constructed instances satisfy \Cref{def:spur_v_struct} and the proofs transfer without modification.
\Cref{thm:audit-sufficiency} is a correctness statement on the classifier $B^\star$ given $(\Delta_j, \Delta J)$ and $\Phi_\mathrm{sp}$ as inputs.
It places no assumption on the functional form of $R$ and transfers regardless of how $\Phi_\mathrm{sp}$ is defined.

\subsection{Scale Change Induced by Mitigation and a Scale-Invariant Variant}
\label{app:A3}

Immediately after Lemma~\ref{lem:single_axis_identity}, we note that applying $M_i$ changes $\|\tilde R\|_{L^2(\muaud)}$ through both a diagonal and a cross-correlation contribution, and that this scale change interacts non-trivially with the KL-regularized optimizer because $\pi^\star_\beta$ is not invariant to positive rescaling of the reward. 
In this section, we give the exact scale identity for $M_i$, isolate the mechanism by which scale change conflates with axis reallocation at fixed $\beta$, and define the scale-invariant variant $M_i^\text{norm}$ that separates the two effects.

\textbf{Scale identity.} 
Recall the Gram matrix $G$ of Assumption~\ref{asm:nondeg}, with entries $G_{ij} = \mathbb{E}_{\muaud}[\phi_i \phi_j]$,
let $g = g(\tilde R; \muaud)$ abbreviate the linear-reliance vector of Definition~\ref{def:linear_reliance}, and write $\tilde R' = M_i(\tilde R) = \tilde R - g_i \phi_i$ for the canonical mitigation of Definition~\ref{def:single_axis_mitigation}.
Expanding the $L^2(\muaud)$-norm of $\tilde R'$,
\begin{equation*}
    \|\tilde R'\|_{L^2(\muaud)}^2 \;=\; \|\tilde R\|_{L^2(\muaud)}^2 \;-\; 2 g_i \,\mathbb{E}_{\muaud}[\phi_i \tilde R] \;+\; g_i^2\, G_{ii},
\end{equation*}
and using $\mathbb{E}_{\muaud}[\phi_i \tilde R] = (Gg)_i = \sum_j G_{ij} g_j$ gives the exact identity
\begin{equation*}
    \|\tilde R'\|_{L^2(\muaud)}^2 \;=\; \|\tilde R\|_{L^2(\muaud)}^2 \;-\; g_i\!\left[\,g_i\, G_{ii} \;+\; 2 \sum_{j \neq i} G_{ij}\, g_j\,\right].
\end{equation*}
The sign of the change is governed by the bracketed expression relative to $g_i$: the norm strictly decreases whenever the bracketed term is co-signed with $g_i$, and can \textit{increase} when the cross-correlation contribution $2 \sum_{j \neq i} G_{ij} g_j$ is oppositely signed and sufficiently large in magnitude relative to the diagonal term $g_i G_{ii}$. 
A parameter choice exhibiting this norm increase is given in \Cref{app:nonsuck}.
The diagonal-only approximation $\|\tilde R'\|^2 \approx \|\tilde R\|^2 - g_i^2\, G_{ii}$ is correct only when $G$ is diagonal at $\muaud$, i.e., when the audit distribution renders the feature axes orthogonal.
This condition generically fails for interpretable feature sets such as length, formatting, and style indicators, which are empirically correlated on any natural $\muaud$.

\textbf{Scale-invariant variant.} The $L^2(\muaud)$-normalized mitigation
\begin{equation*}
    M_i^\text{norm}(\tilde R) \;=\; \alpha\,\big(\tilde R - g_i \phi_i\big), \qquad \alpha := \frac{\|\tilde R\|_{L^2(\muaud)}}{\|\tilde R - g_i \phi_i\|_{L^2(\muaud)}},
\end{equation*}
preserves $\|\cdot\|_{L^2(\muaud)}$ by construction. Its linear reliance at $\muaud$ is $\alpha$ times that of $M_i(\tilde R)$, so $g_i(M_i^\text{norm}(\tilde R); \muaud) = 0$ and $g_j(M_i^\text{norm}(\tilde R); \muaud) = \alpha\, g_j(\tilde R; \muaud)$ for $j \neq i$: the $i$-th coordinate is zeroed and the remaining axes are rescaled uniformly.

\textbf{Effective-$\beta$ shift.} By KL scale-equivariance $\pi^\star_\beta(c\bar R) = \pi^\star_{\beta/c}(\bar R)$ and the relation $M_i^\text{norm}(\tilde R) = \alpha\,M_i(\tilde R)$,
\begin{equation*}
    \pi^\star_\beta\!\big(M_i(\tilde R)\big) \;=\; \pi^\star_{\alpha\beta}\!\big(M_i^\text{norm}(\tilde R)\big).
\end{equation*}
The unnormalized comparison $\pi^\star_\beta(\tilde R)$ vs $\pi^\star_\beta(M_i(\tilde R))$ at fixed $\beta$ therefore equals the normalized comparison $\pi^\star_\beta(\tilde R)$ vs $\pi^\star_{\alpha\beta}(M_i^\text{norm}(\tilde R))$, differing from the same-$\beta$ normalized comparison by exactly an effective-$\beta$ shift from $\beta$ to $\alpha\beta$. 
This is a cardinal-distortion mechanism rather than an ordinal one, as two proxies producing identical pairwise rankings on $\muaud$ can differ in $L^2(\muaud)$-norm and induce different $\pi^\star_\beta$ at fixed $\beta$.
This mechanism is in line with the \Cref{sec:background} ordinal-vs-cardinal gap, which motivates reporting $\|\tilde R\|_{L^2(\muaud)}$ alongside $\Delta_j$ in empirical instantiations.

As a practical consequence, under $M_i^\text{norm}$, any observed change in $\mathbb{E}_{\pi^\star_\beta}[\phi_j]$ at fixed $\beta$ is attributable to axis reallocation rather than scale, and $\Delta_j$ values can differ in sign between $M_i$ and $M_i^\text{norm}$ when $\alpha$ is far from 1, classifying the same proxy into different R0-R3 regimes under the two mitigations. 
Diagnosing R1 versus R3 cleanly therefore calls for $M_i^\text{norm}$, since under $M_i$ a nonzero $\Delta_j$ admits both reallocation and scale interpretations and requires a $\beta$-scan or pre/post norm reporting to disambiguate.

\paragraph{Deployment-time normalization.} 
Standard practices in PPO-based RLHF for stability like per-batch reward whitening \citep{lambertrlhf2025reinforcement} apply a sample-dependent shift and scale to $\tilde R$ at each training step.
In expectation the shift approximates a constant under the current rollout distribution and is therefore gauge-equivalent at the population level (no policy effect, by \Cref{sec:3_2_linearfix}), but the scale component divides $\tilde R$ by an estimate of $\mathrm{Std}_{\mu_{\pi_t}}(\tilde R)$, inducing an effective KL coefficient $\beta \cdot \mathrm{Std}_{\mu_{\pi_t}}(\tilde R)$ that $M_i^\text{norm}$ does not control. 
$M_i^\text{norm}$ fixes the audit-side $L^2(\muaud)$ scale, while whitening rescales against the (non-stationary) training-side dispersion.
The $(M_i, M_i^\text{norm})$ comparison therefore does not exhaust the scale story under whitened RLHF, and empirical instantiations should report training-side reward dispersion alongside $\|\tilde R\|_{L^2(\muaud)}$ when interpreting $\Delta_j$ across mitigations.

\subsection{Gauge-Invariant Linear Reliance}
\label{app:A4}

In this section, we give a gauge-invariant variant of the linear reliance statistic and verify that the Section~\ref{sec:3_3_bias_substitution} regime classification transfers under it.
Define the prompt-conditional deviations under $\muaud$,
\begin{equation*}
    \tilde R^\mathrm{cent}(x,y) = \tilde R(x,y) - \mathbb{E}_{\muaud}[\tilde R \mid x], \qquad \bar\Phi(x,y) = \Phi(x,y) - \mathbb{E}_{\muaud}[\Phi \mid x].
\end{equation*}

\begin{assumption}[Non-degeneracy]\label{asm:nondeg_gauge}
The centered Gram matrix $\bar G := \mathbb{E}_{\muaud}[\bar\Phi \bar\Phi^\top] \in \mathbb{R}^{K \times K}$ is positive definite: $\bar G \succ 0$.
\end{assumption}

Note that Assumption~\ref{asm:nondeg_gauge} does not follow from Assumption~\ref{asm:nondeg}, as any feature that is constant in $y$ given $x$ is non-degenerate pooled but vanishes after within-prompt centering. 
The \textit{centered linear reliance} is
\begin{equation*}
    g_\mathrm{cent}(\tilde R;\, \muaud) = \big(\mathbb{E}_{\muaud}[\bar\Phi \bar\Phi^\top]\big)^{-1} \mathbb{E}_{\muaud}[\bar\Phi\, \tilde R^\mathrm{cent}] \;\in\; \mathbb{R}^K.
\end{equation*}
For near-collinear or learned feature sets, $g_\mathrm{cent}$ should be read as a ridge-regularized estimate.

\paragraph{Gauge invariance.} 
Under $\tilde R \mapsto \tilde R + b(x)$ for any prompt-only $b$, $\mathbb{E}[\tilde R + b \mid x] = \mathbb{E}[\tilde R \mid x] + b(x)$, so $\tilde R^\mathrm{cent}$ is unchanged and $\bar\Phi$ does not involve $\tilde R$. Therefore $g_\mathrm{cent}(\tilde R + b;\,\muaud) = g_\mathrm{cent}(\tilde R;\,\muaud)$. The KL-regularized optimum shares this invariance, so $g_\mathrm{cent}$ is a reliance statistic on the same gauge equivalence class the policy respects.

\paragraph{Relation to $g$.} 
The pooled second moments decompose as
\begin{equation*}
    \mathbb{E}[\Phi\Phi^\top] = \mathbb{E}[\bar\Phi\bar\Phi^\top] + \mathbb{E}\!\big[\mathbb{E}[\Phi|x]\mathbb{E}[\Phi|x]^\top\big], \qquad \mathbb{E}[\Phi\tilde R] = \mathbb{E}[\bar\Phi\,\tilde R^\mathrm{cent}] + \mathbb{E}\!\big[\mathbb{E}[\Phi|x]\mathbb{E}[\tilde R|x]\big]
\end{equation*}
so $g$ is a matrix-weighted combination of the within-prompt regression ($g_\mathrm{cent}$) and the between-prompt regression of conditional means. 
The two statistics coincide only in degenerate cases (e.g., vanishing between-prompt covariation) and can differ substantially when $\Phi$ carries strong prompt-level structure.
This structure is plausible for length, where prompt difficulty drives average response length.

\paragraph{Canonical mitigation and Lemma~\ref{lem:single_axis_identity} analogue.} 
The centered canonical mitigation is
\begin{equation*}
    M_i^\mathrm{cent}(\tilde R)(x,y) = \tilde R(x,y) - g_{\mathrm{cent},i}(\tilde R;\muaud)\, \phi_i(x,y).
\end{equation*}
\begin{lemma}[Centered single-axis identity at $\muaud$]\label{lem:centered_single_axis_identity}
For the centered canonical mitigation $M_i^\mathrm{cent}$, $g_{\mathrm{cent},i}(M_i^\mathrm{cent}(\tilde R);\muaud) = 0$ and $g_{\mathrm{cent},j}(M_i^\mathrm{cent}(\tilde R);\muaud) = g_{\mathrm{cent},j}(\tilde R;\muaud)$ for $j \neq i$.
\end{lemma}
\begin{proof}
Direct computation in the centered inner product, using that $\mathbb{E}_{\muaud}[\bar\Phi\bar\phi_i]$ is the $i$-th column of $\mathbb{E}_{\muaud}[\bar\Phi\bar\Phi^\top]$.
\end{proof}
The variant $\tilde R - g_{\mathrm{cent},i}\bar\phi_i$ differs from $M_i^\mathrm{cent}(\tilde R)$ by $g_{\mathrm{cent},i}\,\mathbb{E}[\phi_i\mid x]$, a prompt-only function. 
Given the above-stated gauge invariance, the two variants are policy-equivalent and produce identical $g_\mathrm{cent}$.

\paragraph{Framework transfer.} 
Definition~\ref{def:linear_reliance} ports to $g_\mathrm{cent}$ verbatim, and Definition~\ref{def:single_axis_mitigation} ports as the requirement $|g_{\mathrm{cent},i}(\tilde R'; \muaud)| < |g_{\mathrm{cent},i}(\tilde R; \muaud)|$. 
\Cref{lem:centered_single_axis_identity} shows $M_i^\mathrm{cent}$ qualifies.
Definition~\ref{def:all_regimes} (R0-R3) port unchanged because their criteria reference policy-induced expectations $\Delta_j(\pi,\pi')$ and $\Delta J$, not the reliance statistic.

\paragraph{Gauge invariance of the regime classification.} 
The \Cref{sec:3_2_linearfix} obstruction is that under $\tilde R \mapsto \tilde R + b(x)$ for prompt-only $b$, $M_i(\tilde R + b)$ differs from $M_i(\tilde R) + b$ by $(g_i(\tilde R) - g_i(\tilde R+b))\phi_i$, so $\pi^\star_\beta(M_i(\tilde R+b)) \neq \pi^\star_\beta(M_i(\tilde R))$ and $\Delta_j, \Delta J$ can change. 
Replacing $M_i$ by $M_i^\mathrm{cent}$ removes the obstruction. 
By gauge invariance of $g_\mathrm{cent}$, $g_{\mathrm{cent},i}(\tilde R + b) = g_{\mathrm{cent},i}(\tilde R)$, so $M_i^\mathrm{cent}(\tilde R + b) = M_i^\mathrm{cent}(\tilde R) + b$. 
Since the KL-regularized optimum is invariant under prompt-only shifts,
\begin{equation*}
    \pi^\star_\beta(\tilde R + b) = \pi^\star_\beta(\tilde R), \qquad \pi^\star_\beta(M_i^\mathrm{cent}(\tilde R + b)) = \pi^\star_\beta(M_i^\mathrm{cent}(\tilde R)),
\end{equation*}
hence $\Delta_j(\pi,\pi')$ and $\Delta J$ are gauge-invariant and the R0-R3 regime is unchanged. 
R4 (see Definition~\ref{def:R4}) also ports verbatim, as $g_\mathrm{cent}$ retains its $\muaud$-dependence and only the gauge of $\tilde R$ is factored out.
Thus, $M_i^{\mathrm{cent},\muaud}$ depends on $\muaud$ in the same sense as $M_i^{\muaud}$ in Definition~\ref{def:R4} and audit-distribution sensitivity is defined identically.

Note that $M_i$ and $M_i^\mathrm{cent}$ applied to the same proxy $\tilde R$ may still occupy different R0-R3 regimes, as different mitigations induce different policies.
However, this difference reflects the choice of reliance estimator as a degree of freedom in the mitigation pipeline, independent of gauge. 
Auditing a proxy under both statistics localizes the targeted reliance, as a feature on which $|g_i|$ is large but $|g_{\mathrm{cent},i}|$ is small flags reliance that lives in prompt-level structure the optimizing policy averages over rather than acts on.

\subsection{Non-Vacuity of the Regime Taxonomy via Linear-Gaussian Instantiation}
\label{app:nonsuck}

This appendix exhibits closed-form instantiations of the regimes and gaps introduced in \Cref{sec:bias_sub}, establishing that the taxonomy is non-vacuous within the assumptions of our framework. 
The constructions are not intended as empirical models of RLHF and we show the empirical instantiations in \Cref{sec:4_evidence,app:AITA_syclength}.

The construction below satisfies Assumptions~\ref{asm:reg} (regularity), \ref{asm:nondeg} (non-degeneracy of the Gram at $\muaud$), and \ref{asm:featreal} (feature realizability), as well as Gram non-degeneracy at $\mu_{\pi^\star}$. We verify each below.

\paragraph{Shared construction.} 
We work with a trivial prompt space ($\mathcal{X}$ a singleton, suppressed in notation) and continuous response space $\mathcal{Y} = \mathbb{R}^3$. 
The feature map is $\Phi(y) = (\phi_1, \phi_2, \phi_3)$ with $\phi_k(y) = y_k$. 
We fix true and proxy rewards
\begin{equation*}
    R(y) = w\, \phi_3(y), \qquad \tilde R(y) = a\, \phi_1(y) + b\, \phi_2(y) + w\, \phi_3(y),
\end{equation*}
with $w > 0$ and $(a, b) \in \mathbb{R}^2$. 
By Definition~\ref{def:spur_v_struct}, $\Phi_\mathrm{sp} = \{\phi_1, \phi_2\}$ and $\Phi_\mathrm{struct} = \{\phi_3\}$ by construction. 
Reference policy and audit distribution are zero-mean Gaussians on $\mathbb{R}^3$,
\begin{equation*}
    \pi_\mathrm{ref} = \mathcal{N}(0, I_3), \qquad \muaud = \mathcal{N}(0, \Sigma),
\end{equation*}
where $\Sigma$ is positive definite with $\Sigma_{kk} = 1$ and off-diagonal $\Sigma_{12} = \rho \in (-1, 1)$, $\Sigma_{13} = \Sigma_{23} = 0$. 
Fix the KL parameter $\beta > 0$ and target the spurious feature $i = 1$ throughout.
We treat $\mathbb{R}^3$ as the formal response space and restrict $\pi_\mathrm{ref}$ (and $\muaud$ in the coupled case below) to the box $\mathcal{Y}_T$ as is done in \Cref{lem:box_truncation}. 
Every closed form below is the $T \to \infty$ limit of \Cref{lem:box_truncation}(ii), every strict sign below holds for all $T$ above a threshold depending only on the finitely many parameters of the instance, and the exact zeros $\Delta_2 = 0$ at $\rho = 0$ and $\Delta J = 0$ at $\sigma_{13} = 0$ hold for every $T$ by \Cref{lem:box_truncation}(iii), since $\Sigma_{23} = 0$ throughout makes $y_2$ independent of $(y_1, y_3)$ when $\rho = 0$ and $y_3$ independent of $(y_1, y_2)$ when $\sigma_{13} = 0$, while $\tilde R$ and $M_1(\tilde R)$ differ only in the coefficient of $\phi_1$. 
In the coupled case the audit Gram is the covariance of $\mathcal{N}_T(0, \Sigma)$, and $g(\tilde R; \muaud) = (a, b, w)$ remains exact because $g$ recovers the coefficient vector of any $\bar R \in \mathrm{span}(\Phi)$ at every non-degenerate measure.

\paragraph{Verification of assumptions.} 
Assumption~\ref{asm:reg} holds under the box truncation as defined and used by \Cref{lem:box_truncation}(i). Assumption~\ref{asm:nondeg} holds because $\mathbb{E}_{\muaud}[\Phi\Phi^\top] = \Sigma \succ 0$. 
Assumption~\ref{asm:featreal} holds because $\mathcal{Y} = \mathbb{R}^3$ has full support under any positive-definite Gaussian, so for $\muaud$-a.e.\ $y$ and any $i$, varying $y_i$ while holding $y_{j \neq i}$ fixed remains in the support. 
Gram non-degeneracy at $\mu_{\pi^\star}$ holds for every $T$ by \Cref{lem:box_truncation}(i), and in the limit $\mu_{\pi^\star}$ is itself Gaussian with positive-definite covariance.

\paragraph{Closed forms for the optimum.} 
For any linear reward $\bar R(y) = c^\top y$ with $c \in \mathbb{R}^3$, the KL-regularized optimum against $\pi_\mathrm{ref}$ satisfies $\pi^\star_\beta(\bar R) = \mathcal{N}(c/\beta, I_3)$ by completing the square in the softmax form of \Cref{equ:softmaxsolution}.
Two consequences used throughout: 
(i) the policy-induced first moment is $\mathbb{E}_{\mu_{\pi^\star}}[\phi_k] = c_k/\beta$, so 
$\Delta_j(\pi, \pi') = (c'_j - c_j)/\beta$ for any pair of linear rewards with coefficients $c, c'$; 
(ii) the plain return is $J(\pi^\star_\beta(\bar R), R) = \sum_k w_k c_k / \beta$ where $w_k$ are the coefficients of $R$, so under our $R = w\phi_3$, only the $\phi_3$-coefficient of the proxy contributes to $\Delta J$.

\paragraph{Linear reliance and canonical mitigation.} 
The reliance vector at $\muaud$ is $g(\tilde R; \muaud) = \Sigma^{-1} \mathbb{E}_{\muaud}[\Phi \tilde R]$. 
Since $\tilde R$ is linear in $\Phi$ with coefficient vector $(a, b, w)$ and $\mathbb{E}_{\muaud}[\Phi \Phi^\top] = \Sigma$, we have $g(\tilde R; \muaud) = (a, b, w)$ exactly. 
The canonical mitigation $M_1(\tilde R)(y) = \tilde R(y) - a\, \phi_1(y) = b\,\phi_2 + w\,\phi_3$ has reliance $g(M_1(\tilde R); \muaud) = (0, b, w)$ by Lemma~\ref{lem:single_axis_identity}.

\paragraph{Measurement-vs-optimization gap.} 
We progress through three settings of $(\pi_\mathrm{ref}, \muaud)$ to isolate where the gap of \Cref{sec:3_2_linearfix} can manifest in this construction.

\textit{Isotropic reference and audit ($\pi_\mathrm{ref} = \muaud = \mathcal{N}(0, I_3)$):} the mitigated proxy induces $\pi' = \pi^\star_\beta(M_1(\tilde R)) = \mathcal{N}((0, b, w)/\beta, I_3)$. The policy-side Gram is $\mathbb{E}_{\mu_{\pi'}}[\Phi\Phi^\top] = I_3 + (0, b, w)^\top (0, b, w)/\beta^2$, and 
\begin{equation*}
    g_1(M_1(\tilde R); \mu_{\pi'}) = \big(\mathbb{E}_{\mu_{\pi'}}[\Phi\Phi^\top]^{-1}\, \mathbb{E}_{\mu_{\pi'}}[\Phi\, M_1(\tilde R)]\big)_1 = 0,
\end{equation*}
since $\mathbb{E}_{\mu_{\pi'}}[\phi_1 M_1(\tilde R)] = 0$ in this case.

\textit{Isotropic reference, correlated audit ($\pi_\mathrm{ref} = \mathcal{N}(0, I_3)$, $\muaud = \mathcal{N}(0, \Sigma)$ with $\Sigma_{12} = \rho \neq 0$):} the canonical mitigation is $M_1(\tilde R) = \tilde R - g_1(\tilde R; \muaud)\phi_1$ with $g_1(\tilde R; \muaud) = a$ unchanged (since $\Sigma_{13} = 0$ keeps the $\phi_3$ contribution clean and the $(\phi_1, \phi_2)$ block of $\Sigma^{-1}$ recovers $a$ for a linear proxy). Evaluating reliance at $\mu_{\pi'} = \mathcal{N}((0,b,w)/\beta, I_3)$ gives Gram $I_3 + (0,b,w)^\top(0,b,w)/\beta^2$, diagonal in the $\phi_1$ row, so $g_1(M_1(\tilde R); \mu_{\pi'}) = 0$ here as well.

\begin{figure}
    \centering
    \includegraphics[width=0.99\linewidth]{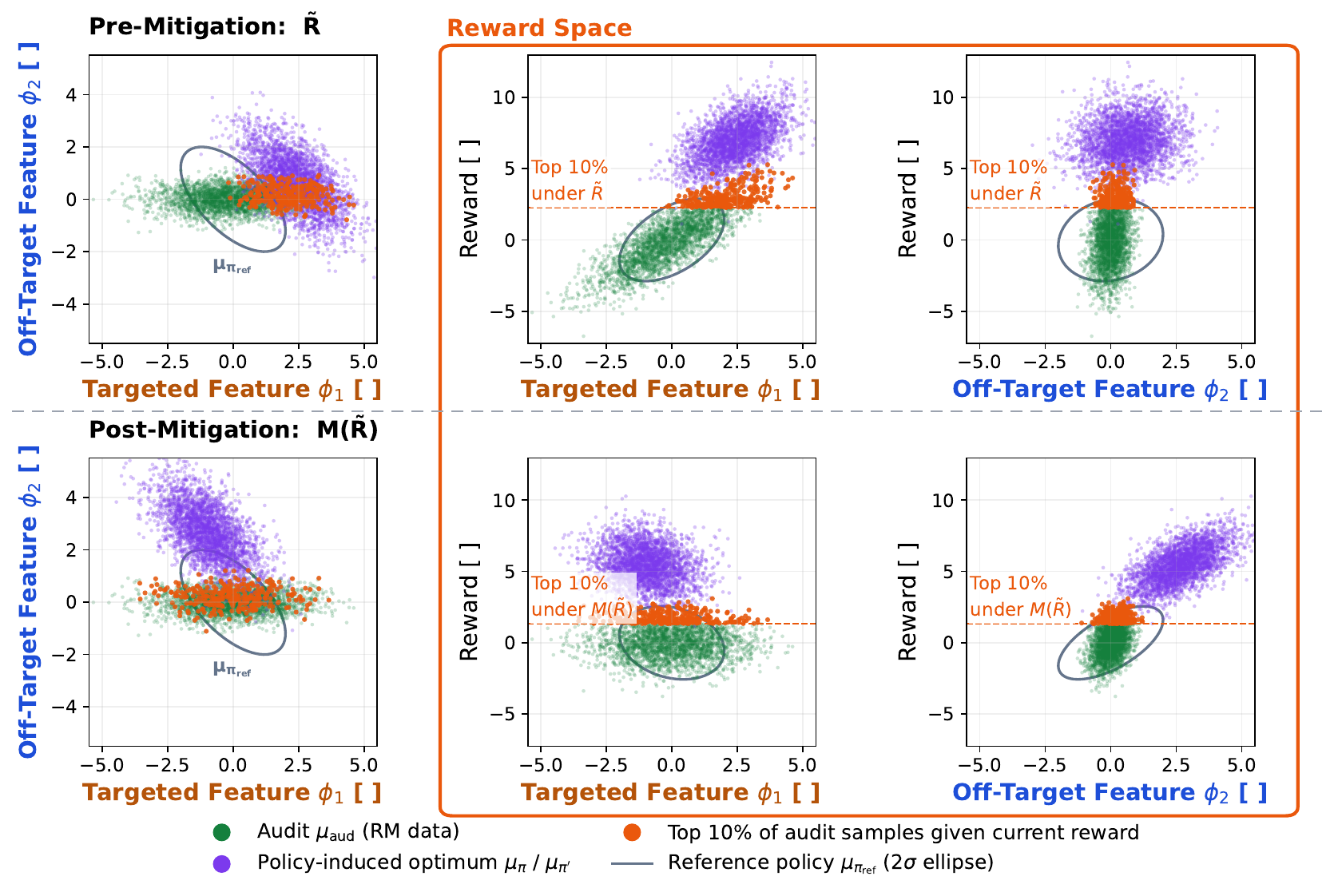}
    \caption{\textbf{Harmful bias substitution (R1) in the linear Gaussian instantiation} of \Cref{app:nonsuck} with true reward $R(y) = w\, \phi_3(y)$, proxy reward $\tilde R(y) = a\, \phi_1(y) + b\, \phi_2(y) + w\, \phi_3(y)$, $\pi_\mathrm{ref} = \mathcal{N}(0, \Sigma)$ with $\rho = \Sigma_{12}$ and $\sigma_{13} = \Sigma_{13}$, and $\muaud = \mathcal{N}(0, \Sigma_{\mathrm{aud}})$.
    After applying a canonical mitigation $M_1$ targeting $\phi_1$, there is no $\phi_1$-reward correlation under the audit distribution (middle column), and the reliance $g_2$ on the off-target feature $\phi_2$ at $\muaud$ is unchanged as intended (\Cref{lem:single_axis_identity}. 
    The audit cloud in the right column  narrows because removing $a\phi_1$ lowers the reward variance (\Cref{app:A3}).
    The samples with high rewards on $\muaud$ are not concentrated at high $\phi_1$ anymore (left column).
    However, the optimal policy drifts along $\phi_2$ by $\Delta_2 = -\rho a/\beta$ (left column) and true return falls by $\Delta J = -\sigma_{13} a w/\beta$ through the unplotted structural $\phi_3$.
    The measurement-versus-optimization gap is visible as the $\phi_2$-reward covariance of the policy-induced cloud (bottom right, from $\rho a + b$ to $b$) and the resulting drift along $\phi_2$.}
    \label{fig:pointclouds_substitution}
\end{figure}

\textit{Coupled correlated case ($\pi_\mathrm{ref} = \muaud = \mathcal{N}(0, \Sigma)$ with $\Sigma_{12} = \rho \neq 0$):} the policy-side Gram now inherits the audit cross-correlation, $\mu_{\pi'} = \mathcal{N}(\Sigma c / \beta, \Sigma)$ for $c = (0, b, w)$, and one might expect a non-trivial $g_1$. It does not arise, for a structural reason: for any linear proxy $\bar R = c^\top y$ on $\Phi(y) = y$ and any non-degenerate measure $\mu = \mathcal{N}(m, V)$,
\begin{equation*}
    g(\bar R; \mu) = \big(\mathbb{E}_\mu[\Phi\Phi^\top]\big)^{-1} \mathbb{E}_\mu[\Phi\, \bar R] = (V + m m^\top)^{-1}(V + m m^\top)\, c = c,
\end{equation*}
so $g$ recovers the coefficient vector exactly and is invariant to $\mu$. Applied to $c = (0, b, w)$, this gives $g_1(M_1(\tilde R); \mu_{\pi'}) = 0$ in the coupled case as well, and the $g$-level gap is degenerate throughout the linear-Gaussian closed form.

The substantive gap manifests instead through the policy-induced first moment. 
In the coupled correlated case, $\mathbb{E}_{\mu_{\pi'}}[\Phi] = \Sigma c / \beta$, so
\begin{equation*}
    \mathbb{E}_{\mu_{\pi'}}[\phi_1] \;=\; (\Sigma c)_1 / \beta \;=\; \rho\, b / \beta \;\neq\; 0
\end{equation*}
whenever $\rho \neq 0$ and $b \neq 0$, even though $g_1$ vanishes at every measure. 
The mechanism is the cross-correlation $\rho$: at $\muaud$ the projection orthogonalizes $M_1(\tilde R)$ against $\phi_1$ in the Gram-inner-product sense, but the policy under $M_1(\tilde R)$ shifts the mean of $(\phi_2, \phi_3)$ by $(b, w)/\beta$, and any $\phi_1$-$\phi_2$ coupling at $\muaud$ propagates into $\phi_1$ first-moment drift at the policy distribution. 
This is the form of the gap the regime taxonomy of \Cref{sec:3_3_bias_substitution} is sensitive to, since Definition~\ref{def:all_regimes} is stated in terms of $\Delta_j$ and $\Delta J$ rather than $g_i$ at $\mu_{\pi^\star}$.
\Cref{fig:pointclouds_substitution} shows this gap in sampled form with point clouds, with $\Delta_2$ visible as the drift of the policy-induced cloud along $\phi_2$ while every audit-side panel registers only the intended removal of $\phi_1$.

\textbf{Remark (non-linear proxies activate the $g$-level gap).} 
The linear-Gaussian construction is structurally incapable of exhibiting a non-degenerate $g$-level gap, because the OLS-style reliance map $g(\bar R; \mu)$ recovers $c$ exactly for any linear $\bar R = c^\top y$, independent of $\mu$. 
Activating a non-trivial $g$-level gap requires breaking exact representability of $\tilde R$ in $\mathrm{span}(\Phi)$ in a way that couples to moments of $\muaud$ that vary across audit distributions. 
Augmenting $\tilde R$ with a non-linear term outside $\mathrm{span}(\Phi)$ (e.g., an interaction $\kappa\phi_1\phi_2$ at non-symmetric $\muaud$, or a higher-order term whose relevant cross-moments differ across $\muaud^{(1)}$ and $\muaud^{(2)}$) makes $g(\tilde R; \muaud)$ depend on those moments. 
Under such an augmentation, $M_1^{\muaud^{(1)}} \neq M_1^{\muaud^{(2)}}$ as operators and the operator-level audit sensitivity of Definition~\ref{def:R4} becomes instantiable with $\pi_\mathrm{ref}$ held fixed across audit distributions. 
While we establish non-vacuity with the linear-Gaussian, \Cref{app:phasediagrams} provides the closed-form non-linear instantiation.
Real reward models are non-linear in any tractable feature basis, so the operator-level reading of Definition~\ref{def:R4} is the relevant one in deployment, beyond the controlled construction.

\paragraph{Regimes R0-R3.} 
We work in the coupled case $\pi_\mathrm{ref} = \muaud = \mathcal{N}(0, \Sigma)$ for the same reason as the gap analysis above: under decoupled isotropic $\pi_\mathrm{ref}$ the post-mitigation policy is $\mathcal{N}((0,b,w)/\beta, I_3)$, which gives $\Delta_2 = 0$ and $\Delta J = 0$ identically and collapses the regime taxonomy to a single regime. 
The coupled case is the minimal setting in which the five regimes are jointly distinguishable in this construction. The limitation of relying on $\pi_\mathrm{ref} = \muaud$ noted under R4 below applies symmetrically here.
Under this coupling, the policy-mean for any linear reward $\bar R = c^\top y$ is $\Sigma c / \beta$.
The pre-mitigation policy is $\pi = \pi^\star_\beta(\tilde R) = \mathcal{N}(\Sigma(a,b,w)^\top/\beta, \Sigma)$ and the post-mitigation policy is $\pi' = \pi^\star_\beta(M_1(\tilde R)) = \mathcal{N}(\Sigma(0,b,w)^\top/\beta, \Sigma)$.
The off-target spurious-feature drift and true-reward change evaluate to
\begin{equation*}
    \Delta_2(\pi, \pi') = -\rho a / \beta, \qquad \Delta J = -\Sigma_{13} \cdot a w / \beta = 0,
\end{equation*}
the latter because $\Sigma_{13} = 0$ in our parameterization. 
This makes $\Delta J = 0$ in the strict $\Sigma_{13} = 0$ case and redirection onto $\phi_2$ governed entirely by $\rho a$. 
To populate all five regimes we relax $\Sigma_{13}$ as a free parameter $\sigma_{13} \in (-1, 1)$ (with $\Sigma$ remaining positive definite), giving $\Delta J = -\sigma_{13}\, a w / \beta$. 
The five regimes correspond to:
\begin{itemize}
    \item \textit{R0 (successful mitigation):} $\rho = 0$, $\sigma_{13} < 0$, $a, w > 0$. 
    Then $\Delta_2 = 0$ and $\Delta J = -\sigma_{13} aw/\beta > 0$. 
    Removing spurious $\phi_1$-reliance improves true reward because $\phi_1$ was anti-correlated with $\phi_3$ at $\muaud$, so the unmitigated proxy was pushing the policy away from high-$\phi_3$ regions.
    \item \textit{R0\textsubscript{cont} (contaminated success):} $\rho \neq 0$, $\sigma_{13} < 0$, $a, w > 0$. Then $\Delta_2 = -\rho a/\beta \neq 0$ and $\Delta J = -\sigma_{13} aw/\beta > 0$. 
    True reward improves as in R0 (driven by $\sigma_{13} < 0$), but the cross-correlation $\rho$ couples $\phi_1$-reduction to first-moment drift on $\phi_2$, so the substitution mechanism is active despite improvement.
    \item \textit{R1 (bias substitution):} $\rho \neq 0$, $\sigma_{13} = 0$, $a > 0$.
    Then $\Delta_2 = -\rho a / \beta \neq 0$ and $\Delta J = 0$ (neutral substitution). 
    Setting $\sigma_{13} > 0$ instead gives $\Delta J < 0$ (harmful substitution), the case shown in \Cref{fig:pointclouds_substitution}.
    \item \textit{R2 (overcorrection):} $\rho = 0$, $\sigma_{13} > 0$, $a, w > 0$. Then $\Delta_2 = 0$ and $\Delta J = -\sigma_{13} a w/\beta < 0$. 
    The mechanism is audit-side cross-correlation between $\phi_1$ and the structural $\phi_3$: removing $a\phi_1$ shifts the policy mean of $\phi_3$ by $-\sigma_{13}a/\beta$, lowering true reward without redirecting pressure onto $\phi_2$. 
    The rescalability test of \Cref{app:a8-takeaways} gives $\Delta J(c) = -c\,\sigma_{13} a w/\beta$ for partial mitigation $M_1^c(\tilde R) = \tilde R - c\,a\phi_1$ in the limit of \Cref{lem:box_truncation}(ii), monotone in $c \in (0, 1)$, so no partial mitigation recovers the unmitigated return and the construction sits on the non-rescalable side of R2
    (driven here by reference-policy coupling $\Sigma_{13} \neq 0$ rather than the named target-misspecification mechanism, which would require $\phi_1 \in \Phi_\text{struct}$).
    \item \textit{R3 (silent non-op):} $\rho = 0$, $\sigma_{13} = 0$. Then $\Delta_2 = 0$ and $\Delta J = 0$ regardless of $a$. The mitigation alters $\tilde R$ along an axis orthogonal to the true-reward-relevant directions of the policy.
\end{itemize}
R0\textsubscript{cont} and harmful R1 are realized on open sets of parameters. 
R0, R2, R3 and neutral R1 require the structural zeros $\rho = 0$ and/or $\sigma_{13} = 0$ of the shared construction, are open within those subfamilies, and have positive measure under any band $\varepsilon > 0$ of \Cref{app:epsilonbands}.
\paragraph{Audit-distribution sensitivity (R4).}
Definition~\ref{def:R4} (varying $\muaud$ with $\pi_\mathrm{ref}$ held fixed) is degenerate under linearity: in the linear-Gaussian setting $g(\tilde R; \muaud) = (a, b, w)$ at every non-degenerate $\muaud$, so $M_1^{\muaud^{(1)}}(\tilde R) = M_1^{\muaud^{(2)}}(\tilde R) = \tilde R - a\phi_1$ as operators, whatever the difference between $\muaud^{(1)}$ and $\muaud^{(2)}$.
Linearity thus characterizes the boundary at which Definition~\ref{def:R4} activates: any operator-level instantiation requires the non-linear extension of the preceding remark 
(instantiated closed-form in \Cref{app:phasediagrams}).

The closed form does realize the regime transition jointly with $\pi_\mathrm{ref}$ under the coupling $\pi_\mathrm{ref} = \muaud$ (the audit distribution serving as the rollout reference).
Hold $(\tilde R, R, \Phi_\mathrm{sp}, \beta)$ and the operator $M_1$ fixed with $a, b, w > 0$, and take two pairs $(\muaud^{(\ell)}, \pi_\mathrm{ref}^{(\ell)}) = \big(\mathcal{N}(0, \Sigma^{(\ell)}), \mathcal{N}(0, \Sigma^{(\ell)})\big)$ for $\ell = 1, 2$ with $\Sigma^{(1)}_{12} = \rho^{(1)} \neq 0$, $\sigma_{13}^{(1)} = 0$, and $\Sigma^{(2)}_{12} = \rho^{(2)} \neq 0$, $\sigma_{13}^{(2)} < 0$ (each $\Sigma^{(\ell)}$ positive definite, all other entries equal across the two). 
Then
\begin{equation*}
    \Delta_2^{(\ell)} = -\rho^{(\ell)} a / \beta, \qquad \Delta J^{(\ell)} = -\sigma_{13}^{(\ell)} a w / \beta,
\end{equation*}
so $\ell = 1$ realizes R1 (neutral substitution: $\Delta_2 \neq 0$, $\Delta J = 0$) and $\ell = 2$ realizes R0\textsubscript{cont} ($\Delta_2 \neq 0$, $\Delta J > 0$): the audit distribution selects the regime, with $\pi_\mathrm{ref}$ entering jointly via the coupling. 
For an instantiation satisfying Definition~\ref{def:R4} ($\pi_{\mathrm{ref}}$ fixed across audit distributions), \Cref{app:phasediagrams} provides the closed-form instantiation.

\paragraph{Norm increase under $M_i$.} 
The scale identity from \Cref{app:A3} reads $\|\tilde R'\|^2_{L^2(\muaud)} = \|\tilde R\|^2_{L^2(\muaud)} - g_1[g_1 G_{11} + 2(G_{12} g_2 + G_{13} g_3)]$. 
With $G = \Sigma$ in our parameterization, $G_{11} = 1$, $G_{12} = \rho$, $G_{13} = \sigma_{13}$, and $(g_1, g_2, g_3) = (a, b, w)$, the bracket equals $a + 2\rho b + 2 \sigma_{13} w$. 
Choosing $a = 0.5$, $b = w = 1$, $\rho = -0.6$, $\sigma_{13} = -0.4$ gives bracket $= 0.5 - 1.2 - 0.8 = -1.5$ with bracket sign opposite to $g_1 = 0.5$. The norm change is $-g_1 \cdot (-1.5) = +0.75$, so $\|\tilde R'\|^2 > \|\tilde R\|^2$ strictly. 
Mitigation increases the $L^2$ norm of the proxy because the cross-correlation contribution dominates the diagonal term, exactly the failure mode flagged in \Cref{app:A3}.

\subsection{Epsilon-Banded Regime Classification}
\label{app:epsilonbands}

Finite-sample classification of mitigations into the regimes of Definition~\ref{def:all_regimes} requires bands, since exact equality on $\Delta_j$ and $\Delta J$ is a measure-zero event under any non-degenerate sampling design.

\paragraph{Banded definitions.}
Fix tolerances $\varepsilon_j \geq 0$ for each $\phi_j \in \Phi_\mathrm{sp} \setminus \{\phi_i\}$ and $\varepsilon_J \geq 0$. With $\pi, \pi'$ as in Definition~\ref{def:all_regimes}, define
\begin{align*}
    \mathrm{R0}_\varepsilon: &\quad \Delta J > \varepsilon_J \text{ and } |\Delta_j(\pi,\pi')| \leq \varepsilon_j \text{ for all } \phi_j \in \Phi_\mathrm{sp} \setminus \{\phi_i\}, \\
    \mathrm{R0}_{\varepsilon,\mathrm{cont}}: &\quad \Delta J > \varepsilon_J \text{ and } |\Delta_j(\pi,\pi')| > \varepsilon_j \text{ for some } \phi_j \in \Phi_\mathrm{sp} \setminus \{\phi_i\}, \\
    \text{R1}_\varepsilon: &\quad |\Delta_j(\pi, \pi')| > \varepsilon_j \text{ for some } \phi_j \in \Phi_\mathrm{sp} \setminus \{\phi_i\}, \quad \text{and} \quad \Delta J \leq \varepsilon_J, \\
    \text{R2}_\varepsilon: &\quad |\Delta_j(\pi, \pi')| \leq \varepsilon_j \text{ for all } \phi_j \in \Phi_\mathrm{sp} \setminus \{\phi_i\}, \quad \text{and} \quad \Delta J < -\varepsilon_J, \\
    \text{R3}_\varepsilon: &\quad |\Delta_j(\pi, \pi')| \leq \varepsilon_j \text{ for all } \phi_j \in \Phi_\mathrm{sp} \setminus \{\phi_i\}, \quad \text{and} \quad |\Delta J| \leq \varepsilon_J.
\end{align*}
The R1 sub-cases port unchanged: neutral and harmful substitution correspond to $|\Delta J| \leq \varepsilon_J$ and $\Delta J < -\varepsilon_J$ respectively. The spurious-redirection-with-improvement corner of \Cref{app:A2} is regime $\mathrm{R0}_{\varepsilon,\mathrm{cont}}$ per the definition above. 
Definition~\ref{def:R4} ports verbatim, since it is transversal to the regime classification and unaffected by the choice of bands.

\paragraph{Nested limit.}
Taking $\varepsilon_j, \varepsilon_J \to 0$ recovers Definition~\ref{def:all_regimes} pointwise: $\mathrm{R0}_\varepsilon \to \{\text{no redirection},\,\Delta J > 0\}$, $\mathrm{R0}_{\varepsilon,\mathrm{cont}} \to \{\text{redirection},\,\Delta J > 0\}$, $\text{R1}_\varepsilon \to \{\text{redirection},\,\Delta J \leq 0\}$, $\text{R2}_\varepsilon \to \{\text{no redirection},\,\Delta J < 0\}$, $\text{R3}_\varepsilon \to \{\text{no redirection},\,\Delta J = 0\}$, with the disjunctive condition $|\Delta_j| > \varepsilon_j$ for some $j$ converging to $\Delta_j \neq 0$ for some $j$ as $\varepsilon_j \to 0$.

\paragraph{Default choice: noise-floor bands.}
We default to noise-floor bands, in which $\varepsilon_j = z_{\alpha/2} \cdot \widehat{\mathrm{SE}}(\widehat{\Delta}_j)$ under the empirical sampling design (e.g., the standard error of a fixed-effect coefficient in a mixed linear model) and $\varepsilon_J = z_{\alpha/2} \cdot \widehat{\mathrm{SE}}(\widehat{\Delta J})$ for a significance level $\alpha$ chosen and reported by the practitioner. Setting $\varepsilon_j = \widehat{\mathrm{SE}}$ without the $z_{\alpha/2}$ factor corresponds to approximately 68\% confidence and should not be conflated with a standard significance test; regime assignments should always state the chosen $\alpha$. An effect-size-relative alternative ($\varepsilon_j$ as a fraction of unmitigated $\mathbb{E}_{\mu_\pi}[\phi_j]$, $\varepsilon_J$ as a fraction of unmitigated $J$) is appropriate for cross-proxy comparisons; the two can disagree on borderline cases and the choice should be stated. 
When $|\Phi_{\mathrm{sp}} \setminus \{\phi_i\}| > 1$, the conservative reading (declare redirection if any axis exceeds its $\varepsilon_j$ band, no multiple-testing correction) inflates the false-positive rate of the redirection criterion. 
Because redirection is the discriminating predicate for both $\mathrm{R0}_{\varepsilon,\mathrm{cont}}$ (against $\mathrm{R0}_\varepsilon$) and $\mathrm{R1}_\varepsilon$ (against $\mathrm{R3}_\varepsilon \cup \mathrm{R2}_\varepsilon$), this inflation pulls cases out of $\mathrm{R0}_\varepsilon$ into $\mathrm{R0}_{\varepsilon,\mathrm{cont}}$ when $\Delta J > \varepsilon_J$, and out of $\mathrm{R3}_\varepsilon$ and $\mathrm{R2}_\varepsilon$ into $\mathrm{R1}_\varepsilon$ when $\Delta J \le \varepsilon_J$. 
This choice is deliberate to prioritize sensitivity to substitution over specificity in either $\Delta J$ regime, and the neutral sub-case absorbing near-zero improvements ($\Delta J \in (0, \varepsilon_J]$) pushes in the same direction.

\paragraph{Mapping wild classifications onto the bands.}
The R0-R4 calls in \Cref{sec:3_3_bias_substitution,sec:4_evidence} are made under the banded reading, with $\varepsilon_j, \varepsilon_J$ inferred from the uncertainty each cited work reports.
Works only giving point estimates receive a sign-only qualitative call and these are marked as such.
The data-side length-balancing baseline \citep{singhal2024long,bu2025beyond} is consistent with the $\mathrm{R2}_\varepsilon$ pattern but not separable from harmful $\mathrm{R1}_\varepsilon$, since no off-target axis is instrumented and no partial-mitigation sweep is reported, so the rescalability test of \Cref{app:a8-takeaways} cannot separate the two origins of R2 there. 
That test is itself a banded operationalization (existence of $c \in (0,1)$ improving $\widehat{J}$ implicitly assumes detection against $\varepsilon_J$).
\citet{fein2026one} observe sign-flipped length reliance of SOTA reward models that also reduces win rate for correct verbose vs incorrect concise math completions.
This effect arises from data-side length curation by the reward-model developers \citep{liu2025skyworkrewardv2scalingpreferencedata,malik2026rewardbench2} instead from the linear probe mitigation operator of \citet{fein2026one}, as it moves the correlation of reward and length (within-prompt) toward zero by construction.
We read their observation as $\text{R2}_\varepsilon$ on the targeted length axis (overcorrection with reduced accuracy) at the reward-model level.
Alternatively, $\mathrm{R1}_\varepsilon$ would also be consistent with the reported statistics but requires an unmeasured off-target axis.
In either case, an undesirable outcome.


\subsection{Phase-Diagram Validation of Regime Transitions Under Quadratic Non-Linearity}
\label{app:phasediagrams}

Extending the linear Gaussian non-vacuity instantiation of \Cref{app:nonsuck}, we use a quadratic non-linear reward structure to empirically show
\begin{itemize}
    \item clear instantiation of all regimes and sub-regimes (R0, R0\textsubscript{cont}, R1\textsubscript{neut}, R1\textsubscript{harm}, R2, R3) with phase-diagrams illustrating how the regime boundaries are outcomes of a single mitigation operator in a controlled setting
    \item that the regimes are not a property of the reward model alone but of $(R, \tilde R, M_i, \muaud)$ as stated in \Cref{sec:bias_sub}.
    \item clear instantiation for the regime transition as described by R4 from only varying the audit distribution $\muaud$ mean
    \item a match of the theoretical predictions connecting the analytic phase-diagram claims to finite-$N$ experiments, strengthening the non-vacuity results. 
\end{itemize}
This generalizes the linear-Gaussian setting, as we show that the regime taxonomy is not a  linear-Gaussian artifact.
For real deployment instantiations, see different experiments in \Cref{app:experiments}.

\paragraph{Setup.}
We extend the setup in \Cref{app:nonsuck}:
We work with a trivial prompt space ($\mathcal{X}$ a singleton, suppressed in notation) and continuous response space $\mathcal{Y} = \mathbb{R}^3$. 
The feature map is $\Phi(y) = (\phi_1, \phi_2, \phi_3)$ with $\phi_k(y) = y_k$. 
We fix true and the preference-generating annotator reward 
\begin{equation*}
    R(y) = w\, \phi_3(y), \qquad R_{\text{anno}}(y) = a\, \phi_1(y) + b\, \phi_2(y) + w\, \phi_3(y) + \gamma\, \phi_1(y)^2,
\end{equation*}
with $w, \gamma > 0$ and $(a, b) \in \mathbb{R}^2$.
By Definition~\ref{def:spur_v_struct}, $\Phi_\mathrm{sp} = \{\phi_1, \phi_2\}$ and $\Phi_\text{struct} = \{\phi_3\}$ by construction. 
For the reward model (learned proxy reward), we use a linear polynomial such that
\begin{equation*}
    \tilde R(y) = \theta_1 \phi_1(y) + \theta_2 \phi_2(y) + \theta_3 \phi_3(y) + \theta_4 \phi_1(y)^2,
\end{equation*}
which is correctly specified using enough samples and the maximum likelihood estimator for Bradley-Terry (BT-MLE) \citep{bradleyterrt1952pairwisecomparisons} as $\hat \theta = (a, b, w, \gamma)$.\footnote{We use the analytical idealization for the correct specification of the reward model for closed-form regime predictions. 
Under realistic RM misspecification, the same regime phenomena persist with additional noise, as documented empirically across five reward models and a non-linear operator in \Cref{app:bon}.}
Reference policy and audit distribution are Gaussians on $\mathbb{R}^3$,
\begin{equation*}
    \pi_\mathrm{ref} = \mathcal{N}(0, \Sigma_{\mathrm{ref}}), \qquad \muaud = \mathcal{N}(m, \Sigma_{\mathrm{ref}}),
\end{equation*}
where $\Sigma_\mathrm{ref}$ is positive definite with
\begin{equation*}
    \Sigma_{\mathrm{ref}} = 
    \begin{pmatrix}
    1 & \rho_{12} & \rho_{13} \\
    \rho_{12} & 1 & 0 \\
    \rho_{13} & 0 & 1
    \end{pmatrix}, \qquad \rho_{12}^2 + \rho_{13}^2 < 1
\end{equation*}
Although this setup ties the covariance of the reference and audit distribution, we show that the mean parameter $m \in \mathbb{R}^3$ of the audit distribution $\muaud$ is sufficient to instantiate R4.
The truncation of \Cref{lem:box_truncation} applies with $\Sigma = \Sigma_\mathrm{ref}$, and its integrability hypothesis holds exactly when $2\gamma/\beta < 1$, since $(\Sigma_\mathrm{ref})_{11} = 1$ makes the $y_1$-marginal of $\pi_\mathrm{ref}$ standard normal. 
The exact zeros at $\rho_{12} = 0$, $\rho_{13} = 0$, and $m_1 = m_1^\star$ used for the regime boundaries hold for every $T$ by \Cref{lem:box_truncation}(iii), as $(\Sigma_\mathrm{ref})_{23} = 0$, the quadratic term lives on $y_1$, and the mitigation changes only the coefficient of $\phi_1$ (at $m_1^\star$ it is the identity).
We fix the KL parameter $\beta > 0$ and target the spurious feature $i = 1$ for mitigation throughout. 

\paragraph{Mitigation observables.}
Following Definition~\ref{def:linear_reliance}, with $\muaud = \mathcal{N}(m, \Sigma_\mathrm{ref})$ and $m = (m_1, 0, 0)$, direct computation of the Gaussian moments gives for the targeted component
\begin{equation*}
    g_1(\theta, m) = \theta_1 + \theta_4\Big[\,2m_1 + \frac{(m_1 - m_1^3)\,s}{1 + m_1^2 s}\Big], \qquad s = (\Sigma_\mathrm{ref}^{-1})_{11} = \frac{1}{1 - \rho_{12}^2 - \rho_{13}^2},
\end{equation*}
so $g_1 = \theta_1$ at $m = 0$ for every $(\rho_{12}, \rho_{13})$, and the audit mean enters only through the quadratic coefficient $\theta_4$.
The resulting single-axis mitigation with strength $c \ge 0$ is then
\begin{equation*}
    M_1(\tilde R(y)) = \tilde R(y) - c g_1 \phi_1(y).
\end{equation*}

We rewrite the reward model as 
\begin{equation*}
    \tilde R(y) = \alpha \cdot y + \frac{1}{2} y^\top H y, \qquad \text{with}  ~ \alpha = (\theta_1, \theta_2, \theta_3), H = \text{diag}(2\theta_4, 0, 0), 
\end{equation*}
reparameterizing the KL-regularized optimum policy as
\begin{equation*}
    \pi^{\star} = \mathcal{N}(\mu^\star, \Sigma^\star) \qquad \text{with}  ~ \Sigma^{\star - 1} = \Sigma_{\mathrm{ref}}^{-1} - \frac{H}{\beta}, \mu^\star = \frac{\Sigma^\star \alpha}{\beta}.
\end{equation*}
Since $H$ is unchanged by the mitigation, $\Sigma^\star$ is also unchanged and only $\alpha$ shifts by $\Delta \alpha = (-c g_1, 0, 0)$, leading to the regime-defining changes in the policy-induced feature expectation and true reward (see \Cref{sec:3_3_bias_substitution}) as
\begin{equation*}
    \Delta_j = (\Sigma^\star \Delta \alpha) / \beta = -\frac{c g_1}{\beta} \Sigma^\star_{j, 1}, \qquad \Delta J = w \Delta_3 = -w \frac{c g_1}{\beta} \Sigma^\star_{3, 1}.
\end{equation*}
Expanding with Sherman-Morrison gives $\Sigma^\star_{j, 1} = \rho_{1j}\,(1 + \kappa)$ with $\kappa = (2\gamma / \beta)/(1 - 2\gamma/\beta) \geq 0$, implying that rescaling $\gamma$ does not flip the sign of $\Delta_j$ for regime classification.
The condition $2\gamma/\beta < 1$ implicit in $\kappa \geq 0$ is also the existence condition for the softmax policy, as $\mathbb{E}_{\pi_\mathrm{ref}}[\exp(\tilde R/\beta)]$ is finite exactly there, and all of the reported settings satisfy this condition.

\paragraph{Simulation parameters.} With this setup, we have the following parameters to tune and induce regime transitions:
\begin{itemize}
    \item $\rho_{12}$ sets redirection for bias substitution and also the sign of $\Delta_2$.
    \item $\rho_{13}$ sets the sign of $\Delta J$ (if the sign of $g_1$ does not change).
    \item $\gamma$ amplifies $\Delta_j$ and $\Delta J$, but does not move boundaries at $m = 0$.
    \item $m_1$ shifts the linear reliance $g_1$ through the quadratic term and flips its sign at the root $m_1^\star$ of $g_1(\theta, m) = 0$ ($m_1^\star \approx -0.35$ at the settings of \Cref{fig:r4_1d}), inducing the transition for R4. 
    At $\gamma = 0$, $g_1 = \theta_1$ for every $m$ and no transition occurs.
\end{itemize}
The KL parameter $\beta$ and the mitigation scale parameter $c$ only set margins and scales.

\paragraph{Experiment Results.} 
Using $N = 10^5$ preference samples, 50 seeds per validation cell, a linearly spaced grid of 81 by 81 values for $\rho_{12}, \rho_{13} \in (-0.9, 0.9)^2$ each, and a (near-numerical precision) $\varepsilon$-bound of $10^{-8}$ (see \Cref{app:epsilonbands}), we instantiate all five regimes R0-R3 and the transition R4, while also validating the above theoretical predictions with finite-$N$ sample experiments. 
\Cref{fig:headline_and_r4} shows the instantiation, \Cref{fig:regimevsnonlin} shows the regime boundary robustness to non-linearity strength, and \Cref{tab:phase-diagram-validation} verifies predictions and measurements. 

\begin{figure}
  \centering
  \begin{subfigure}[t]{0.48\linewidth}
    \centering
    \includegraphics[width=\linewidth]{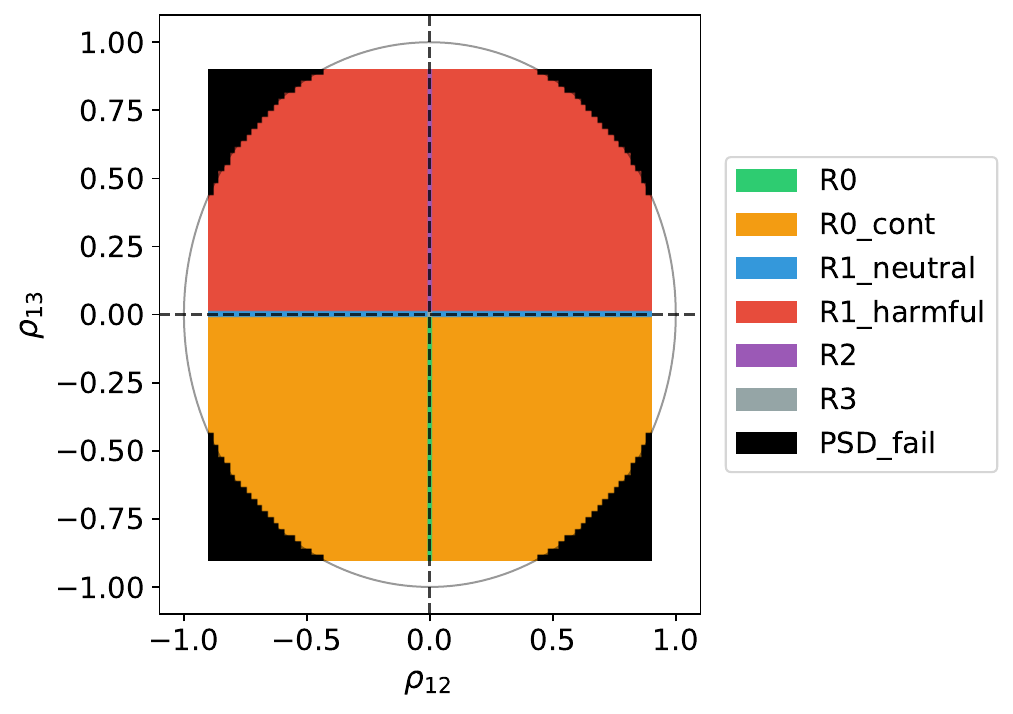}
    \caption{Phase diagram from varying $(\rho_{12}, \rho_{13})$ at fixed $\gamma = 0.4, m = 0, c= 1, \beta = 4$. 
    We color each grid cell with the corresponding regime from \Cref{sec:3_3_bias_substitution}.}
    \label{fig:headline}
  \end{subfigure}
  \hfill
  \begin{subfigure}[t]{0.48\linewidth}
    \centering
    \includegraphics[width=\linewidth]{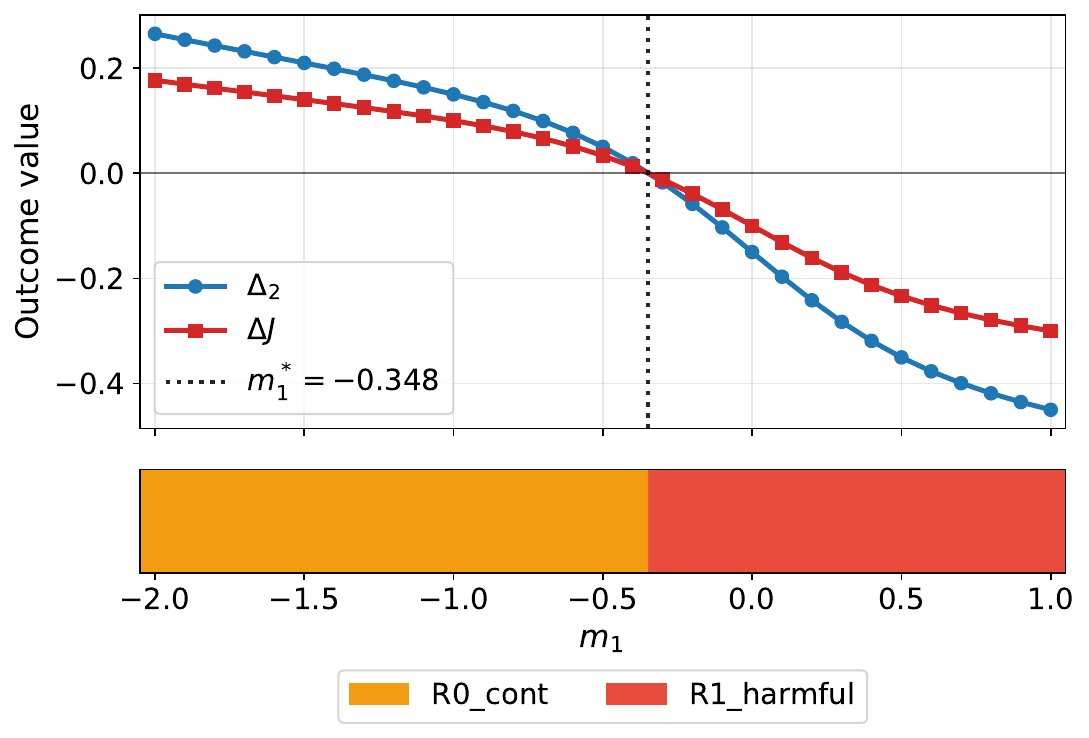}
    \caption{R4 transition of $\Delta_2$ and $\Delta J$ induced by varying the audit distribution $\muaud$ mean $m_1 \in [-2, 1]$ at fixed $\gamma = 1.0, \rho_{12} = 0.3, \rho_{13} = 0.2, c = 1, \beta = 4$.}
    \label{fig:r4_1d}
  \end{subfigure}
  \caption{We instantiate all five R0-R3 as tune-able outcomes of a single mitigation operator and R4 by varying only the mean of the audit distribution $\muaud$ in a controlled setting. 
  \Cref{fig:headline} shows that at $m = 0$ the regime is set by the signs of $(\rho_{12}, \rho_{13})$ alone, with $\gamma$ scaling magnitudes but not boundaries (\Cref{fig:regimevsnonlin}), so the regimes are not solely induced as a reward model property.}
  \label{fig:headline_and_r4}
\end{figure}

\begin{figure}
  \centering
  \includegraphics[width=0.99\linewidth]{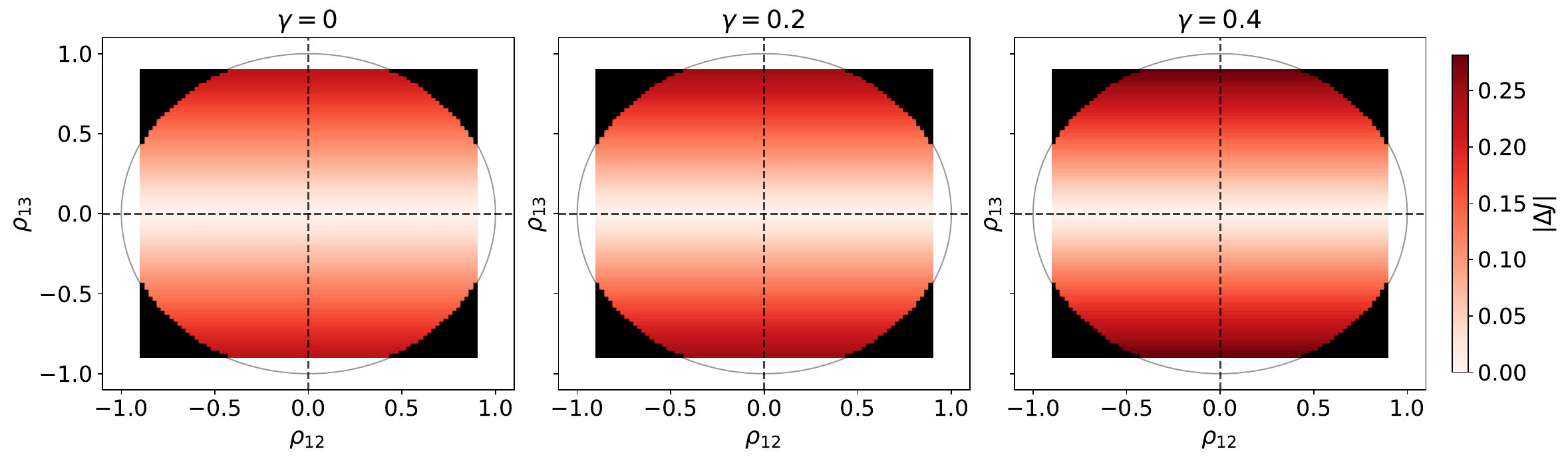}
  \caption{Shifting $\gamma$ to produce different $|\Delta J|$-heatmaps of \Cref{fig:headline} to illustrate regime robustness to non-linearity strength, as the regime boundaries do not change themselves.}
  \label{fig:regimevsnonlin}
\end{figure}

\begin{table}
  \caption{Validation of closed-form regime predictions against finite-$N$ BT-MLE on simulated preferences. 
  All cells use $a = b = w = c = 1, \beta{=}4$, $N{=}10^5$ samples per seed, for $S{=}50$ seeds. 
  The first six columns are the regime instances for $m_1 = 0$, except for R4 where we vary $m_1$ at fixed $\rho_{12}, \rho_{13}, \gamma$ (see \Cref{fig:r4_1d}).
  All eight cells achieve 100\% agreement on the respective regime call.}
  \label{tab:phase-diagram-validation}
  \centering
  \small
  \setlength{\tabcolsep}{4pt}
  \begin{tabular}{lrrrrrrrr}
    \toprule
    \textbf{Cell} & \textbf{R0} & \textbf{R0}$_\text{cont}$ & \textbf{R1}$_\text{neut}$ & \textbf{R1}$_\text{harm}$ & \textbf{R2} & \textbf{R3} & \multicolumn{2}{c}{\textbf{R4}} \\
    \midrule
    Scenario & & & & & & & &  \\
    $\rho_{12}$ & $0.0$ & $0.5$ & $0.5$ & $0.5$ & $0.0$ & $0.0$ & $0.3$ & $0.3$ \\
    $\rho_{13}$ & $-0.5$ & $-0.3$ & $0.0$ & $0.3$ & $0.5$ & $0.0$ & $0.2$ & $0.2$ \\
    $\gamma$ & $0.4$ & $0.4$ & $0.4$ & $0.4$ & $0.4$ & $0.4$ & $1.0$ & $1.0$ \\
    $m_1$ & $0.0$ & $0.0$ & $0.0$ & $0.0$ & $0.0$ & $0.0$ & $-1.5$ & $0.0$ \\
    \midrule
    Induced $g_1$ & $1.000$ & $1.000$ & $1.000$ & $1.000$ & $1.000$ & $1.000$ & $-1.399$ & $1.000$ \\
    \midrule
    $\Delta_2$ & & & & & & & &  \\
    predicted & $0.0$ & $-0.1563$ & $-0.1563$ & $-0.1563$ & $0.0$ & $0.0$ & $0.2099$ & $-0.1500$ \\
    measured & $0.0$ & $-0.1566$ & $-0.1563$ & $-0.1558$ & $0.0$ & $0.0$ & $0.2102$ & $-0.1496$ \\
    SE & $0.0$ & $0.0003$ & $0.0003$ & $0.0003$ & $0.0$ & $0.0$ & $0.0005$ & $0.0004$ \\
    \midrule
    $\Delta J$ & & & & & & & &  \\
    predicted & $0.1563$ & $0.0938$ & $0.0$ & $-0.0938$ & $-0.1563$ & $0.0$ & $0.1399$ & $-0.1000$ \\
    measured & $0.1563$ & $0.0940$ & $0.0$ & $-0.0935$ & $-0.1563$ & $0.0$ & $0.1401$ & $-0.0997$ \\
    measured SE & $0.0002$ & $0.0002$ & $0.0$ & $0.0002$ & $0.0003$ & $0.0$ & $0.0003$ & $0.0002$ \\
    \midrule
    Regime classification & R0 & R0\textsubscript{cont} & R1\textsubscript{neut} & R1\textsubscript{harm} & R2 & R3 & R0\textsubscript{cont} & R1\textsubscript{harm} \\
    \bottomrule
  \end{tabular}
\end{table}

\subsection{Formal Proof of Audit-Distribution Insufficiency}
\label{app:audit-insufficiency}

We give a formal proof of the structural-blindness claim of \Cref{sec:impossibility-sufficiency} that standard ordinal reward-model benchmarks cannot reliably distinguish R0 from R0\textsubscript{cont}, R1, or R2.

\paragraph{Benchmark class.}
Let $\mathcal{B}$ denote the class of benchmark functionals depending only on the joint distribution of $(\tilde R, M_i(\tilde R), R, \Phi)$ under $\muaud$. 
We take $\muaud$ as the empirical measure of the evaluation set of $\mathcal{B}$ where applicable.
This class subsumes ranking accuracy, pairwise win-rate, style-controlled accuracy, preference-prediction calibration, reward-target correlation, the linear-reliance statistic $g$ of Definition~\ref{def:linear_reliance}, and any composite of these, covering the evaluation paradigm of \Cref{sec:3_4_evalimps}.
Note that including $R$ in the input of $\mathcal{B}$ strengthens the result, since standard benchmarks lack oracle access to $R$.
Reward-distance metrics with regret guarantees \citep{gleave2021quantifying,
skalse2024starc}, computed at $\muaud$, are functionals of this same joint and lie in $\mathcal{B}$, and since only $\pi_{\mathrm{ref}}$ varies in \Cref{thm:audit-insufficiency}, the functions $\tilde{R}$, $M_i(\tilde{R})$, and $R$ are the same in all four instances, so any such distance, and any regret bound it induces, coincides across the four while $\Delta J$ changes sign.

\auditinsufficiency*

\begin{proof}
\textbf{Construction.} Take $\mathcal{X}$ a singleton (suppressed), $\mathcal{Y} = \mathbb{R}^3$, feature map $\phi_k(y) = y_k$, and rewards $R(y) = \phi_3(y)$, $\tilde R(y) = \phi_1(y) + \phi_2(y) + \phi_3(y)$. Set $\Phi_\mathrm{sp} = \{\phi_1, \phi_2\}$, target $\phi_i = \phi_1$, and audit distribution $\muaud = \mathcal{N}(0, I_3)$. Fix a small $\delta \in (0, 1/2)$. The four instances differ only in $\pi_\mathrm{ref}^{(k)} = \mathcal{N}(0, \Sigma^{(k)})$, where each $\Sigma^{(k)}$ has unit diagonal, $\Sigma^{(k)}_{23} = 0$, and
\[
\bigl(\Sigma^{(0)}_{12}, \Sigma^{(0)}_{13}\bigr) = \bigl(0, -\tfrac{1}{2}\bigr),
\quad
\bigl(\Sigma^{(1)}_{12}, \Sigma^{(1)}_{13}\bigr) = \bigl(\tfrac{1}{2}, -\tfrac{1}{2}\bigr),\]
\[
\bigl(\Sigma^{(2)}_{12}, \Sigma^{(2)}_{13}\bigr) = \bigl(\tfrac{1}{2}, \delta\bigr),
\quad
\bigl(\Sigma^{(3)}_{12}, \Sigma^{(3)}_{13}\bigr) = \bigl(0, \tfrac{1}{2}\bigr).
\]
Each $\Sigma^{(k)}$ has $\det = 1 - \Sigma_{12}^2 - \Sigma_{13}^2 > 0$ by Sylvester's criterion: $\det^{(0)} = \det^{(3)} = 3/4$, $\det^{(1)} = 1/2$, and $\det^{(2)} = 3/4 - \delta^2 > 0$ for $\delta < 1/2$.
Each $\pi_\mathrm{ref}^{(k)}$ is restricted to the box $\mathcal{Y}_T$ of \Cref{lem:box_truncation}, so Assumption~\ref{asm:reg} holds for every $T$ by \Cref{lem:box_truncation}(i), while $\muaud = \mathcal{N}(0, I_3)$ is untruncated. 
The closed forms in (ii) are the $T \to \infty$ limits of \Cref{lem:box_truncation}(ii). 
The two exact zeros $\Delta_2^{(0)} = \Delta_2^{(3)} = 0$ hold for every $T$ by \Cref{lem:box_truncation}(iii), since $\Sigma^{(k)}_{12} = \Sigma^{(k)}_{23} = 0$ for $k \in \{0, 3\}$ makes $y_2$ independent of $(y_1, y_3)$ under $\pi_\mathrm{ref}^{(k)}$ and $\tilde R$, $M_i(\tilde R)$ share the term $\phi_2$. 
The six strict signs in (ii) are limits of continuous quantities, so they hold simultaneously for all $T$ above a finite threshold, and we therefore fix any such $T$.
Linearity and Gaussianity are used for closed-form existence and are not required for the impossibility claim.
Assumption~\ref{asm:nondeg} holds since $\mathbb{E}_{\muaud}[\Phi\Phi^\top] = I_3$.
Assumption~\ref{asm:featreal} holds since $\mathcal{Y} = \mathbb{R}^3$ admits independent coordinate perturbations on a set of full $\muaud$-measure.

\textbf{(i) Audit-side identity.}
The Gram at $\muaud$ is $I_3$, giving $g(\tilde R; \muaud) = (1, 1, 1)$ and $M_i(\tilde R) = \phi_2 + \phi_3$ across all four instances. Since $\tilde R$, $M_i(\tilde R)$, $R$, $\Phi$, and $\muaud$ are identical across $k$, the joint distribution of $(\tilde R, M_i(\tilde R), R, \Phi)$ under $\muaud$ is identical across $k$, so $B$ takes the same value on all four instances for every $B \in \mathcal{B}$.

\textbf{(ii) Policy-side classification.}
For any linear $\bar R(y) = c^\top y$ and $\pi_\mathrm{ref} = \mathcal{N}(0, \Sigma)$, completing the square in \Cref{equ:softmaxsolution} gives $\pi^\star_\beta(\bar R) = \mathcal{N}(\Sigma c / \beta,\ \Sigma)$. Applying this with $c = (1,1,1)^\top$ and $c' = (0,1,1)^\top$ at each $\Sigma^{(k)}$:
\[
\Delta_2(\pi^{(k)}, \pi'^{(k)}) = -\Sigma^{(k)}_{12}/\beta,
\qquad
\Delta J^{(k)} = -\Sigma^{(k)}_{13}/\beta,
\]
where $\Delta_j$ is checked against $\Phi_\mathrm{sp} \setminus \{\phi_i\} = \{\phi_2\}$. Substituting:
\[
\bigl(\Delta_2, \Delta J\bigr)^{(0)} = \bigl(0, +\tfrac{1}{2\beta}\bigr),\quad
\bigl(\Delta_2, \Delta J\bigr)^{(1)} = \bigl(-\tfrac{1}{2\beta}, +\tfrac{1}{2\beta}\bigr),\]
\[
\bigl(\Delta_2, \Delta J\bigr)^{(2)} = \bigl(-\tfrac{1}{2\beta}, -\tfrac{\delta}{\beta}\bigr),\quad
\bigl(\Delta_2, \Delta J\bigr)^{(3)} = \bigl(0, -\tfrac{1}{2\beta}\bigr).
\]
By Definition~\ref{def:all_regimes}, these are R0, R0\textsubscript{cont}, R1 (harmful substitution), and R2 respectively.
\end{proof}

\paragraph{Remark (scope).}
The theorem fixes $\muaud$ and $M_i$ across instances. 
R3 is additionally realizable in the same construction by setting $\Sigma_{12} = \Sigma_{13} = 0$, with audit observables still identical and distinguishability from R0 requiring $D_\mathrm{KL}(\pi' \,\Vert\, \pi)$, which is not audit-local. 
R4 requires non-linear proxies (see \Cref{app:nonsuck,app:phasediagrams}).
The varying $\pi_\mathrm{ref}$ corresponds to the same reward model deployed against different reference policies, the standard regime in which benchmarks claim to evaluate reward models independently of downstream RLHF setup.

\paragraph{Functional blindspot.}
\Cref{thm:audit-insufficiency} establishes the distributional blindspot. The functional blindspot is independent and follows from cardinal-scale sensitivity:

\begin{corollary}[Functional blindspot via reward rescaling]
\label{cor:functional-blindspot}
Let $\mathcal{B}_\mathrm{ord} \subseteq \mathcal{B}$ denote the benchmark functionals satisfying $B(f \circ \tilde R, f \circ M_i(\tilde R), R, \Phi, \muaud) = B(\tilde R, M_i(\tilde R), R, \Phi, \muaud)$ for every strictly monotone increasing $f \colon \mathbb{R} \to \mathbb{R}$. Then for the construction of \Cref{thm:audit-insufficiency}, every $B \in \mathcal{B}_\mathrm{ord}$, and every $c > 0$,
\[
B(c\tilde R, M_i(c\tilde R), R, \Phi, \muaud)
= B(\tilde R, M_i(\tilde R), R, \Phi, \muaud),
\]
while $J(\pi^\star_\beta(c\tilde R), R)$ is strictly monotone in $c$ in the limit of \Cref{lem:box_truncation}(ii), and takes distinct values at any two fixed $c \neq c'$ for $T$ large.
\end{corollary}

\begin{proof}
By linearity, $M_i(c\tilde R) = c\,M_i(\tilde R)$, and ordinal invariance under $f(y) = cy$ gives the benchmark identity. The KL identity $\pi^\star_\beta(c\bar R) = \pi^\star_{\beta/c}(\bar R)$ (\Cref{app:A3}) gives $J(\pi^\star_\beta(c\tilde R), R) = c \cdot (\Sigma c_{\tilde R})_3 / \beta$ in the construction of \Cref{thm:audit-insufficiency}, where $c_{\tilde R} = (1,1,1)^\top$ and $(\Sigma^{(k)} c_{\tilde R})_3 = \Sigma^{(k)}_{13} + 1 \neq 0$ for each $k$, so $J$ is strictly monotone in $c$ in the limit, and for any two fixed $c \neq c'$ the two limits differ as stated above, so the truncated values differ for large $T$ as stated in \Cref{lem:box_truncation}(ii).
\end{proof}

\paragraph{Scope of the construction.} 
All four reference policies in the proof of \Cref{thm:audit-insufficiency} differ from $\muaud = \mathcal{N}(0, I_3)$, and the result relies on this decoupling. 
If instead $\muaud = \mu_{\pi_{\mathrm{ref}}}$ and $\beta$ is known, \Cref{equ:softmaxsolution} allows reweighting audit samples by $\exp(\tilde{R}/\beta)$ and $\exp(M_i(\tilde{R})/\beta)$, normalized per prompt, to estimate $\Delta J$ and $\Delta_j$. 
This is a policy-distribution evaluation computed from reference samples, consistent with \Cref{thm:audit-sufficiency}. 
It is unreliable when the optimized policy concentrates on responses that $\pi_{\mathrm{ref}}$ rarely samples, the regime in which \citet{fluri25a} show error-regret mismatch under KL regularization.

\subsection{Audit-Distribution Sufficiency}
\label{app:audit-sufficiency}

\Cref{thm:audit-insufficiency} establishes that benchmark functionals depending only on $(\tilde R, M_i(\tilde R), R, \Phi, \muaud)$ cannot reliably separate R0 from R0\textsubscript{cont}, R1, or R2. We now show that adding the policy-induced distributions $\mu_{\pi^\star_\beta(\tilde R)}$ and $\mu_{\pi^\star_\beta(M_i(\tilde R))}$ to the input suffices, certifying the prescription that benchmarks must evaluate at policy-induced distributions (Do~1 of \Cref{app:a8-takeaways}).

\paragraph{Extended benchmark class.}
Let $\mathcal{B}^+$ denote the class of benchmark functionals depending on the joint distributions of $(\tilde R, M_i(\tilde R), R, \Phi)$ under each of $\muaud$, $\mu_{\pi^\star_\beta(\tilde R)}$, and $\mu_{\pi^\star_\beta(M_i(\tilde R))}$, together with the partition $\Phi = \Phi_\mathrm{sp} \sqcup \Phi_\mathrm{struct}$. We have $\mathcal{B} \subset \mathcal{B}^+$.

\auditsufficiency*

\begin{proof}
$\Delta_j$ and $\Delta J$ are first moments of $\Phi$ and $R$ under $\mu_\pi$ and $\mu_{\pi'}$, both inputs to $\mathcal{B}^+$, so $B^\star \in \mathcal{B}^+$. 
Finiteness of these expectations follows from Assumption~\ref{asm:reg} via the boundedness chain in \Cref{app:A1}, applied to both $\tilde R$ and $M_i(\tilde R)$ (the latter by closure under finite linear combinations).

Correctness on each regime follows directly from Definition~\ref{def:all_regimes}: R0 requires $\Delta J > 0$ with no off-target redirection (first branch),
$\mathrm{R0}_{\mathrm{cont}}$ requires $\Delta J > 0$ with off-target redirection (second branch), R1 requires $\Delta J \le 0$ with off-target redirection (third branch), R2 requires $\Delta J < 0$ with no off-target redirection (fourth branch). 
The four branches are pairwise disjoint: $\{\mathrm{R0}, \mathrm{R0}_{\mathrm{cont}}\}$ have $\Delta J > 0$ and $\{\mathrm{R1}, \mathrm{R2}\}$ have $\Delta J \le 0$, separating any cross-pair by the sign of $\Delta J$. 
Within $\{\mathrm{R0},
\mathrm{R0}_{\mathrm{cont}}\}$ and within $\{\mathrm{R1}, \mathrm{R2}\}$ the branches differ in redirection ($\mathrm{R0}_{\mathrm{cont}}$, R1 require redirection, R0 and R2 do not). 
Inputs in R3 (no redirection, $\Delta J = 0$) are excluded by hypothesis.
\end{proof}

\begin{corollary}[Sufficiency for R3]
\label{cor:sufficiency-r3}
The classifier $B^\star$ extends to $\{\mathrm{R0}, \mathrm{R0}_{\mathrm{cont}}, \mathrm{R1}, \mathrm{R2}, \mathrm{R3}\}$ by adding a fifth branch:
return R3 if $\Delta J = 0$ and $\Delta_j = 0$ for all $\phi_j \in \Phi_{\mathrm{sp}} \setminus \{\phi_i\}$. 
This branch matches R3 (silent non-op) of Definition~\ref{def:all_regimes} exactly and is disjoint from the R0-R2 branches by the $\Delta J = 0$ condition.
The observable $D_{\mathrm{KL}}(\pi' \,\|\, \pi)$ is not required for the regime call, though reporting it alongside the label distinguishes policy-relevant R3 ($D_{\mathrm{KL}} > 0$, mitigation moved the policy along directions orthogonal to $(\Delta_j, \Delta J)$)  from policy-irrelevant R3 ($D_{\mathrm{KL}} = 0$, mitigation did not 
move the policy at all).
\end{corollary}

\begin{corollary}[Finite-sample sufficiency]
\label{cor:eps-sufficiency}
Under the noise-floor $\varepsilon$-bands of \Cref{app:epsilonbands}, replacing exact equalities in \Cref{equ:bstar} with the banded conditions yields $B^\star_\varepsilon \in \mathcal{B}^+$ returning the correct $\varepsilon$-banded label.
\end{corollary}

\begin{corollary}[Multi-audit sufficiency for R4]
\label{cor:multi-audit-r4}
Let $\mathcal{B}^{++}$ denote the class of benchmark functionals depending on the inputs of $\mathcal{B}^+$ at each of $K \geq 2$ audit distributions $\muaud^{(1)}, \dots, \muaud^{(K)}$, with the canonical $M_i^{(k)}$ constructed from each. 
Define $B^\star_\text{R4}$ as the
indicator that $B^\star$ (Theorem~\ref{thm:audit-sufficiency}) returns different regime labels across the $K$ instances. 
Then $B^\star_\text{R4}$ detects R4 in the sense of Definition~\ref{def:R4}.
\end{corollary}

\begin{proof}
By Theorem~\ref{thm:audit-sufficiency}, each per-distribution call of $B^\star$ returns the correct R0-R2 label. Definition~\ref{def:R4} defines R4 as cross-distribution disagreement on the regime label, which
is exactly what $B^\star_\text{R4}$ tests.
\end{proof}

\paragraph{Discussion.}
Theorem~\ref{thm:audit-insufficiency} (audit-distribution insufficiency) and Theorem~\ref{thm:audit-sufficiency} (audit-distribution sufficiency) together show that audit-only inputs cannot separate R0, R0\textsubscript{cont}, R1, and R2, while augmenting the input of $\mathcal{B}$ with the policy-induced distributions $\mu_{\pi^\star_\beta(\tilde R)}$ and $\mu_{\pi^\star_\beta( M_i(\tilde R))}$ suffices. 
Corollary~\ref{cor:multi-audit-r4} extends this to R4 detection under multi-audit evaluation.
This grounds the policy-distribution evaluation prescription structurally rather than heuristically. 
We do not claim minimality of $\mathcal{B}^+$, as the classifier $B^\star$ depends only on the scalars $(\Delta_j, \Delta J)$ and the partition $\Phi_{\mathrm{sp}}$, so any extension of $\mathcal{B}$ that determines these quantities also suffices.

\subsection{Extension to Direct Preference Optimization}
\label{app:dpo}

In this section, we show how our framework and theorems from \Cref{sec:bias_sub} apply to DPO and related direct alignment methods \citep{rafailov2023direct, park2024disentangling, meng2024simpo, lu2024eliminating, lilmpo2025length}.
As the DPO optimum is equal to the KL-regularized RLHF optimum in \Cref{equ:j_rlhf,equ:softmaxsolution} under the Bradley-Terry model, the optimization-side of the extension carries over directly, with the exception of reference-free variants like SimPO, discussed below.
However, the extension from our explicit reward formulation throughout \Cref{sec:bias_sub} to an implicit reward formulation needs further support.
The single-axis mitigation operator $M_i$ (Definition~\ref{def:single_axis_mitigation}) is defined acting on a proxy reward \emph{before} optimization, whereas a DPO implicit reward exists only \emph{after} training.
We make the extension precise by treating the DPO implicit reward $\hat r$ as the proxy $\tilde R$ of our framework, as the policy $\pi = \pi^\star_\beta (\hat r)$ is the KL-regularized optimum of its own implicit reward (see below).
Under this reading the sufficiency classifier of \Cref{thm:audit-sufficiency} transfers verbatim, the regime taxonomy and the impossibility result of \Cref{thm:audit-insufficiency} transfer for mitigations that act as single-axis operators on the implicit reward.

\paragraph{The implicit reward as a proxy.}
Let $\pi$ be a policy and $\pi_\mathrm{ref}$ the reference policy of the optimization setup in \Cref{app:A1}.

\begin{definition}[DPO implicit reward]\label{def:dpo_implicit_reward}
The \emph{DPO implicit reward} of $\pi$ relative to $\pi_\mathrm{ref}$ is
\begin{equation*}
    \hat r(x,y) \;=\; \beta \, \log \frac{\pi(y \mid x)}{\pi_\mathrm{ref}(y \mid x)}.
\end{equation*}
\end{definition}

The next proposition records why $\hat r$ is the right object to place in the $\tilde R$ role. 
The implicit reward $\hat r$ induces $\pi$ as its own KL-regularized optimum, and it recovers any explicit reward that generated $\pi$ up to the prompt-only gauge freedom of \Cref{app:A4}.

\begin{lemma}[Implicit-reward consistency and gauge]
 \label{lem:dpo_consistency}
    Fix $\beta > 0$.
    \begin{enumerate}
        \item[(i)] $\pi = \pi^\star_\beta(\hat r)$ whenever $\pi \ll \pi_\mathrm{ref}$, so $\pi$ is the KL-regularized optimum of its own implicit reward.
        \item[(ii)] If $\pi = \pi^\star_\beta(\bar R)$ for some $\bar R$ satisfying Assumption~\ref{asm:reg}, then $\hat r = \bar R - b(x)$ with $b(x) = \beta \log Z_{\bar R}(x)$ prompt-only, where $Z_{\bar R}(x) = \mathbb{E}_{y \sim \pi_\mathrm{ref}(\cdot \mid x)}[\exp(\bar R(x,y)/\beta)]$. 
        Hence $\hat r$ and $\bar R$ lie in the same gauge class of \Cref{app:A4} and induce the same policy.
    \end{enumerate}
\end{lemma}

\begin{proof}
(i) By \Cref{equ:softmaxsolution}, $\pi^\star_\beta(\hat r)(y \mid x) \propto \pi_\mathrm{ref}(y \mid x)\exp(\hat r(x,y)/\beta) = \pi_\mathrm{ref}(y \mid x)\cdot \pi(y \mid x)/\pi_\mathrm{ref}(y \mid x) = \pi(y \mid x)$, and the right-hand side already normalizes to one, so the proportionality is an equality. 

(ii) From $\pi = \pi^\star_\beta(\bar R)$ we have $\log \pi = \log \pi_\mathrm{ref} + \bar R/\beta - \log Z_{\bar R}(x)$. Multiplying by $\beta$ and rearranging gives $\hat r = \bar R - \beta \log Z_{\bar R}(x)$, and $b(x) = \beta \log Z_{\bar R}(x)$ depends on $x$ only.
\end{proof}

The hypothesis in (i) is only $\pi \ll \pi_\mathrm{ref}$ and a regularity requirement (\Cref{def:dpo_regular}) is not needed for the identity.
\Cref{lem:dpo_consistency}(ii) is the policy-level invariance of \citet{rafailov2023direct}, and the prompt-only shift it leaves free is the same gauge that preference data leaves free (see \Cref{sec:background}).
The DPO parameterization fixes a representative of this class through the normalizer $\beta \log Z_{\bar R}(x)$.
When the implicit reward is instead a free network, as in the multi-step method of \citet{hejna2023inverse} that reduces to DPO in the contextual-bandit case, the shift must be constrained explicitly, and removing their regularizer drops the hopper-medium-replay score.
In the token-level view of \citet{rafailov2024from}, the DPO implicit reward is an optimal advantage and is fixed only up to potential-based shaping. 
Summed over a response, the intermediate potentials telescope and only a prompt-only term remains, so the token-level gauge reduces at the response level to the prompt-only shift of \Cref{lem:dpo_consistency}(ii).
It is also the one-step case of the entropy-regularized IRL characterization of \citet{cao2021identifiability}. 
There, rewards consistent with a fixed policy are parameterized by an arbitrary value function, and at a single step this residual is exactly a prompt-only shift.
The gauge-invariant reliance $g_\mathrm{cent}$ and the operator $M_i^\mathrm{cent}$ of \Cref{app:A4} are therefore the appropriate statistics for DPO, and the regime classification ports under $M_i^\mathrm{cent}$ as shown in \Cref{app:A4}.
Because $\hat r$ is fixed only up to this prompt-only gauge, the canonical operator-side construction below uses the centered operator $M_i^\mathrm{cent}$ of \Cref{app:A4}, whose regime classification is gauge-invariant.
We use $M_i^\mathrm{cent}$ rather than the plain $M_i$, as the regime call of $M_i$ would depend on the chosen representative of $\hat r$.

\paragraph{DPO single-axis mitigations.} 
Next, we need to define single-axis mitigations for different DPO-like methods. 
Fix a feature index $i$ and a targeted spurious feature $\phi_i \in \Phi_\mathrm{sp}$.
We distinguish two ways a single-axis mitigation enters a DPO pipeline.

\begin{definition}[Operator-side and loss/data-side DPO mitigations]
\label{def:dpo_mitigation_placement}
    An \emph{operator-side} (or post-hoc) mitigation applies a single-axis operator $M_i$ (Definition~\ref{def:single_axis_mitigation}) directly to the implicit reward, producing $M_i(\hat r)$ and the post-mitigation policy $\pi' = \pi^\star_\beta(M_i(\hat r))$. 
    Its canonical instance is the centered operator $M_i^\mathrm{cent}$ of \Cref{app:A4}. \\
    A \emph{loss-side} or \emph{data-side} mitigation modifies the DPO objective or its preference data and retrains, producing a new policy $\pi'$ with its own implicit reward $\hat r' = \beta \log(\pi'/\pi_\mathrm{ref})$. \\
    In general $\hat r'$ does not equal the canonical centered operator $M_i^\mathrm{cent}$ of \Cref{app:A4} applied to $\hat r$.
\end{definition}

Examples of loss- or data-side DPO length mitigations include R-DPO \citep{park2024disentangling}, SimPO \citep{meng2024simpo}, SamPO \citep{lu2024eliminating}, and LMPO \citep{lilmpo2025length}.
R-DPO is the instructive case, as by Eq.~7 in \citet{park2024disentangling}, its policy equals $\pi^\star_\beta(\hat r - \alpha|y|)$ at the idealized optima of both objectives. 
This is the same fixed-coefficient length penalty as in \Cref{sec:4_evidence}. 
It is, however, trained through an offline preference loss that need not attain this optimum and uses a hyperparameter $\alpha$ rather than $g_{\mathrm{cent},i}(\hat r;\muaud)$, so $\hat r'$ need not equal $M_i^{\mathrm{cent}}(\hat r)$ and the method is loss-side under \Cref{def:dpo_mitigation_placement}.
SimPO is reference-free and length-normalizes its reward, so it is loss-side and does not act on an implicit reward of the $\hat r$ form. 
The $1/|y|$ normalization is a length-dependent reweighting, not the monotone reparametrization of \Cref{cor:functional-blindspot}.
Operator-side mitigation of a DPO policy instead corresponds to applying a post-hoc operator directly to $\hat r$, equivalent to post-hoc mitigation of an explicit reward model (we add them here for completeness). 
The LOESS length calibration of \citet{huang2025posthoc} and the latent-space probe shaping of \citet{fein2026one} 
are both operator-side but not canonical projections on $\phi_i$, so they are single-axis mitigations in the general sense of Definition~\ref{def:single_axis_mitigation} and inherit the outcome-level $(\Delta_j, \Delta J)$ classification rather than the constructive guarantee of \Cref{lem:centered_single_axis_identity} (\Cref{app:a8-takeaways}). 
Both are admissible as proxies through \Cref{lem:dpo_consistency} and \Cref{def:dpo_regular}.
\Cref{tab:dpo-methods} summarizes where the surveyed DPO length mitigations fall and which results reach each.

\begin{table}[ht]
    \caption{
    Placement of the surveyed DPO length mitigations under \Cref{def:dpo_mitigation_placement}.
    All are classifiable by the sufficiency protocol (\Cref{cor:dpo_sufficiency}) and subject to the audit-distribution insufficiency, which stems from auditing at $\muaud$ rather than the policy pair and so applies to every row, with \Cref{cor:dpo_impossibility} the sharp four-regime instantiation for the canonical linear operator.
    Operator-side rows reach the framework through an operator on $\hat r$, loss-side rows through the policy pair, and reference-free rows have no guaranteed regular implicit reward (\Cref{def:dpo_regular}).}
    \label{tab:dpo-methods}
    \centering
    \begin{tabular}{lll}
    \toprule
    Method & Placement & $\pi_\mathrm{ref}$ \\
    \midrule
    R-DPO \citep{park2024disentangling} & Loss-side & Yes \\
    SamPO \citep{lu2024eliminating} & Loss-side & Yes \\
    SimPO \citep{meng2024simpo} & Loss-side & Free \\
    LMPO \citep{lilmpo2025length} & Loss-side & Free \\
    \midrule
    LOESS calibration \citep{huang2025posthoc} & Operator-side & Inherited \\
    Probe shaping \citep{fein2026one} & Operator-side & Inherited \\
    \bottomrule
    \end{tabular}
\end{table}

The operator-side case is, by \Cref{lem:dpo_consistency}, the explicit-reward framework with $\hat r$ in place of $\tilde R$. 
For the canonical centered operator $M_i^\mathrm{cent}$, \Cref{lem:centered_single_axis_identity} and the full taxonomy apply once $\hat r$ is admissible. 
For non-linear operator-side mitigations the outcome-level classification still applies, while the constructive single-axis guarantee does not (\Cref{app:a8-takeaways}).
The loss/data-side case has no operator to analyze, and connects to the framework only through the policy pair $(\pi, \pi')$ (see \Cref{lem:dpo_regularity,cor:dpo_sufficiency}).

\paragraph{Regularity of the implicit reward.}
The implicit reward $\hat r$ is only admissible as a proxy if it is bounded.

\begin{definition}[Regular implicit reward]\label{def:dpo_regular}
The implicit reward $\hat r$ is \emph{regular against $\pi_\mathrm{ref}$} if the density ratio $d\mu_\pi/d\mu_{\pi_\mathrm{ref}} = \pi/\pi_\mathrm{ref}$ is essentially bounded above and below on the support of $\pi_\mathrm{ref}$, that is, there exist constants $0 < m \le M < \infty$ with $m \le \pi(y \mid x)/\pi_\mathrm{ref}(y \mid x) \le M$ for $\mu_{\pi_\mathrm{ref}}$-a.e.\ $(x,y)$.
\end{definition}

Under \Cref{def:dpo_regular} the implicit reward satisfies the $L^\infty(\mu_{\pi_\mathrm{ref}})$ clause of Assumption~\ref{asm:reg}.
The $L^2(\muaud)$ clause is automatic when $\muaud = \mu_{\pi_\mathrm{ref}}$ and otherwise requires $\muaud \ll \mu_{\pi_\mathrm{ref}}$, the same condition already assumed for annotator-conditioned audits in \Cref{app:A2}.
Then $\hat r$ enters every statement of \Cref{sec:bias_sub} and \Cref{app:more_math} in the $\tilde R$ role, with the consequences proved in \Cref{lem:dpo_regularity}.
Regularity is free whenever $\pi$ optimizes a bounded reward and fails only under policy collapse or out-of-distribution drift.
Reference-free objectives such as SimPO \citep{meng2024simpo} make this boundary explicit, as with no $\pi_\mathrm{ref}$ term constraining the density ratio, regularity of $\hat r$ against any fixed reference is not guaranteed, a possibility \citet{meng2024simpo} themselves flag.

Two consequences follow. 
First, a regular implicit reward is admissible as a proxy in the $\tilde R$ role. 
Second, the sufficiency classifier transfers to every DPO mitigation regardless of placement.

\begin{lemma}[Admissibility of the implicit reward]\label{lem:dpo_regularity}
Suppose $\hat r$ is regular against $\pi_\mathrm{ref}$ (\Cref{def:dpo_regular}), the feature map $\Phi$ satisfies Assumption~\ref{asm:reg}, the centered Gram is non-degenerate (Assumption~\ref{asm:nondeg_gauge}), and $\muaud \ll \mu_{\pi_\mathrm{ref}}$ (automatic when $\muaud = \mu_{\pi_\mathrm{ref}}$). 
Then
\begin{enumerate}
    \item[(i)] $\hat r \in L^2(\muaud) \cap L^\infty(\mu_{\pi_\mathrm{ref}})$, so $\hat r$ satisfies Assumption~\ref{asm:reg} in the $\tilde R$ role.
    \item[(ii)] $M_i^\mathrm{cent}(\hat r) \in L^2(\muaud) \cap L^\infty(\mu_{\pi_\mathrm{ref}})$.
\end{enumerate}
Consequently $\pi^\star_\beta(\hat r)$ and $\pi^\star_\beta(M_i^\mathrm{cent}(\hat r))$ are well-defined KL-regularized optima with $\pi^\star_\beta(\hat r) = \pi$ by \Cref{lem:dpo_consistency}(i), and the centered single-axis identity \Cref{lem:centered_single_axis_identity} holds for $\hat r$.
\end{lemma}

\begin{proof}
(i) By \Cref{def:dpo_regular} there are $0 < m \le M < \infty$ with $m \le \pi/\pi_\mathrm{ref} \le M$ $\mu_{\pi_\mathrm{ref}}$-a.e., so $\hat r = \beta \log(\pi/\pi_\mathrm{ref}) \in [\beta \log m,\, \beta \log M]$ $\mu_{\pi_\mathrm{ref}}$-a.e.\ and hence $\hat r \in L^\infty(\mu_{\pi_\mathrm{ref}})$. 
Since $\muaud \ll \mu_{\pi_\mathrm{ref}}$, this essential bound transfers to $\muaud$-a.e., so $\hat r \in L^\infty(\muaud) \subset L^2(\muaud)$ as $\muaud$ is a probability measure.

(ii) Under (i) and Assumption~\ref{asm:reg} on $\Phi$, the centered reliance $g_{\mathrm{cent},i}(\hat r; \muaud)$ is a finite scalar, well-defined by Assumption~\ref{asm:nondeg_gauge} together with $\hat r^\mathrm{cent}, \bar\Phi \in L^2(\muaud)$. 
Then $M_i^\mathrm{cent}(\hat r) = \hat r - g_{\mathrm{cent},i}(\hat r; \muaud)\, \phi_i$ is a finite $\mathbb{R}$-linear combination of $\hat r$ and $\phi_i$, both in $L^2(\muaud) \cap L^\infty(\mu_{\pi_\mathrm{ref}})$, and the claim follows by closure under finite linear combinations (\Cref{app:A1}).

The consequence is then immediate. 
Assumption~\ref{asm:reg} holds for $\hat r$ and $M_i^\mathrm{cent}(\hat r)$ in the $\tilde R$ role, so the optimization setup of \Cref{app:A1} applies to both and the two optima are well-defined, $\pi^\star_\beta(\hat r) = \pi$ by \Cref{lem:dpo_consistency}(i), and \Cref{lem:centered_single_axis_identity} applies to $\hat r$.
\end{proof}

With $\hat r$ admissible and $M_i^\mathrm{cent}$ qualifying as a single-axis operator on it, the regime taxonomy of \Cref{sec:3_3_bias_substitution} applies to the operator-side pair $\pi' = \pi^\star_\beta(M_i^\mathrm{cent}(\hat r))$ exactly as for an explicit proxy. 
The sufficiency classifier transfers more broadly, without reference to $\hat r$ at all.

\begin{corollary}[Sufficiency transfer to DPO mitigations]\label{cor:dpo_sufficiency}
Let $\pi$ and $\pi'$ be the pre- and post-mitigation DPO policies of \Cref{def:dpo_mitigation_placement}, either operator-side with $\pi' = \pi^\star_\beta(M_i^\mathrm{cent}(\hat r))$ or loss/data-side with $\pi'$ the retrained policy. 
Suppose $\Phi$ and $R$ satisfy Assumption~\ref{asm:reg} and the policy-induced expectations $\mathbb{E}_{\mu_\pi}[\phi_j]$, $\mathbb{E}_{\mu_{\pi'}}[\phi_j]$, $J(\pi, R)$, and $J(\pi', R)$ are finite. 
Then $B^\star$ of \Cref{thm:audit-sufficiency}, evaluated on $(\pi, \pi')$, returns the correct regime label among $\{\mathrm{R0}, \mathrm{R0}_{\mathrm{cont}}, \mathrm{R1}, \mathrm{R2}\}$, extended to $\mathrm{R3}$ by \Cref{cor:sufficiency-r3} and, in the operator-side case, to $\mathrm{R4}$ by \Cref{cor:multi-audit-r4} under the multi-$\muaud$ protocol of \Cref{app:a8-takeaways}.
\end{corollary}

\begin{proof}
$\Delta_j(\pi, \pi') = \mathbb{E}_{\mu_{\pi'}}[\phi_j] - \mathbb{E}_{\mu_\pi}[\phi_j]$ and $\Delta J = J(\pi', R) - J(\pi, R)$ are first moments of $\Phi$ and $R$ under $\mu_\pi$ and $\mu_{\pi'}$, both inputs to $\mathcal{B}^+$, so $B^\star \in \mathcal{B}^+$. 
Their finiteness is the stated hypothesis. 
The four-branch correctness and pairwise disjointness are exactly \Cref{thm:audit-sufficiency}, whose proof uses nothing about how $\pi'$ is produced beyond finiteness of the first moments.
Hence $B^\star(\pi, \pi')$ returns the correct label, and the R3 and R4 extensions follow from \Cref{cor:sufficiency-r3,cor:multi-audit-r4}.
\end{proof}

The finiteness hypothesis is automatic in the operator-side case, where \Cref{lem:dpo_regularity} makes $M_i^\mathrm{cent}(\hat r)$ admissible and the density-ratio bound of \Cref{app:A1} gives finite feature and reward expectations at $\mu_{\pi'}$. 
In the loss/data-side case it is a mild condition on the trained policy, strictly weaker than requiring its implicit reward $\hat r'$ to be regular. 
Since $B^\star$ reads only $(\Delta_j, \Delta J)$ and never $\hat r'$, the sufficiency half of the framework reaches loss-side and data-side mitigations that lie outside the single-axis operator class, which is the precise sense in which the prescriptive core covers DPO mitigations regardless of placement.

\paragraph{Operator-side transfer.}
The operator-side reading inherits the full taxonomy. By \Cref{lem:dpo_regularity} the centered single-axis identity holds for $\hat r$, so R0-R3 classify $(\pi, \pi')$ as in \Cref{sec:3_3_bias_substitution}, and R4 ports because $M_i^\mathrm{cent}$ retains its $\muaud$-dependence (\Cref{app:A4}), with detection by \Cref{cor:multi-audit-r4}. 
The impossibility transfers as well for gauge-invariant benchmarks.

\begin{corollary}[Operator-side impossibility for DPO]\label{cor:dpo_impossibility}
For $B \in \mathcal{B}$ invariant under the prompt-only gauge of \Cref{app:A4}, \Cref{thm:audit-insufficiency} transfers to operator-side DPO mitigations, with the implicit reward $\hat r$ in the $\tilde R$ role and $M_i^\mathrm{cent}$ in the $M_i$ role. 
There are four reference policies, with $\muaud, M_i^\mathrm{cent}, R, \Phi, \Phi_\mathrm{sp}, \beta$ fixed and the four implicit rewards sharing one gauge class, such that every such $B$ agrees across the four while the mitigated policies $\pi^\star_\beta(M_i^\mathrm{cent}(\hat r))$ fall into R0, R0\textsubscript{cont}, R1, and R2.
\end{corollary}

\begin{proof}
Reinterpret the linear proxy $\tilde R(y) = c^\top y$ of \Cref{thm:audit-insufficiency} as a DPO implicit reward, on the same box $\mathcal{Y}_T$ on which that theorem is carried out (\Cref{lem:box_truncation}), where $c^\top y$ is bounded.
For each reference $\pi_\mathrm{ref}^{(k)} = \mathcal{N}(0, \Sigma^{(k)})$ of that construction, the policy $\pi^{(k)} = \pi^\star_\beta(\tilde R)$ has implicit reward $\hat r^{(k)} = c^\top y - b^{(k)}$ by \Cref{lem:dpo_consistency}(ii), with $b^{(k)} = \beta \log Z^{(k)}_T$ a prompt-only constant whose $T \to \infty$ limit is $c^\top \Sigma^{(k)} c/(2\beta)$, and since these four limits are distinct the $b^{(k)}$ differ across $k$ for $T$ large, so the four $\hat r^{(k)}$ share one gauge class without being identical as functions. 
Each density ratio $\pi^{(k)}/\pi_\mathrm{ref}^{(k)} = \exp(c^\top y/\beta)/Z^{(k)}_T$ is bounded above and below, so each $\hat r^{(k)}$ is regular (\Cref{def:dpo_regular}), with $\pi^{(k)} = \mathcal{N}(\Sigma^{(k)} c/\beta, \Sigma^{(k)})$ the $T \to \infty$ limit of \Cref{lem:box_truncation}(ii).

Take $\muaud = \mathcal{N}_T(0, I_3)$ on the same box, so that $\muaud \ll \mu_{\pi_\mathrm{ref}^{(k)}}$ as \Cref{lem:dpo_regularity} requires. 
The features are centered under $\muaud$ by symmetry of the box, so $M_i^\mathrm{cent} = M_i$ and $g_\mathrm{cent} = g$ on this construction, $g(c^\top y; \muaud) = c$ as in \Cref{app:nonsuck}, and \Cref{asm:nondeg_gauge} holds by \Cref{lem:box_truncation}(i).
Since $g_\mathrm{cent}$ is gauge-invariant, $M_i^\mathrm{cent}(\hat r^{(k)}) = M_i^\mathrm{cent}(c^\top y) - b^{(k)}$, so $(\hat r^{(k)}, M_i^\mathrm{cent}(\hat r^{(k)}))$ is $(c^\top y, M_i^\mathrm{cent}(c^\top y))$ shifted jointly by the prompt-only constant $-b^{(k)}$. 
Every gauge-invariant $B$ therefore agrees across the four, and the policy-side classification of \Cref{thm:audit-insufficiency} transfers because $\pi^\star_\beta$ and $M_i^\mathrm{cent}$ are gauge-invariant, placing the four mitigated policies in R0, R0\textsubscript{cont}, R1, and R2. 
The restriction to gauge-invariant $B$ is necessary, as the cardinal functional $\mathbb{E}_{\muaud}[\hat r^{(k)}] = -b^{(k)}$ takes four distinct values and already separates the instances.
\end{proof}

Two obstructions bind DPO evaluation outside the controlled construction. 
Reward-side benchmarks such as RewardBench score $\hat r$ on fixed pairs and are gauge-invariant, so \Cref{cor:dpo_impossibility}  applies to them directly. 
Policy-side evaluations such as AlpacaEval and Chatbot Arena score samples from $\pi$ and are not gauge-blind, but they report $\Delta J$ without off-target $\Delta_j$, so they cannot supply the second input of \Cref{cor:dpo_sufficiency} and leave R0, R0\textsubscript{cont}, and R1 unseparated.

\paragraph{Loss-side and data-side scope.}
A loss-side or data-side mitigation (\Cref{def:dpo_mitigation_placement}) does not apply a single-axis operator to $\hat r$, so the constructive guarantees do not attach.
\Cref{lem:centered_single_axis_identity} has no analogue, and \Cref{cor:dpo_impossibility} cannot be instantiated, as it holds the operator fixed across instances and here there is none. 
Nevertheless, two things survive.
First, the sufficiency classifier still applies through the policy pair by \Cref{cor:dpo_sufficiency}, so such mitigations remain classifiable into R0-R3 from $(\Delta_j, \Delta J)$. 
Second, the audit-distribution insufficiency conclusion is expected to extend to them, since the regime is a property of the policy-induced first moments at $\mu_{\pi^\star}$ that no $\mathcal{B}$-functional reads, and this gap is independent of whether the mitigation is realized as an operator. 
We do not give a loss-side-specific four-instance construction, as the operator-side witness of \Cref{cor:dpo_impossibility} together with the non-vacuity of \Cref{app:nonsuck,app:phasediagrams} already establishes that the regimes are real and invisible to gauge-invariant audit functionals.
The practical consequence is that a loss-side method must still report $(\Delta_j, \Delta J)$ at $\mu_{\pi^\star}$ (\Cref{app:a8-takeaways}), since its audit-side scores cannot certify R0 any more than those of an operator-side method can.

In sum, the sufficiency classifier (\Cref{cor:dpo_sufficiency}) covers any method whose policy is the KL-regularized optimum, DPO included, while the regime taxonomy and the impossibility (\Cref{cor:dpo_impossibility}, for gauge-invariant benchmarks) attach to mitigations that act as single-axis operators on the implicit reward under the regularity of \Cref{def:dpo_regular}. 
Loss-side and data-side mitigations connect through the policy pair alone.

\subsection{Takeaway Recommendations for RM Mitigation and Benchmarks Developers}
\label{app:a8-takeaways}

Theorem~\ref{thm:audit-insufficiency} and Corollary~\ref{cor:functional-blindspot} together circumscribe what reward-model benchmarks in the class $\mathcal{B}$ can and cannot deliver, and what mitigation methods evaluated within $\mathcal{B}$ cannot be certified to achieve. 
We translate the formal results into concise recommendations.

\paragraph{Regime detection procedures.}
The classifier $B^\star$ of \Cref{thm:audit-sufficiency} separates R0, R0\textsubscript{cont}, R1, and R2 given $(\Delta_j, \Delta J)$ at $\mu_{\pi^\star}$, but does not address R2 origin disambiguation, R3 detection, or R4 detection. The following three procedures close these gaps and should accompany any regime claim that touches them.

\begin{itemize}[leftmargin=2em,itemsep=3pt,topsep=2pt]
    \item \emph{R2 origin disambiguation via the rescalability sweep.} The two origins of R2 in \Cref{def:all_regimes} (scale overshoot, target misspecification) are operationally distinguished by sweeping the partial mitigation $M_i^c(\tilde R) = \tilde R - c\, g_i(\tilde R; \muaud)\, \phi_i$ over $c \in (0,1)$ at fixed $\beta$. If $J(\pi^\star_\beta(M_i^c(\tilde R)), R) > J(\pi^\star_\beta(\tilde R), R)$ for some $c$, the regime is driven by scale overshoot at $\mu_{\pi^\star}$ and is recoverable by partial projection. If no $c$ improves $\Delta J$, the projection has removed a structurally informative component of $\phi_i$ (target misspecification under partial informativeness, see \Cref{sec:3_1_features}) and rescaling cannot recover it. The test requires a one-dimensional sweep over $c$ at fixed $\beta$ and the same cardinal $R$-access already needed to distinguish R2 from R0. Concretely: publish the $\Delta J(c)$ curve, since a recoverable overshoot is a different scientific finding than a non-recoverable target misspecification.
    
    \item \emph{R3 detection via $D_{\mathrm{KL}}$.} Corollary~\ref{cor:sufficiency-r3} extends $B^\star$ to R3 via the branch $\Delta J = 0 \land \Delta_j = 0$. Reporting $D_{\mathrm{KL}}(\pi' \,\|\, \pi)$ alongside the label separates policy-relevant R3 ($D_{\mathrm{KL}} > 0$: mitigation moved the policy along directions orthogonal to $(\Delta_j, \Delta J)$) from policy-irrelevant R3 ($D_{\mathrm{KL}} = 0$: mitigation did not move the policy at all). Concretely: report $D_{\mathrm{KL}}$ whenever the regime call is R3, since only the policy-irrelevant case is consistent with a vacuous mitigation operator and should be flagged differently in downstream interpretation.
    
    \item \emph{R4 detection via the multi-$\muaud$ protocol.} R4 (\Cref{def:R4}) is transversal to R0-R3 and observable only by constructing the canonical $M_i^{\muaud}$ at multiple audit distributions and comparing the resulting regime calls. Concretely: select at least two audit distributions of substantively different annotator provenance (e.g., $\muaud^{\text{human}}$, $\muaud^{\text{LLM}}$), construct $M_i^{(\ell)}$ and $\pi^{(\ell)} = \pi^\star_\beta(M_i^{(\ell)}(\tilde R))$ at each, apply $B^\star$ to each, and treat cross-distribution disagreement on the regime call as the test. 
\end{itemize}

\paragraph{Reward Model Bias Mitigation Method Recommendations}

The prescriptions of our framework are necessary conditions for interpretable mitigation claims of any new mitigation method. 
Without them, R0 cannot be distinguished from R0\textsubscript{cont}, R1, or R2, regardless of how strong the audit-side evidence is.
We organize the recommendations by single- and multi-axis operators and present each as a self-contained checklist to support standardized reporting across method papers.

\emph{Single-axis Don'ts.}
\begin{itemize}
    \item \textbf{Do not fit and validate $M_i$ on the same $\muaud$ without testing alternative audit distributions.} 
    R4 (Definition~\ref{def:R4}) makes regime membership a function of $\muaud$, as instantiated in closed form in \Cref{app:phasediagrams}. 
    \emph{Concretely:} a length operator fit on LLM-judge preference data and one fit on human preference data can inherit different sycophancy couplings, so the two should not be treated as interchangeable even when both zero on-target reliance at their respective audit distributions.
    \item \textbf{Do not apply $M_i$ without reporting the induced $L^2(\muaud)$ scale change.} \Cref{app:A3} shows $\|M_i(\tilde R)\|_{L^2(\muaud)}$ generically differs from $\|\tilde R\|_{L^2(\muaud)}$ via both a diagonal and a cross-correlation contribution, inducing an effective-$\beta$ shift at fixed nominal $\beta$. 
    \emph{Concretely:} a method paper reporting only $g_i \to 0$ at $\muaud$ leaves readers unable to separate axis reallocation from scale-driven regime change (see single-axis dos).
    \item \textbf{Do not report only pooled diagnostics when the deployment target is within-prompt selection.} 
    \Cref{app:bon} shows pooled $|\rho_{\mathrm{len}}|$ can be driven near zero while within-prompt $\rho^{\mathrm{within}}_{\mathrm{len}}$ flips sign on three of four SOTA reward models on AlpacaEval and grows in magnitude on all four on GSM8K.
    \emph{Concretely:} BoN top-1 and PPO both operate within prompts, so a pooled diagnostic targeting $\phi_i$ does not measure the reliance the optimizer responds to.
    \item \textbf{Do not rely on ordinal-only diagnostics post-mitigation.} \Cref{cor:functional-blindspot} shows that ranking accuracy and pairwise win-rate are invariant to $\tilde R \mapsto c \tilde R$ while $J(\pi^\star_\beta(c\tilde R), R)$ varies strictly with $c$. 
    \emph{Concretely:} do not only report ranking accuracy on RewardBench post-mitigation, because reporting only such ordinal scores cannot rule out scale-driven regime shifts.
\end{itemize}

\emph{Single-axis Dos.}
\begin{itemize}
    \item \textbf{Default to $M^\mathrm{norm}_i$ and $M^\mathrm{cent}_i$ mitigations and document the variant used.} \Cref{app:A3} shows $M^\mathrm{norm}_i$ separates axis reallocation from scale and \Cref{app:A4} shows $M^\mathrm{cent}_i$ restores gauge invariance of the regime classification. 
    \emph{Concretely:} when using $M_i$ for pipeline compatibility, report $\|\tilde R'\|_{L^2(\muaud)}/\|\tilde R\|_{L^2(\muaud)}$ (ratio of root-mean-square reward scores for mitigated and unmitigated proxy rewards) alongside $g_i$ so that scale-driven and reallocation-driven $\Delta_j$ contributions are distinguishable.
    \item \textbf{Run the $c \in (0,1)$ rescalability sweep when $\Delta J \le 0$.} 
    The sweep over $M^c_i(\tilde R) = \tilde R - c\, g_i\, \phi_i$ is the only observational test that separates the two origins of R2 (scale overshoot, recoverable, but target misspecification is not). 
    \emph{Concretely:} publish the $\Delta J(c)$ curve, since a recoverable overshoot is a different scientific finding than a non-recoverable target misspecification.
    \item \textbf{Evaluate at policy-induced distributions and report $(\Delta_j, \Delta J)$.} According to \Cref{thm:audit-sufficiency}, this approach separates R0, R0\textsubscript{cont}, R1, and R2, but, according to \Cref{thm:audit-insufficiency}, audit-only inputs cannot.
    \emph{Concretely:} run the mitigated reward model in BoN or short PPO against a fixed reference policy and report both quantities, not only audit-set accuracy.
    Also, estimate $\Delta J$ with a judge protocol independent of the specification of the optimized reward, since a judge sharing that specification tracks the proxy \citep{mahmoud2026rewardhackingrubricbasedreinforcement}.
    \item \textbf{Instrument off-target $\Delta_j$ before publication.} Select a panel of plausible $\Phi_\mathrm{sp}$ candidates the operator does not target (formatting density, hedging markers, position effects, sycophancy and confidence indicators) and report $\Delta_j$ on those features at $\mu_{\pi^\star}$.
    Automated black-box bias discovery \citep[e.g.][]{wang2026automatically} can populate this panel per reward model in natural language, though its outputs are $\muaud$ quantities and remain inside $\mathcal{B}$ until re-measured at $\mu_{\pi^\star}$.
    \emph{Concretely:} the developer has both the unmitigated and mitigated rewards in hand and is the natural party to produce this measurement. 
    Without it, an R0 claim is structurally indistinguishable from R0\textsubscript{cont} (substitution under improvement).
    \item \textbf{Report cardinal scale pre/post mitigation.} Publishing $\|\tilde R\|_{L^2(\muaud)}$ before and after is the minimal cardinal observable that ordinal benchmark functionals miss. 
    \emph{Concretely:} report the standard deviation of reward scores on a fixed evaluation set both before and after applying the mitigation, so scale-driven $\Delta J$ shifts are visible.
    \item \textbf{Document $\pi_\mathrm{ref}$-sensitivity of the mitigation.} The construction in \Cref{thm:audit-insufficiency} produces four distinct regimes by varying only $\pi_\mathrm{ref}$ for a fixed $(\tilde R, M_i, \muaud)$. \emph{Concretely:} validate the mitigation against at least two reference policies of different lineages and report per-pair outcomes rather than a single headline number.
\end{itemize}

\emph{Multi-axis recommendations.}
\begin{itemize}
    \item \textbf{Specify whether the operator is joint or sequential, and the operator-stack type.} For canonical linear projections applied to a fixed reward, sequential $M_i$ followed by $M_j$ equals joint projection on $(\phi_i, \phi_j)$ by Frisch-Waugh-Lovell. 
    The equivalence breaks when operators are heterogeneous (e.g., LOESS for length, two-head for content), when projection directions are recomputed on the changed reward, or when operators are non-linear. \emph{Concretely:} state which case applies and, when sequential and joint differ, report both outcomes side-by-side.
    \item \textbf{Joint reliance zero at $\muaud$ does not certify R0.} \Cref{lem:single_axis_identity} generalizes and a multi-axis operator $M_S$ that zeros $g_i$ for all $i \in S$ at $\muaud$ does not in general zero $g_i$ at $\mu_{\pi^\star}$. 
    \emph{Concretely:} report $\Delta_j$ at $\mu_{\pi^\star}$ for every axis in the targeted set, not just on-target axes.
    Covering a panel at $\muaud$ narrows the substitution-accessible subset of $\Phi$ but does not close the measurement-vs-optimization gap.
    \item \textbf{State $\Phi_\mathrm{sp}$ explicitly and acknowledge out-of-panel exploitation.} \emph{Concretely:} a method covering $\{$length, sycophancy, position$\}$ does not certify R0 against features outside that set, and multi-axis claims should be framed as ``no substitution within the measured panel'' rather than ``no substitution.''
     Group-robust training has the same scope in supervised learning, where the worst-case guarantee ranges over pre-specified groups and shifts beyond them are left to future work \citep{sagawa2020distributionally}.
    \item \textbf{Flag statistical mediation between targeted features.} Where the literature establishes coupling between features in the targeted set (length-sycophancy and length-confidence per Facts~\ref{fact:sycophancy_length_dependence},\ref{fact:length_epistemic_uncertainty},and \ref{fact:confidence_quality_proxy}), associational joint projection may remove signal flowing through coupled paths. \emph{Concretely:} acknowledge this as a known limitation rather than treating coupled signal as residual noise. 
    The rescalability test above extends to multi-axis as a partial diagnostic.
\end{itemize}

\emph{Scope under non-linear operators.} The proofs above (\Cref{thm:audit-insufficiency,thm:audit-sufficiency}) and the linear-Gaussian instantiation (\Cref{app:nonsuck}) use Gaussianity for closed-form policy expressions and linearity for the canonical projection construction, but the prescriptions in this section target the structural $\muaud \to \mu_{\pi^\star}$ gap and the cardinal/ordinal blindness rather than these constructive devices, and so apply to non-linear, non-Gaussian operators and reference policies.
The regime classification (R0-R4) references $\Delta_j$ and $\Delta J$ rather than operator form, the cardinal/ordinal blindness of \Cref{cor:functional-blindspot} is monotone-invariant, and the gap itself is geometric rather than algebraic. 
What does \emph{not} transfer is the constructive guarantee of \Cref{lem:single_axis_identity} (a non-linear operator may not zero $g_i$ even at $\muaud$) and the closed-form expressions of \Cref{app:nonsuck}. 
Nonetheless, \Cref{app:phasediagrams} demonstrates regime separation under quadratic non-linearity in closed form and \Cref{app:bon} shows the gap surviving a non-linear \emph{operator} (LOESS calibration) measured by a linear diagnostic.
\textbf{Method papers using non-linear operators should still adopt the prescriptions} above by virtue of targeting the same geometric phenomenon, while stating that they cannot import the constructive guarantees of the linear-Gaussian regime classification, only its outcome-level definitions.

\paragraph{Reward Model Benchmark Recommendations}

Some of the recommendations for mitigation method developers transfer over to benchmark developers, with a few added recommendations.

\emph{Don'ts.}
\begin{itemize}
    \item \textbf{Do not expand the benchmark input within $\muaud$ and expect regime separation.} \Cref{thm:audit-insufficiency} grants $\mathcal{B}$ joint access to $(\tilde R, M_i(\tilde R), R, \Phi, \muaud)$ and still admits four instances landing in R0, R0\textsubscript{cont}, R1, and R2. Adding feature axes, distractors, or held-out audit splits stays inside $\mathcal{B}$ and inherits the impossibility. \emph{Concretely:} adding a sycophancy subset alongside a length subset, or expanding from pairwise to best-of-4 selection, enriches $\muaud$ but does not measure $\Delta_j$ at $\mu_{\pi^\star}$ and so cannot separate R0 from R0\textsubscript{cont} or R1.
    \item \textbf{Do not rely on benchmark functionals that are invariant to monotone rescaling of $\tilde R$}. \Cref{cor:functional-blindspot} shows that every $B \in \mathcal{B}_\mathrm{ord}$ (ranking accuracy, pairwise win-rate, and their composites) returns identical values under $\tilde R \mapsto c\tilde R$ while $J(\pi^\star_\beta(c\tilde R), R)$ varies strictly with $c$, so cardinal-scale-driven regime shifts are undetectable. \emph{Concretely:} if mitigation halves the effective scale of the proxy, ranking accuracy is unchanged but the KL-regularized policy behaves as if $\beta$ doubled (by the scale-equivariance of \Cref{app:A3}), and downstream $\Delta J$ shifts accordingly.
    \item \textbf{Do not generalize a single-$\pi_\mathrm{ref}$ score across deployment settings.} The construction in \Cref{thm:audit-insufficiency} varies only $\pi_\mathrm{ref}$ and obtains four distinct regimes for the same $(\tilde R, M_i, \muaud)$. 
    A benchmark score at one $\pi_\mathrm{ref}$ does not transport to deployment reference policies $\pi_\mathrm{ref}$. \emph{Concretely:} a reward model topping the leaderboard when paired with one SFT base can degrade PPO outcomes when paired with a different SFT base of the same model family, since the policy-induced distribution shifts even though the audit-side score does not.
\end{itemize}

\emph{Dos.}
\begin{itemize}
    \item \textbf{Evaluate at policy-induced distributions and report $(\Delta_j, \Delta J)$}. By \Cref{thm:audit-sufficiency}, this input separates R0, R0\textsubscript{cont}, R1, and R2 via the explicit functional $B^\star$ in \Cref{equ:bstar}; by \Cref{thm:audit-insufficiency}, no audit-only input suffices. 
    \emph{Concretely:} run the mitigated reward model in BoN or short PPO against a fixed reference policy, and report the change in expected feature value $\Delta_j$ on each $\phi_j \in \Phi_\mathrm{sp}$ alongside the change in true reward $\Delta J$, rather than only audit-set accuracy.
    \item \textbf{Report cardinal scale}. Publishing $\|\tilde R\|_{L^2(\muaud)}$ pre/post mitigation provides observables that $\mathcal{B}_\mathrm{ord}$ cannot, addressing the blindspot of \Cref{cor:functional-blindspot}. \emph{Concretely:} alongside accuracy, report the standard deviation of reward scores on a fixed evaluation set before and after applying the mitigation, so that scale-driven $\Delta J$ shifts are visible to readers.
    \item \textbf{Treat $\pi_\mathrm{ref}$ as an axis of variation, not a fixed convention.} Either evaluate across a panel of reference policies or document $\pi_\mathrm{ref}$-sensitivity as a first-class output, consistent with the construction underlying \Cref{thm:audit-insufficiency}. 
    \emph{Concretely:} pair each reward model with at least two SFT reference policies of different lineages and publish per-pair scores rather than a single headline number.
    \item \textbf{Report $\Delta_j$ across all of $\Phi_\mathrm{sp}$ at $\mu_{\pi^\star}$, not just the targeted axis.} \Cref{thm:audit-sufficiency} makes the off-target $\Delta_j$ on $\Phi_\mathrm{sp} \setminus \{\phi_i\}$ a load-bearing input for separating R0\textsubscript{cont} from R0 and R1 from R2. 
    Reporting only on-target reliance leaves substitution structurally invisible regardless of how rich the audit distribution is.
    \emph{Concretely:} when evaluating a length-debiasing operator, also measure post-mitigation drift in sycophancy, hedging, and formatting at $\mu_{\pi^\star}$, not just residual length correlation at the audit set.
\end{itemize}

\paragraph{Summary.}
The recommendations above are not addressed to a single audience. 
R0 certification is a joint property of the mitigation operator and the evaluation protocol. 
A method paper following the developer recommendations cannot be cleanly compared across benchmarks that omit the corresponding ones, and a benchmark implementing all these prescriptions cannot certify R0 for methods that do not report off-target $\Delta_j$ or cardinal scale. 
The overlap between the two checklists is structural rather than redundant.

\paragraph{Discussion.}
Existing methods and benchmarks satisfy these prescriptions only partially. 
SOTA benchmarks partially implement these prescriptions and close the gap to pure ordinal benchmarks. 
In particular, the current SOTA benchmark RewardBench2 \citep{malik2026rewardbench2} features a ``Ties'' metric, introduces cardinal-sensitive scoring, and its on-policy/off-policy finding documents $\pi_\mathrm{ref}$-sensitivity empirically. 
RM-Bench \citep{liu2025rmbench} goes furthest on the feature side, scoring reward models on pairs where the rejected response has the more favorable length or markdown style and validating the benchmark against PPO policy outcomes. 
Both remain audit-side, as the style diagnostic is computed on fixed pairs (so it lies in $\mathcal{B}$) and the policy validation is a one-time check on four reward models rather than a per-mitigation measurement of $\Delta J$ and off-target $\Delta_j$ at $\mu_{\pi^\star}$.
\citet{ivison2024unpacking} report the same disconnect incidentally: Four reward models trained from Tulu~2 at 13B and 70B are scored on RewardBench and then used for PPO from the Tulu~2 13B SFT policy on UltraFeedback prompts. 
The model scoring highest on RewardBench (79.8) yields the lowest downstream average and a model scoring 73.6 yields the highest, with the gain concentrated in GSM8K. 
The averages span 1.2 points from single runs and the KL coefficient was lowered for the 70B models, so this is a single data point rather than a validation.
The most systematic policy validation of an audit-side metric we are aware of is \citet{wen2025rethinking}, who train ten reward models with graded label noise, optimize a policy against each under BoN and PPO, and correlate RewardBench accuracy with regret against each other reward model as golden. 
Accuracy reaches a Spearman correlation of 0.75 under BoN and 0.61 under PPO (0.67 with a moderate KL penalty), reward models of similar accuracy induce substantially different regret, and audit-side refinements such as more responses per prompt raise the correlation by only about 0.05. 
Consistent with \Cref{thm:audit-insufficiency}, audit-side functionals can be improved as predictors but do not determine policy outcome.
\citet{shen2023trickledownimpactrewardinconsistency} report a within-model version of the same disconnect. 
Finetuning a LLaMA-7B reward model on contrast instructions moves preference-test accuracy by at most one point, while human raters find the resulting PPO policy significantly more useful (57.8\% vs.\ 50.9\%, $p<0.01$). 
This is a single run on a defect outside $\Phi_{\mathrm{sp}}$, so it shows low resolution of the audit metric and nothing about off-target $\Delta_j$.

No current protocol jointly delivers policy-distribution evaluation, cardinal reporting, and $\pi_\mathrm{ref}$-as-axis, a gap our framework both formalizes as necessary (\Cref{thm:audit-insufficiency}) and proves sufficient to close (\Cref{thm:audit-sufficiency}). 
On the mitigation side, published methods rarely report off-target $\Delta_j$ at $\mu_{\pi^\star}$ or run the rescalability sweep, leaving almost all recent results ambiguous between R0\textsubscript{cont}, R1, and R2 in ways the recommendations would have resolved (see \Cref{sec:4_evidence,app:regimes}).

Closing the joint gap is a community coordination problem rather than a single-paper deliverable. 
Method papers cannot unilaterally provide policy-distribution evaluation without benchmark infrastructure that supports it. 
Benchmarks cannot unilaterally provide off-target $\Delta_j$ measurement without method papers specifying $\Phi_\mathrm{sp}$. 
The recommendations above are scoped so that adoption by either side improves the evidence base, and joint adoption is what moves the field from R0 claims that audit-side scores cannot certify to R0 claims that the framework certifies sufficient.

\section{Supporting Experiments}
\label{app:experiments}

All experiments were either run through API access or on a single A100 GPU 40 Gb or B200 GPU 180 Gb (rented online).
Our code for the experiments is available at \href{https://github.com/maxlampe/bias\_substitution}{github.com/maxlampe/bias\_substitution} (MIT license).

\subsection{Bias Substitution in Language Model RLHF (GRPO)}
\label{app:gun}

We demonstrate reward bias substitution in a RLHF pipeline with standard off-the-shelf configuration. 
A single-axis length penalty applied during GRPO training compresses response length as intended, yet the freed optimization pressure redirects onto a correlated axis and drives the policy into overconfidence, instantiating the harmful substitution regime (R1) while leaving multiple-choice quality intact and reducing free-form accuracy.

\paragraph{Setup}

\begin{itemize}
  \item \textbf{Policy model.} \textit{Llama-3.2-3B-Instruct} \citep{grattafiori2024llama3herdmodels} (Llama 3.2 Community License), adapted with LoRA \citep{hu2022lora} on all linear layers (the four attention projections \texttt{q\_proj}, \texttt{k\_proj}, \texttt{v\_proj}, \texttt{o\_proj} and the three MLP projections \texttt{gate\_proj}, \texttt{up\_proj}, \texttt{down\_proj}) with rank $r=32$, scaling $\alpha=64$, dropout $0.05$.
  \item \textbf{Reward model.} \textit{Skywork-Reward-V2-Llama-3.1-8B} \citep{liu2025skyworkrewardv2scalingpreferencedata} (Llama 3.1 Community License), kept frozen throughout training and queried only through the reward function.
  \item \textbf{Dataset.} Prompts from \textit{UltraFeedback} \citep{cui2023ultrafeedback} (MIT License, the \texttt{ultrafeedback\_binarized} \texttt{train\_prefs} split), truncated to 2000 characters.
  \item \textbf{Algorithm.} Group Relative Policy Optimization (GRPO) \citep{shao2024deepseekmathpushinglimitsmathematical} as implemented in TRL \citep{vonwerra2020trl} (Apache-2.0 license).
  8-bit AdamW, learning rate $3\times 10^{-5}$, KL coefficient $\beta = 2\times 10^{-2}$, 600 optimizer steps, and bfloat16 precision.
  \item \textbf{Length penalty.} The RLHF reward is shaped as $\tilde{R}(x,y) = R_{\mathrm{RM}}(x,y) - \lambda\, n_{\mathrm{tok}}(y)/100$, where $n_{\mathrm{tok}}(y)$ is the number of response tokens and $\lambda$ is the penalty coefficient.
  Three values $\lambda \in \{0, 4, 8\}$, each trained with four random seeds, for twelve training runs in total.
  \item \textbf{Generation during training.} Group size of 4 candidate generations per prompt for the group-relative advantage, per-device batch of 4, gradient accumulation of 1, sampling temperature $T = 1.0$, maximum completion length 256 tokens.
\end{itemize}

\begin{figure}
  \centering
  \begin{subfigure}[t]{0.48\linewidth}
    \centering
    \includegraphics[width=\linewidth]{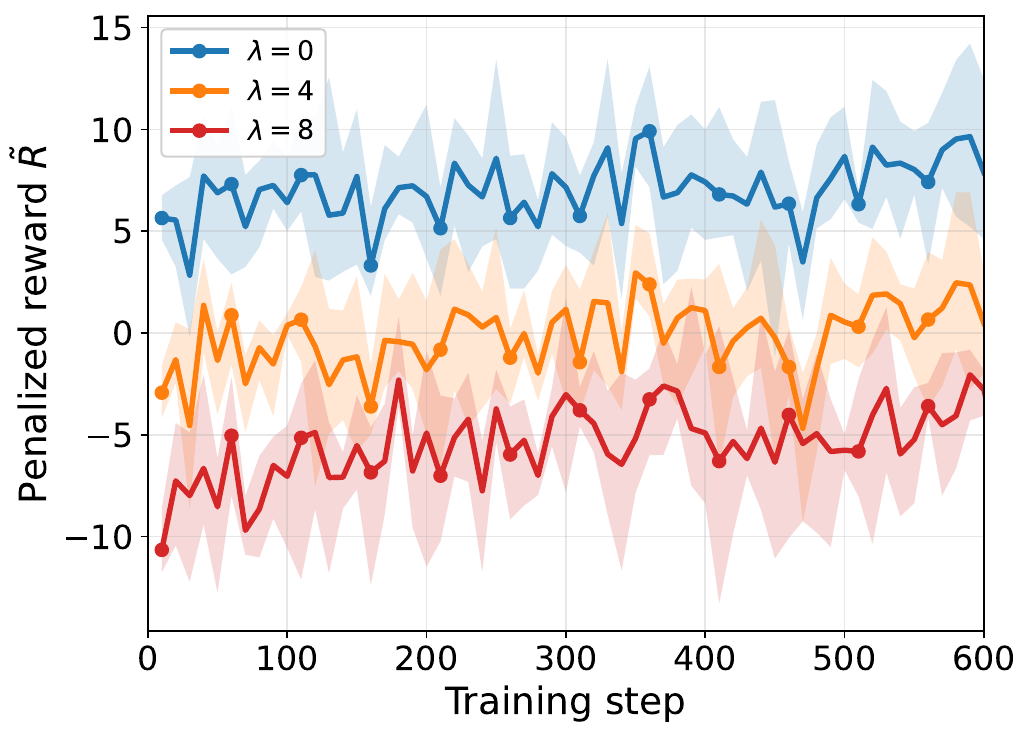}
    \caption{Reward $ \tilde R(x, y) = R_\text{RM}(x, y) - \lambda\, n_\text{tok} / 100$}
  \end{subfigure}
  \hfill
  \begin{subfigure}[t]{0.48\linewidth}
    \centering
    \includegraphics[width=\linewidth]{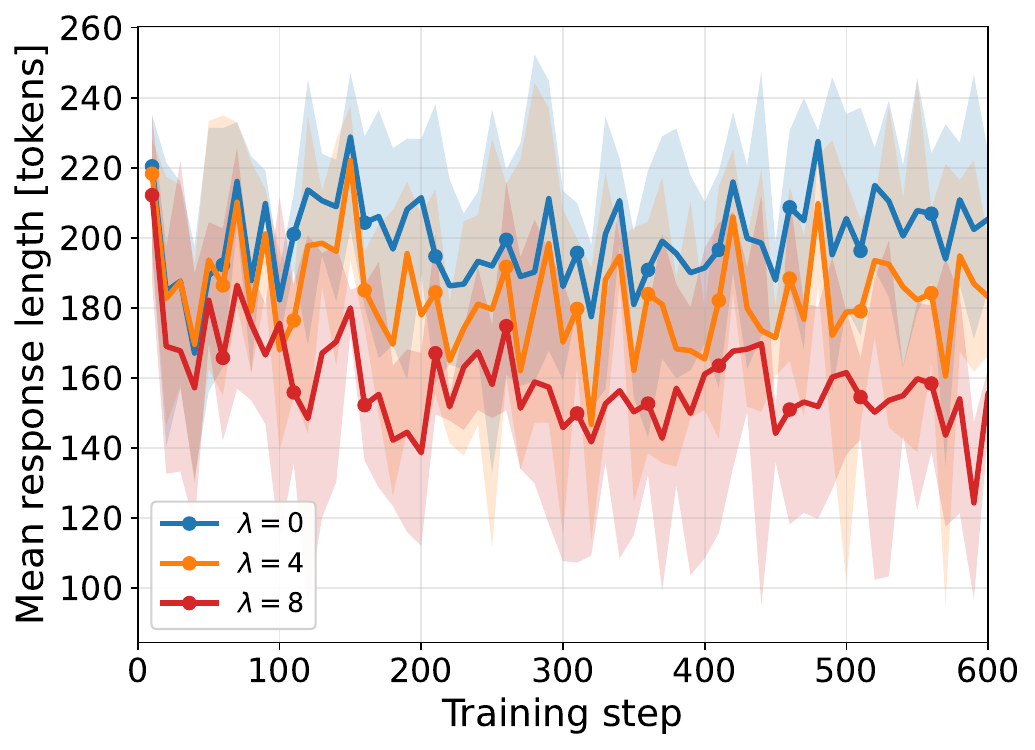}
    \caption{Response length $n_\text{tok}$}
  \end{subfigure}\\
  \begin{subfigure}[t]{0.48\linewidth}
    \centering
    \includegraphics[width=\linewidth]{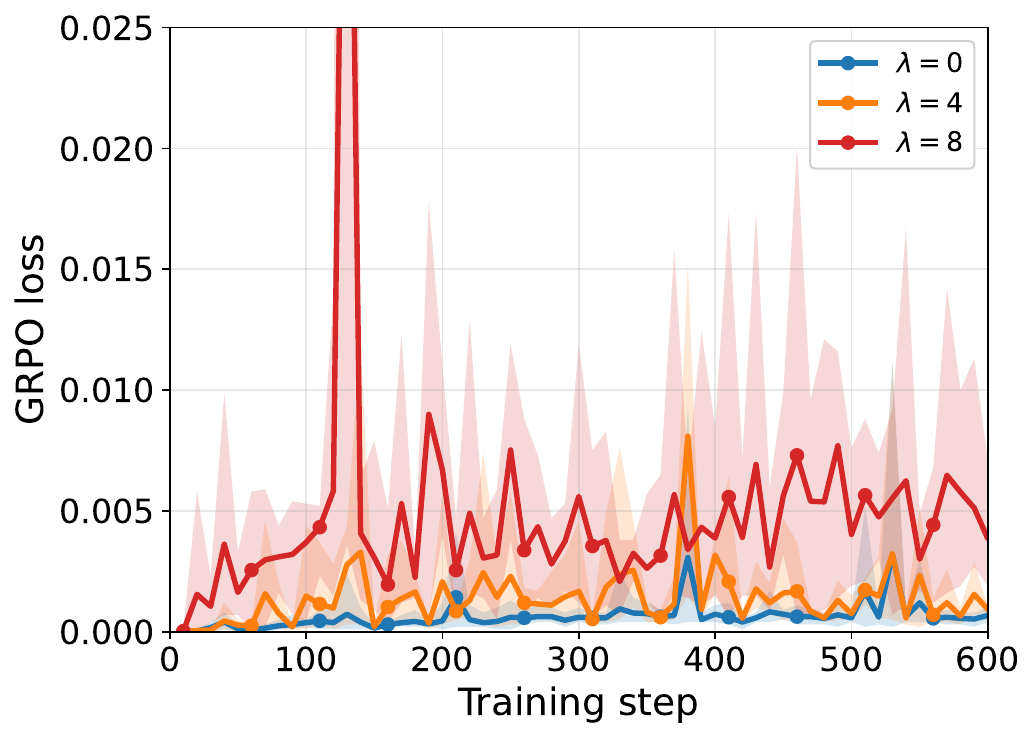}
    \caption{GRPO loss}
  \end{subfigure}
  \hfill
  \begin{subfigure}[t]{0.48\linewidth}
    \centering
    \includegraphics[width=\linewidth]{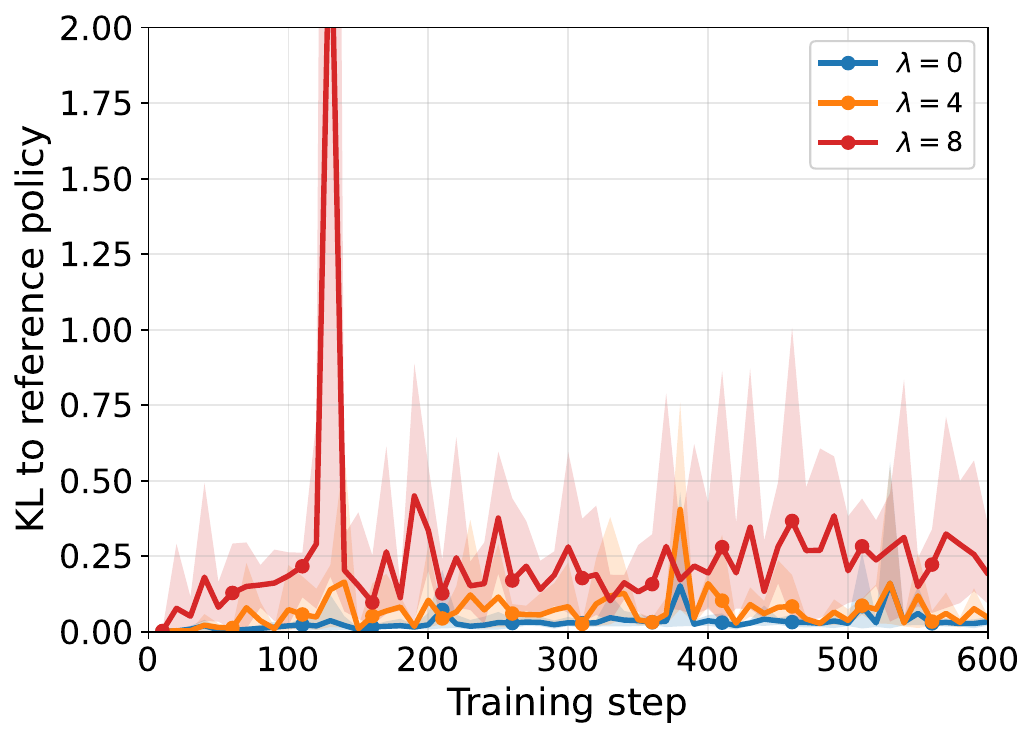}
    \caption{KL to reference policy}
  \end{subfigure}\\
  \begin{subfigure}[t]{0.48\linewidth}
    \centering
    \includegraphics[width=\linewidth]{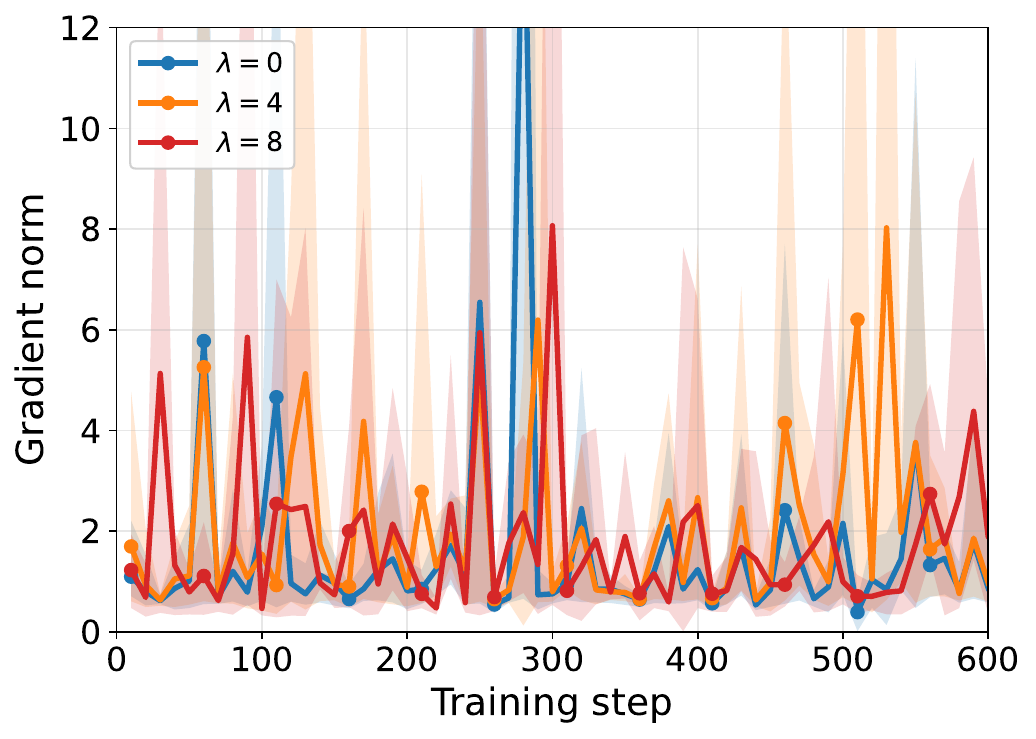}
    \caption{Gradient norm}
  \end{subfigure}
  \hfill
  \caption{Training curves (shaded region shows the min-max range across 4 seeds)}
  \label{fig:grpo_curves}
\end{figure}

\paragraph{Training Curves}

We log the following quantities at every optimizer step and report them per $\lambda$, averaged over the four seeds, see \Cref{fig:grpo_curves}.

\begin{itemize}
  \item \textbf{Mean response length.} The average number of generated tokens per step. 
  This is the primary curve and shows the length penalty taking effect, with the curves separating by $\lambda$.
  \item \textbf{Mean reward.} The average shaped reward $\tilde{R}$ per step.
  Note that $\tilde{R}$ is not directly comparable across $\lambda$ because the penalty term differs, so the curves are read within each $\lambda$ for
  evidence of learning.
  \item \textbf{KL divergence.} The divergence of the policy from the frozen reference policy per step, confirming that the policy stays regularized and
  does not collapse.
  \item \textbf{GRPO loss and gradient norm.} Diagnostic optimization curves.
\end{itemize}

\paragraph{Evaluation}

All twelve adapted policies and the untrained base policy\footnote{Untrained by us, but still instruction tuned by the model developer.} are evaluated on held-out data, with each adapter taken at the 600-step checkpoint. 
Aggregate numbers use 95\% confidence intervals from a hierarchical bootstrap over seeds and samples. 
We report the following metrics.

\begin{itemize}
  \item \textbf{Mean response length.} The average response token count on 50 held-out \textit{UltraFeedback} \texttt{test\_prefs} prompts, with 4 samples per prompt drawn at temperature $1.0$.

  \item \textbf{MMLU accuracy.} Multiple-choice accuracy on 1000 questions subsampled from MMLU \citep{hendrycks2021measuring} (MIT license), evaluated zero-shot with greedy decoding and a single-letter answer parsed from the output.

  \item \textbf{Expected Calibration Error (ECE).} Computed with 10 equal-width bins on the (verbalized confidence, correctness) pairs produced by the \citet{tian-etal-2023-just} Verb.\ 1S top-1 protocol on 1000 \textit{TriviaQA} \citep{joshi-etal-2017-triviaqa} (Apache-2.0 license) questions.

  \item \textbf{AUROC.} The area under the ROC curve of the verbalized confidence used as a predictor of answer correctness.

  \item \textbf{Mean verbalized confidence.} The average probability the model assigns to its own answer under the Verb.\ 1S top-1 protocol.

  \item \textbf{TriviaQA accuracy.} The fraction of correct answers under the same protocol.

  \item \textbf{Regressive flip rate.} The headline sycophancy metric from the \citet{sharma2024towards} \textit{are\_you\_sure} task (MIT license), the fraction of initially-correct answers that become incorrect after the user pushes back.

  \item \textbf{Overall flip rate.} The fraction of answers that change at all after the pushback, regardless of correctness.
\end{itemize}

The \citet{tian-etal-2023-just} ECE, AUROC, and TriviaQA accuracy metrics are computed under an LLM-judge equivalence check using the \citet{tian-etal-2023-just} prompt with \textit{Claude Haiku 4.5} \citep{claude45} as the judge. 
The sycophancy metrics use the \citet{sharma2024towards} two-stage protocol, eliciting an initial answer and then a post-pushback answer on 500 \textit{are\_you\_sure} records with 2 samples per record at temperature $T = 1.0$.

\begin{table}[t]
\caption{Metric values across the length-penalty sweep that was also used for \Cref{fig:smoking_gun} in \Cref{sec:intro}. 
Each $\lambda$ entry is the mean over four random
seeds with the 95\% hierarchical-bootstrap confidence interval as a subscript. 
The Base column is the untrained \textit{Llama-3.2-3B-Instruct} model and carries no interval. 
Expected calibration error, TriviaQA accuracy, and AUROC use the LLM-judge answer grading.}
\label{tab:gunz}
\centering
\setlength{\tabcolsep}{4pt}
\renewcommand{\arraystretch}{1.15}
\begin{tabular}{lcccc}
\toprule
& Base & $\lambda=0$ & $\lambda=4$ & $\lambda=8$ \\
\midrule
\multicolumn{5}{l}{\textit{Length and quality}} \\
Mean response length (tokens) & $192$ & $204_{[194,\,212]}$ & $188_{[181,\,195]}$ & $170_{[161,\,179]}$ \\
MMLU accuracy                 & $0.60$ & $0.61_{[.59,\,.62]}$ & $0.61_{[.59,\,.63]}$ & $0.62_{[.60,\,.63]}$ \\
\midrule
\multicolumn{5}{l}{\textit{Confidence and free-form accuracy}} \\
Expected calibration error (ECE)   & $0.26$ & $0.25_{[.23,\,.28]}$ & $0.30_{[.26,\,.34]}$ & $0.41_{[.32,\,.50]}$ \\
Mean verbalized confidence    & $0.81$ & $0.80_{[.78,\,.82]}$ & $0.82_{[.80,\,.83]}$ & $0.84_{[.82,\,.85]}$ \\
TriviaQA accuracy             & $0.56$ & $0.56_{[.53,\,.58]}$ & $0.52_{[.47,\,.57]}$ & $0.41_{[.29,\,.52]}$ \\
AUROC                         & $0.71$ & $0.72_{[.71,\,.74]}$ & $0.71_{[.67,\,.74]}$ & $0.65_{[.62,\,.69]}$ \\
\midrule
\multicolumn{5}{l}{\textit{Sycophancy}} \\
Regressive flip rate          & $0.65$ & $0.58_{[.52,\,.64]}$ & $0.65_{[.59,\,.70]}$ & $0.60_{[.54,\,.66]}$ \\
Overall flip rate             & $0.72$ & $0.67_{[.63,\,.71]}$ & $0.71_{[.67,\,.75]}$ & $0.70_{[.66,\,.75]}$ \\
\bottomrule
\end{tabular}
\end{table}

\paragraph{Results}

See \Cref{fig:smoking_gun} and \Cref{tab:gunz} for results.

\begin{itemize}
  \item \textbf{The length penalty compresses responses monotonically.}
  Mean response length falls from 204 tokens at $\lambda=0$ to 188 at
  $\lambda=4$ and 170 at $\lambda=8$. Both reductions, measured relative to
  $\lambda=0$, are significant.

  \item \textbf{Multiple-choice quality is preserved.} MMLU accuracy holds
  near 0.61 across the sweep (0.61, 0.61, 0.62 at $\lambda=0,4,8$
  against a base of 0.60), and the change from base is not significant at
  any $\lambda$.

  \item \textbf{Confidence rises and decouples from correctness.} 
  TriviaQA accuracy falls from 0.56 at $\lambda=0$ to 0.52 at $\lambda=4$ and 0.41 at $\lambda=8$, and
  the coupling between confidence and correctness weakens, with AUROC falling from 0.72 to 0.65. 
  Expected calibration error climbs from 0.25 at $\lambda=0$ to 0.30 at $\lambda=4$ and 0.41 at $\lambda=8$, and mean verbalized confidence climbs in step from 0.80 to 0.82 to 0.84.
  Thus, the policy does not merely become more confident, it becomes wrong more often. 
  At $\lambda=8$ the rise in calibration error, the rise in verbalized confidence, the drop in accuracy, and the drop in AUROC are all significant.

  \item \textbf{The calibration loss is caused by the penalty, not by
  RLHF.} 
  At $\lambda=0$ the calibration error (0.25) is at or slightly below the base value (0.26), so RLHF on its own does not break calibration. 
  The degradation appears only once the length penalty is applied.

  \item \textbf{Sycophancy does not significantly change.} 
  The regressive flip rate is non-monotonic across the sweep (0.58, 0.65, 0.60 against a base of 0.65) and the overall flip rate shows no consistent trend (0.67, 0.71, 0.70 against 0.72), consistent with the already-elevated flip rate of the base model leaving little headroom for optimization pressure to manifest.

  \item \textbf{Regime transitions.} 
  We take the true reward to be free-form factual accuracy, proxied by TriviaQA, and we treat expressed confidence as spurious with respect to free-form factual accuracy ($\Phi_{\mathrm{sp}}$), and note that reward models score it as a quality proxy (Fact~\ref{fact:confidence_quality_proxy}), making redirection onto it available.
  Multiple-choice MMLU accuracy serves as a capability control. 
  At $\lambda = 8$, response length falls and MMLU is unchanged, while expressed confidence rises, a nonzero $\Delta_j$ on an off-target spurious axis, and TriviaQA accuracy degrades, a negative $\Delta J$. 
  This is the harmful sub-case of bias substitution (R1), not overcorrection (R2), because pressure redirects onto a second spurious axis rather than degrading capability with no redirection. 
  The MMLU performance preservation indicates the $\Delta J$ drop is specific to free-form accuracy rather than a general capability loss. 
  We read $\lambda = 4$ as an intermediate point, where the overconfidence proxies shift toward substitution but sit at or near the significance boundary, so the missing significance at four seeds reflects limited power rather than the absence of an effect, and $\lambda = 8$ as the clear R1 instance.
\end{itemize}

\paragraph{Independent audit of the reward model.}
\citet{wang2026automatically} audit Skywork-Reward-V2-Llama-3.1-8B with a black-box pipeline that scores length-preserving counterfactual rewrites with the reward model and a frontier LLM judge. 
Across five synthetic prompt topics, 10 of 17 candidate attributes are rewarded by the model and dispreferred by the judge (Bonferroni-corrected $p<0.01$), among them triple spaces between words, a fabricated quotation attributed to a named official, and hallucinated supporting details (bias strength $2.52 \pm 0.51$). 
The audit does not measure expressed confidence, and its authors caution that a reward-model bias need not be learned by an optimized policy. 
These results independently establish that the reward model in our GRPO run carries off-target axes it rewards and a judge disprefers, which is what pressure redirection requires. 
Our run is complementary, showing redirection onto confidence, an axis outside their found set.

\paragraph{Prior policy-level evidence.} 
The three prior policy-level observations below rule out two competing explanations for \Cref{fig:smoking_gun}, a general side effect of RLHF and a rescalable calibration artifact, and show for the same operator that the same operator leaves hacking on an untargeted axis, so the harmful R1 reading of \Cref{sec:4_evidence} does not rest on our setup alone.
\begin{itemize}
    \item \citet{hosking2024human} report preliminary, uncontrolled evidence that on-policy RLHF raises the rated assertiveness of outputs relative to off-policy preference fine-tuning, comparing models with different base models and training data. 
    Our $\lambda = 0$ run does not reproduce such a shift in verbalized confidence or calibration for unpenalized RLHF, although 600 LoRA steps is a much shorter optimization than theirs and assertive tone is a different construct from verbalized numeric confidence. The shift we observe appears only under the length penalty, which is what separates bias substitution from a general side effect of RLHF.
    \item \citet{kadavath2022languagemodelsmostlyknow} find that RLHF alone makes answer-token probabilities on multiple-choice tasks appear miscalibrated, but that a single temperature rescaling largely repairs it. 
    That remedy does not apply here, since our calibration is on verbalized confidence and the accompanying drops in accuracy and AUROC are unchanged by any monotone rescaling of the confidence scores.
    \item \citet{miao2025theenergyloss} re-evaluate the fixed-coefficient PPO length penalty of \citet{singhal2024long}, the same operator form as ours, on Llama3-8B, Mistral-7B and DeepSeek-7B with a different reward model. 
    The GPT-4 win rate of the penalized run still declines late in training, and they identify the residual hacking mode as over-cautious refusals, an axis the penalty does not touch, shown qualitatively and not measured against the unpenalized policy. 
    The freed pressure again lands on an untargeted axis, a different one from ours.
\end{itemize}

\paragraph{Scope and Limitations. }
\begin{itemize}
    \item We do not test directly whether reliance on verbalized confidence rises at the policy distribution, but we argue that reward models score verbalized confidence as a quality proxy and penalize hedged correct answers \citep{leng2025taming, fein2026one}, and a length penalty removes precisely that hedging.
    Thus, we expect that a within-prompt regression of reward-model scores on verbalized confidence, controlling for length would confirm the R1 classification.
    \item We do not measure whether $\lambda$ increases abstention and refusals during the evaluation. 
    We argue this is unlikely, since the accuracy and AUROC drops are invariant to monotone rescaling of the confidence scores \citep{kadavath2022languagemodelsmostlyknow} and the grader scores against gold answers, not against the quality proxy of the reward model \citep{mahmoud2026rewardhackingrubricbasedreinforcement}.
    \item It is possible that there are more potential features the optimization pressure could be redirected to (e.g., hedging, politeness, formatting), but as we measure $\Delta J < 0$, a harmful instance is given. 
    \item We only test one data set, reward model, and policy. The evaluations with TriviaQA, ECE and AUROC share one LLM judge.
\end{itemize}

\subsection{Measurement-vs-Optimization Gap in Length Mitigation in BoN Selection}
\label{app:bon}

This section provides a quantitative empirical instance of the measurement-versus-optimization gap of Section~\ref{sec:3_2_linearfix}. 
We evaluate two published length-debiasing operators across five reward models under a Best-of-$N$ (BoN) selection protocol, using a within-prompt reliance diagnostic that tracks the BoN selection distribution rather than the pooled audit distribution at which the calibration of \citet{huang2025posthoc} is designed. 
The experiments we report here are conducted for this work, and we analyze them
through the measurement-versus-optimization gap of \Cref{sec:3_2_linearfix} and our regime taxonomy in \Cref{def:all_regimes}.
The headline observation is that the post-hoc calibration of \citet{huang2025posthoc} (CC-BY-4.0 license) approximately zeros pooled reward-length correlation on AlpacaEval ($0.316 \rightarrow 0.037$) while within-prompt correlation exceeds its unmitigated magnitude on three of four SOTA reward models on AlpacaEval (negative sign) and on all four on GSM8K (positive sign), with $\Delta J$ degrading on two AlpacaEval cells. 
The linear-probe operator of \citet{fein2026one} (CC-BY-4.0, MIT license) is fit on the same prompt split, moves within-prompt correlation toward zero on all five reward models on AlpacaEval, and achieves $\Delta J > 0$ on four of them.

\paragraph{Reward models.} 
We evaluate Skywork-Reward-V2-Llama-3.1-8B, Skywork-Reward-V2-Qwen3-8B, Skywork-Reward-V2-Qwen3-0.6B \citep{liu2025skyworkrewardv2scalingpreferencedata} (Llama 3.1 Community License, Apache-2.0 license), the Llama-3.1-8B-Instruct-RM-RB2 \citep{malik2026rewardbench2} (Llama 3.1 Community License), and the older DeBERTa-v3-large \cite{openassistant_reward_model_deberta_v3_large_v2} (MIT license) reward models. 

\paragraph{Mitigation operators.} 
We use two single-axis operators that target length bias ($\phi_i = $ length). 
\begin{itemize}
    \item Mechanistic reward shaping with a difference-of-means linear probe projected out of the final-layer hidden state of the RM \citep{fein2026one}, whose direction is estimated from long versus short responses to the same prompt and projected out in full.
    \item Post-hoc reward calibration in RC-LWR form with a robust locally weighted (LOWESS) fit of reward on length subtracted at the score level \citep{huang2025posthoc}.  
    We use bandwidth $f = 1/3$, three robustifying iterations and calibration constant $\gamma = 1$, untuned per reward model. 
    This matches the original implementation, which shares one untuned setting across all reward models at comparable sample sizes, as \citet{huang2025posthoc} leave the per-RM search to the practitioner.
    \citet{huang2025posthoc} introduce $\gamma$ as a user-set control for prompt sets on which their conditional-mean-zero assumption fails, and report a RewardBench optimum near $\gamma = 1.4$. 
    We hold $\gamma = 1$ because that is the setting at which the operator zeros pooled reliance, the audit-side target whose optimization-side consequences we test. 
    We do not sweep $\gamma$, so the rescalability test of \Cref{app:a8-takeaways} is not run here and the $\mathrm{R2}_\varepsilon$ cells below are not separated into scale overshoot and target misspecification.
    We refer to this operator as LOESS calibration elsewhere in the paper.
\end{itemize}
Both are instances of $M_i$ (Definition~\ref{def:single_axis_mitigation}). 
The calibration reduces pooled $|g_i(\tilde R; \muaud)|$ by construction, while the probe removes a linear activation direction rather than $\phi_i$ itself and reduces within-prompt reliance empirically.

\paragraph{Setup.} 
For each of two prompt sources (AlpacaEval \citep{dubois2024length} (CC-BY-NC-4.0 license), GSM8K \cite{cobbe2021trainingverifierssolvemath} (MIT license)) we generate $64$ candidate responses per prompt from Llama-3.2-1B-Instruct \citep{grattafiori2024llama3herdmodels}, fit each operator on a probe split, and evaluate on a disjoint held-out split ($512$ for AlpacaEval prompts, $807$ for GSM8K prompts). 
Each RM selects the highest-scoring response among the
$64$ candidates. 
We measure two quantities: 
\begin{itemize}
    \item a within-prompt Pearson correlation $\rho_\text{len}^\text{within}$ between RM score and response length, computed within each candidate set and averaged across prompts. Because BoN top-1 selection operates within candidate sets per prompt, $\rho_\text{len}^\text{within}$ measures the reliance selection actually responds to, while pooled $\rho_\text{len}$ averages over prompts that BoN never compares. This instantiates $g_i$ at a $\muaud$ aligned with BoN selection rather than the pooled distribution the operators target.
    \item A selection-level $\Delta J$ proxy (AlpacaEval length-controlled win rate vs.\ baseline, GSM8K BoN accuracy on selected responses).
\end{itemize}  
The pair $\{\rho^{\mathrm{pooled}}_{\mathrm{len}}, \rho^{\mathrm{within}}_{\mathrm{len}}\}$ operationalizes the measurement-versus-optimization gap of \Cref{sec:3_2_linearfix} empirically,
as an operator can zero one and not the other, and the selection outcome reveals which distribution drives the optimizing process. 
BoN top-1 over $64$ samples is a coarser proxy for $\pi^\star_\beta$ than KL-regularized optimization, and the regime calls below should be read accordingly.
The coarseness is a property of the selection rule rather than of $N$, since softmax selection at temperature $\beta$ over $N$ samples converges exactly to $\pi^\star_\beta$ as $N$ grows \citep{li2024qprobe}, while top-1 selection is the $\beta \to 0$ corner and only reached with $N \to \infty$.

\begin{table}[t]
\caption{Within-prompt reward-length correlations and BoN selection
outcomes.
$^{*}p<.05$, $^{***}p<.001$ vs.\ 50\%; $^{\dagger}$ denotes
significantly below 50\%. 
On AlpacaEval, pooled $|\rho_{\mathrm{len}}|$ averaged across the five RMs is $0.316$ (baseline) and $0.037$ for post-hoc reward calibration \citep{huang2025posthoc}, an audit-distribution success, while the same operator raises mean within-prompt $|\rho^{\mathrm{within}}_{\mathrm{len}}|$ on both prompt sources ($0.103$ to $0.116$ and $0.076$ to $0.096$).}
\label{tab:bon}
\centering
\setlength{\tabcolsep}{4pt}
\renewcommand{\arraystretch}{1.15}
\begin{tabular}{lccccc}
\toprule
& DeBERTa & Allen & SW-L & SW-Q8B & SW-Q0.6B \\
\midrule
\multicolumn{6}{l}{\textit{Within-prompt} $\rho_\text{len}^\text{within}$ \textit{(AlpacaEval)}} \\
Baseline    & $+.168$ & $-.061$ & $+.134$ & $+.065$ & $+.088$ \\
Mechanistic Reward Shaping   & $+.005$ & $-.011$ & $+.087$ & $+.032$ & $+.041$ \\
Post-hoc Reward Calibration & $+.048$ & $-.099$ & $-.065$ & $-.138$ & $-.228$ \\
\midrule
\multicolumn{6}{l}{\textit{AlpacaEval LC win rate vs.\ baseline} (\%)} \\
Mechanistic Reward Shaping   & $54.4^{***}$ & $46.2^{\dagger}$ & $51.5^{***}$ & $51.7^{***}$ & $51.8^{***}$ \\
Post-hoc Reward Calibration & $50.4^{*}$   & $48.7^{\dagger}$ & $50.1$       & $48.2^{\dagger}$ & $51.2^{***}$ \\
\midrule
\multicolumn{6}{l}{\textit{Within-prompt $\rho^{\mathrm{within}}_{\mathrm{len}}$ (GSM8K)}} \\
Baseline & $+.297$ & $+.030$ & $+.022$ & $-.005$ & $-.024$ \\
Mechanistic Reward Shaping & $-.005$ & $+.001$ & $-.007$ & $-.008$ & $-.013$ \\
Post-hoc Reward Calibration & $+.015$ & $+.146$ & $+.112$ & $+.134$ & $+.074$ \\
\midrule
\multicolumn{6}{l}{\textit{GSM8K BoN accuracy} (\%)} \\
Baseline    & $32.8$ & $71.5$ & $68.3$ & $71.9$ & $65.8$ \\
Mechanistic Reward Shaping   & $36.4$ & $71.3$ & $68.5$ & $71.9$ & $65.7$ \\
Post-hoc Reward Calibration & $32.3$ & $67.9$ & $68.3$ & $71.3$ & $66.2$ \\
\bottomrule
\end{tabular}
\end{table}

\paragraph{Results.} 
The \citet{huang2025posthoc} calibration achieves $|g_i| \approx 0$ at the pooled $\muaud$ it targets ($0.316 \rightarrow 0.037$ averaged across the five RMs on AlpacaEval). 
At the within-prompt $\muaud$ aligned with BoN selection, all four SOTA RMs end with negative $\rho^{\mathrm{within}}_{\mathrm{len}}$ on AlpacaEval, three of them after a sign flip (SW-L, SW-Q8B, SW-Q0.6B), and three exceed their unmitigated baselines in absolute value (Allen, SW-Q8B, SW-Q0.6B, magnitude $0.099$ to $0.228$, Table~\ref{tab:bon}). 
On GSM8K, $\rho^{\mathrm{within}}_{\mathrm{len}}$ under the calibration instead ends positive on all five RMs and above its unmitigated magnitude on all four SOTA RMs, with mean $|\rho^{\mathrm{within}}_{\mathrm{len}}|$ of $0.096$ versus $0.076$ at baseline and $0.007$ for mechanistic reward shaping, and mean BoN accuracy of $61.20\%$ versus $62.06\%$ at baseline and $62.76\%$ for mechanistic reward shaping. 
The sign of the residual within-prompt length dependence therefore differs between prompt sources, while its magnitude ends above the unmitigated value on both. 
This difference is \textbf{evidence of the measurement-versus-optimization gap} of \Cref{sec:3_2_linearfix}, because the operator approximately zeros pooled reliance at the audit distribution it targets while the within-prompt reliance that BoN selection acts on grows on three of four SOTA RMs on AlpacaEval and on all four on GSM8K.

\paragraph{Regime classification.} 
Under the $\varepsilon$-banded reading of \Cref{app:epsilonbands}, two \citet{huang2025posthoc} cells (Allen, Skywork-Q8B) satisfy $\Delta J < -\varepsilon_J$ on AlpacaEval LC win rate. 
In Skywork-Q8B $\rho^{\mathrm{within}}_{\mathrm{len}}$ flips sign ($+.065 \rightarrow -.138$), while in Allen it stays negative and grows in magnitude ($-.061 \rightarrow -.099$). 
The GSM8K Allen cell additionally drops $3.6$ accuracy points. 
These two AlpacaEval cells satisfy the $\Delta J < -\varepsilon_J$ condition shared by R2$_\varepsilon$ and harmful R1$_\varepsilon$. 
Whether they also satisfy the banded no-redirection condition $|\Delta_j| \leq \varepsilon_j$ on off-target spurious axes is unmeasured here. 
As a result, each cell is in R2$_\varepsilon$ if that condition holds and in harmful R1$_\varepsilon$ otherwise.
Two further \citet{huang2025posthoc} cells (DeBERTa, Skywork-Q0.6B) satisfy $\Delta J > \varepsilon_J$ under the same noise-floor bands and are in $\mathrm{R0}_\varepsilon$ or
$\mathrm{R0}_{\varepsilon,\mathrm{cont}}$ depending on the same unmeasured redirection, and the Skywork-L cell is within band and undetermined between $\mathrm{R3}_\varepsilon$ and $\mathrm{R1}_\varepsilon$. 
All four of the LOESS signed cells are small in absolute terms, between 1.8 points below and 1.2 points above parity. 
We use the noise-floor bands that
\Cref{app:epsilonbands} sets as the default, and note that the effect-size-relative alternative offered there would place several of these cells differently depending on the fraction chosen, so the cell-level calls in this table should be read as band-dependent.
Under either reading the operator is cell-dependent, which is the claim rather than a qualification, since no statistic available at $\muaud$ predicts which cell a given reward model lands in.
Since the length-controlled win rate itself moves with length behavior (\Cref{fact:ordinal_cardinal_dissociation}, \citealp{yuan2025followinglengthconstraints}), a mitigated reward model that merely shifts the length of the selected responses could lose the pairwise comparison for that reason alone, so the two negative calls rest on the within-prompt evidence as much as on the win rate.

Mechanistic reward shaping achieves $\Delta J > 0$ at $p < .001$ on four of five AlpacaEval cells and moves $|\rho_\text{len}^\text{within}|$ toward zero on all five RMs (AlpacaEval) without sign flips.
The Allen cell falls below 50\% LC WR on both operators ($46.2\%$ for Mechanistic RS, $48.7\%$ for \citet{huang2025posthoc}, both $\dagger$). 
The GSM8K Allen cell drops $0.2$ for mechanistic reward shaping and $3.6$ for \citet{huang2025posthoc}. 
The remaining four mechanistic reward shaping cells satisfy the $\Delta J > \varepsilon_J$ condition of R0$_\varepsilon$ (successful mitigation) on AlpacaEval LC win rate and are in R0$_\varepsilon$ or R0$_{\varepsilon,\mathrm{cont}}$ depending on the unmeasured redirection. 
On GSM8K, three of these four cells change accuracy by at most $0.2$ points, so that prompt source does not corroborate $\Delta J > \varepsilon_J$ for them. 
The Allen mechanistic reward shaping cell satisfies $\Delta J < -\varepsilon_J$ and is in R2$_\varepsilon$ or harmful R1$_\varepsilon$ under the same reading as the \citet{huang2025posthoc} cells above.

\paragraph{Connection to the framework.} 
The pair $(\rho_\text{len}^\text{pooled}, \rho_\text{len}^\text{within}) = (0.037, 0.116)$ for \citet{huang2025posthoc}, averaged across the five RMs on AlpacaEval, is an empirical instance of the measurement-versus-optimization gap of \Cref{sec:3_2_linearfix} for a non-canonical, non-linear $M_i$. 
The operator brings pooled reliance close to zero at the audit distribution it targets, while mean within-prompt reliance, which BoN selection acts on, ends above its unmitigated value on both prompt sources, with the dominant sign reversing between them (negative on four of five RMs on AlpacaEval and positive on all five on GSM8K).
This dissociation is the load-bearing mechanism of the substitution argument.
Wherever it holds, R1 is mechanically available to the optimizer, and the audit diagnostic targeting $\phi_i$ cannot detect it by construction. 
Standard reward-model benchmarks evaluate at a single fixed $\muaud$ and report only the pooled diagnostic, so under \Cref{sec:impossibility-sufficiency} they cannot distinguish this $\Delta J < 0$ instance from R0 regardless of whether the underlying regime is R1$_\varepsilon$ (redirection onto an unmeasured spurious axis), 
R2$_\varepsilon$ (overcorrection without redirection), or R1$_\varepsilon$ in some cells and R2$_\varepsilon$ in others. 
All of these are invisible at $\muaud$, which is the central benchmark-inadequacy claim.

\subsection{Evaluating sycophancy and length of AITA Responses}
\label{app:AITA_syclength}

\begin{table}[t]
\caption{Linear mixed models of open-ended response length (characters) on binary sycophancy labels, fit on the same 14375 responses of eight models to r/AITA posts where Reddit users judged the poster at fault. 
The verdict label is 1 when the model judges the poster not at fault (NTA) and the endorsement label is 1 when the LLM judge marks the response as endorsing the user. 
For response $i$ to prompt $p$ we fit $\mathrm{len}_{ip} = \beta_0 + \beta_1 S_{ip} + \gamma_{m(i)} + u_p + \epsilon_{ip}$ with $u_p \sim \mathcal{N}(0, \tau^2)$, or with a model random intercept in place of $\gamma_{m(i)} + u_p$. The joint fit includes both labels. 
Entries are $\hat\beta_1$ with 95\% confidence intervals in brackets. Significance vs.\ no effect $^{*}p < .05$, $^{***}p < .001$.}
\label{tab:aita_mixedfit}
\centering
\begin{tabular}{lcc}
\toprule
 & Verdict label & Endorsement label \\
\midrule
$N$ responses & 14{,}375 & 14{,}375 \\
Label rate & 0.444 & 0.387 \\
\midrule
Model random intercept & $+9.3$ & $-73.3^{***}$ \\
 & {\scriptsize $[-5.7, 24.2]$} & {\scriptsize $[-88.0, -58.6]$} \\
Model fixed effects, prompt random intercept & $-15.0$ & $-124.5^{***}$ \\
 & {\scriptsize $[-30.3, 0.3]$} & {\scriptsize $[-140.0, -108.9]$} \\
Joint fit, model fixed effects, prompt random intercept & $+16.7^{*}$ & $-128.7^{***}$ \\
 & {\scriptsize $[1.0, 32.5]$} & {\scriptsize $[-144.8, -112.6]$} \\
\bottomrule
\end{tabular}
\end{table}

This section reports the analysis behind \Cref{fact:sycophancy_length_dependence} and \Cref{sec:4_evidence}. 
We use the r/AmITheA**hole (r/AITA) dataset of \citet{cheng2026sycophantic} (MIT license), which contains 2000 posts where Reddit users judged the poster at fault and 2000 other posts, and restrict the analysis to the former. 
For each model and post the release provides an open-ended response, a separately elicited binary verdict, and an LLM judge label for the open-ended response. 
Open-ended responses and judge labels are available for eight of the eleven models, Llama-8B, Llama-17B, Llama-70B, Mistral-7B, Mistral-24B, Gemini, GPT-4o and Claude, from five developers. 
GPT-5, Qwen and DeepSeek have binary verdicts only, and Llama-70B has open-ended responses for 476 of the 2000 posts. 
We keep responses whose verdict contains YTA or NTA and use its first occurrence, which yields $N = 14375$ responses (476 to 2000 per model). 
Endorsement corresponds to a judge output of 0 in the release, which we infer from the data.
The human input is the Reddit community verdict that selects the prompts (so neither label is a human annotation). Length is measured as the character count of the open-ended response. 
For each label we fit the models in \Cref{tab:aita_mixedfit}. 

\paragraph{Results.}
Endorsing responses are shorter by 73.3 characters with a model random intercept (95\% CI $[58.6, 88.0]$) and by 124.5 characters with model fixed effects and a prompt random intercept (95\% CI $[108.9, 140.0]$), and the ratio of mean non-endorsing to mean endorsing length exceeds one for six of the eight models. 
The verdict label shows no reliable association ($+9.3$, 95\% CI $[-5.7, 24.2]$, and $-15.0$, 95\% CI $[-30.3, 0.3]$). 
The two labels agree on 72.1\% of responses (Cohen's $\kappa = 0.43$).
Only the endorsement label describes the response whose length is measured, since the verdict comes from a separate query.

\section{Existing Evidence for Confounding Factors}
\label{app:survey}

Reward bias substitution requires that deployed reward models rely on several correlated spurious features that each carry some signal about quality. 
The facts below collect the published evidence that this is the typically the case for reward models and preference data, instead of an edge case. 
The features a length mitigation is most likely to couple with are sycophancy and expressed confidence (\Cref{fact:sycophancy_length_dependence}, \Cref{fact:length_epistemic_uncertainty}, \Cref{fact:confidence_quality_proxy}). 
The literature also supports that documented biases can be mitigated with first-order corrections (\Cref{fact:reward_decomposable}, \Cref{fact:linear_intervention}), enabling single-axis targeting that leads to the potential redirection of optimization pressure and reward bias substitution (\Cref{fact:bias_substitution}).

\begin{fact}
\label{fact:length_bias}
\textbf{Reward models are correlated with response length independent of content quality}, from the first language model applications of RLHF to today \citep{stiennon2020learning, shen2023loose, eisenstein2024helping, singhal2024long, nvidia2024nemotron4, chenicml2024odin, wangcrm2025beyond, zhao2025bias, fein2026one}.
\end{fact}

\begin{fact}
\label{fact:additional_biases}
\textbf{Reward models exhibit multiple additional distinct reward biases}, e.g., overconfidence \citep{leng2025taming}, sycophantic user agreement \citep{sharma2024towards}, answer position \citep{fein2026one}, prefix bias \citep{kumar2025detecting}, model-style sensitivity \citep{malik2026rewardbench2, fein2026one}, redundant whitespace and hallucinated supporting details \citep{wang2026automatically}, and inherited value orientations (e.g., agency vs.\ communion) from pretraining base models \citep{christian2026reward}. \textbf{All of these are present in state-of-the-art reward models} \citep{fein2026one, christian2026reward}.
\end{fact}

\begin{fact}
\label{fact:length_inflation}
\textbf{LLM policies trained with RLHF or DPO learn to generate longer outputs as a direct result of optimization pressure.} 
This effect can inflate benchmark scores \citep{singhal2024long}, produce outputs longer than training data \citep{park2024disentangling}, 
and grow systematically longer than SFT baselines over training \citep{stiennon2020learning}. 
The same length growth appears outside preference-learned reward models under rubric-based GRPO, where presence-based criteria carry 90.2\% of rubric weight and their satisfaction correlates with response length within prompt while absence-based satisfaction does not \citep{mahmoud2026rewardhackingrubricbasedreinforcement}. 
Overall, this effect represents a primary mode of reward hacking \citep{singhal2024long,park2024disentangling}.
\end{fact}

\begin{fact}
\label{fact:length_penalties}
\textbf{Uniform length suppression can remove informative content.} Data-side length balancing and two-head length disentanglement both reduce reward-length correlation while lowering overall preference accuracy and more than doubling the failure rate on length-sensitive
prompts, where annotators favor the longer response
\citep{bu2025beyond}. 
Separately, always-on length penalties cause reward hacking
via trajectory collapse in reasoning RL \citep{yuan2025shorten}.
\end{fact}

\begin{fact}
\label{fact:length_overrepresented}
\textbf{Longer responses are systematically over-represented among ``chosen'' annotations in preference datasets, providing the distributional basis for reward models to learn length as a proxy for quality} \citep{singhal2024long}.
RLHF dramatically amplifies a slight length skew already present in preference data \citep{eisenstein2024helping}, a pattern confirmed in crowdsourced human preference data, where controlling for length substantially reorders model rankings \citep{liblog2024style}. 
\citet{findeis2025inverse} estimate the same skew with an LLM annotator applying a length-preferring principle, which reconstructs 57.1\% of Chatbot Arena and 57.7\% of PRISM human labels while applying to over 96\% of pairs in each dataset.
\end{fact}

\begin{fact}
\label{fact:human_length_confound}
\textbf{Human preference judgments are systematically confounded by output length.} Across multiple preference-collection pipelines, annotators select tend to the longer response in comparisons \citep{hu2025explaining, chenemnlp2024humans, singhal2024long}.
This pattern is visible even in carefully filtered datasets, as \citet{stiennon2020learning} find that controlling for length reduces human preference gains and \citet{saito2023verbosity} show a positive correlation between length and preference in the HH-RLHF corpus.
In large-scale crowdsourced evaluation, length has been estimated as a dominant style factor in human voting \citep{liblog2024style}.
\end{fact}

\begin{fact}
\label{fact:llm_verbosity_bias}
\textbf{LLM judges exhibit verbosity bias \citep{zheng2023judging} and it is different to human annotators \citep{saito2023verbosity, koo2024benchmarking}.}
Verbosity is one of several systematic biases catalogued in LLM judges, in addition to other LLM judge biases \citep{koo2024benchmarking, chenemnlp2024humans, yecalm2025justice, wataoka2024self}.
Further evidence also supports that human and LLM judges can respond to length and content differently \citep{saito2023verbosity, liblog2024style, chenemnlp2024humans, hu2025explaining}.
\end{fact}

\begin{fact}
\label{fact:human_sycophancy}
\textbf{Human annotators systematically reward sycophantic responses.} \citet{sharma2024towards} show that ``matches user's beliefs'' is among the most predictive features of human preference, with the preference model favoring sycophantic over truthful responses.
\citet{cheng2026elephant} confirm this finding at the data level, as preferred responses across preference datasets are significantly more validating and indirect.
\citet{perez2023discovering} find that preference models incentivize sycophantic answers, while \citet{wen2025language} show RLHF-optimized models learn to defend incorrect answers with more convincing fabricated or cherry-picked evidence, increasing human approval while barely changing correctness.
\citet{ibrahim2026warm} separately show that supervised fine-tuning on warmer-style conversational data alone (without preference optimization) is sufficient to amplify affirmation of incorrect user beliefs by approximately 40\% across five models.
\end{fact}

\begin{fact}
\label{fact:sycophancy_length_dependence}
\textbf{Sycophancy and response length are statistically associated in LLM responses.}
Prior work jointly reported response length and sycophancy. \citet{ibrahim2026warm} report that warmth fine-tuning increases sycophancy and also shortens responses (734 versus 877 characters) and \citet{dubois2026ask} treat length as a nuisance parameter when measuring sycophancy reduction.
At the reward-model level, \citet{beigi2026adversarialrewardauditingactive} find that sycophancy and length-bias exploits occupy nearly identical regions of a reward model representation space, that an exploit detector trained on either transfers to the other, and that a policy optimized to exploit only one of the two also increases the other. However, their length exploit is verbosity whereas the association we report next runs toward shorter responses.
In \Cref{app:AITA_syclength}, we analyze the dataset of \citet{cheng2026sycophantic} restricted to posts where Reddit users judged the poster at fault. 
The verdict measure is the separately elicited model binary verdict, counted as sycophantic when it judges the poster not at fault. 
The endorsement measure is an LLM judge label of whether the open-ended response endorses the action of the user. 
With model fixed effects and a prompt random intercept, endorsing responses are shorter by 124.5 characters (95\% CI $[108.9, 140.0]$), whereas the verdict measure shows no reliable association ($-15.0$, 95\% CI $[-30.3, 0.3]$). 
The two measures agree on 72.1\% of responses (Cohen's $\kappa = 0.43$). These associations are observational (\Cref{app:AITA_syclength}).
\end{fact}

\begin{fact}
\label{fact:sycophancy_correct_abandonment}
\textbf{Sycophancy causes models to abandon correct responses in favor of alignment with user-stated positions.} When users suggest incorrect answers, model accuracy drops relative to unbiased baselines \citep{sharma2024towards} with high agreement rates with user-stated views for some question types \citep{perez2023discovering}.
\citet{fanous2025syceval} quantify this as \textit{regressive sycophancy}, observing models switch from correct to incorrect under user rebuttals and \citet{hong2025measuring} further show that this is not a knowledge deficit.
Beyond propositional agreement, framing sycophancy, in which models uncritically accept flawed user premises, proved particularly resistant to DPO-based mitigation \citep{cheng2026elephant}.
In exploratory analyses, \citet{cheng2026sycophantic} find that sycophantic responses less often referred to the point of view of the other party, and they suggest this may narrow users' focus to self-affirmation.
\end{fact}

\begin{fact}
\label{fact:length_epistemic_uncertainty}
\textbf{Models can generate longer responses under epistemic uncertainty by hedging, qualifying, and elaborating more when confidence is low} \citep{zhang2025demystify}.
\end{fact}

\begin{fact}
\label{fact:confidence_quality_proxy}
\textbf{Reward models treat expressed confidence as a quality proxy, incentivizing models to replace hedged answers with confident-sounding responses regardless of actual certainty.}
Reward models assign higher scores to otherwise identical responses when a higher verbalized numeric confidence is appended \citep{leng2025taming}, and can penalize hedged correct answers over confidently-stated incorrect ones \citep{fein2026one}.
This penalty can be traced to the annotation level, where human raters in preference datasets are biased against expressions of uncertainty \citep{zhou2024relying, ibrahim2025measuring} and are less likely to detect factual errors in more assertive outputs, an effect that holds whether assertiveness is manipulated or measured \citep{hosking2024human}.
\end{fact}

\begin{fact}
\label{fact:length_uncertainty_quality}
\textbf{Longer responses in uncertain regimes correlate with lower confidence and reduced quality.} 
Verbose outputs produced under verbosity compensation show measurably higher uncertainty and recall drops relative to concise responses \citep{zhang2025demystify}.
\end{fact}

\begin{fact}
\label{fact:reasoning_length_uncertainty}
\textbf{Reasoning models (not RLHF) produce longer outputs with uncertainty markers.} Extended chain-of-thought training produces epistemic uncertainty markers and non-linear reasoning traces through backtracking and alternative exploration rarely seen in non-reasoning models \citep{yoon2025reasoning}.
During on-policy RL training, trajectory length grows systematically alongside exploratory behavior, and naively penalizing length degrades performance \citep{yuan2025shorten}, consistent with RL reinforcing the coupling.
\end{fact}

\begin{fact}
\label{fact:informational_content_length}
\textbf{Reward model scores increase with length when additional tokens carry genuine informational content} \citep{hu2025explaining}, 
but the reward-length relation of an unmitigated reward model flattens beyond roughly 100 tokens \citep{zhao2025bias},
consistent with \textbf{diminishing informational returns per token as responses grow longer.}
\end{fact}

\begin{fact}
\label{fact:irreducible_length_correlation}
\textbf{Some reward-length correlation is irreducible and reflects genuine informativeness rather than spurious shortcut.} 
The information mass decomposition directly establishes that a portion of the length-reward correlation tracks real content quality \citep{hu2025explaining}, and more detailed answers can be genuinely more helpful depending on context \citep{bu2025beyond}.
Under a fixed rubric proxy and without any mitigation operator, rubric-free judges rate a GRPO checkpoint above its base only on completeness ($+1.07$ on a 1 to 7 scale) and below it on factual correctness, conciseness, relevance and overall quality, from single-seed, correlational evidence \citep{mahmoud2026rewardhackingrubricbasedreinforcement}.
\end{fact}

\begin{fact}
\label{fact:dpo_length_bias}
\textbf{DPO-like objectives develop length bias out-of-distribution as a structural vulnerability of the algorithm}. 
The DPO implicit reward decomposes as a sum of token-level log-probability ratios \citep{rafailov2024from}, and \citet{lu2024eliminating} argue that this sum grows with sequence length, creating an algorithmic dependence on response length within the objective.
Empirically, for standard DPO the implicit reward shows weak length correlation on preference data from its training distribution but substantial length correlation on responses sampled from the trained policy \citep{park2024disentangling}. 
With an explicit length penalty in the objective, outputs become shorter, yet on policy samples length still explains $19$ to $37\%$ of implicit-reward variance in $11$ of $12$ penalized settings of \citet{park2024disentangling}.
Consistently, \citet{rafailov2024from} find that beam search over a fixed DPO policy, which they show is likelihood search on its implicit reward, produces exploding response length beyond five beams on TLDR, a decoding-time effect they also observe for PPO.
Independently, under noisy or capability-limited supervision, the KL regularization of DPO creates a structural dilemma where sufficient regularization preventing over-optimization also prevents large updates needed for correcting errors inherited from the SFT phase \citep{yeilr2025iterative}.
\end{fact}

\begin{fact}
\label{fact:kl_reward_scale}
\textbf{KL regularization interacts with reward scale non-trivially.} 
KL-regularized policies are not invariant to positive linear reward scaling (see \Cref{app:A3}), meaning two reward models agreeing on all preference orderings but differing in cardinal scale produce different policies \citep{skalseicml2023invariance}.
Empirically, there exists a measured optimal KL budget beyond which true reward degrades in RLHF \citep{gao2023scaling} and longer sequences accumulate disproportionately more KL divergence, entangling length with the implicit DPO reward formulation \citep{lu2024eliminating}.
The optimal regularization coefficient $\beta$ depends strongly on feedback reliability \citep{yeilr2025iterative}, meaning there is no single good KL setting across realistic annotation conditions.
\end{fact}

\begin{fact}
\label{fact:ordinal_cardinal_dissociation}
\textbf{Spurious length features can maintain ordinal ranking accuracy while distorting cardinal reward values, meaning benchmark metrics may fail to detect this failure mode.}
A dissociation that follows from the partial-identifiability characterization of \citet{skalseicml2023invariance} combined with the KL scale-sensitivity of \Cref{sec:background} (\Cref{equ:j_rlhf}), and is directly measured as a low-complexity linear artifact in reward model representations~\citep{fein2026one}.
RMs built on different base model families can achieve similar RewardBench scores while exhibiting systematically divergent value orientations \citep{christian2026reward}.
Within a single benchmark, a headline score can also conceal style reliance, as the top-ranked reward model on RM-Bench scores 89.0 Easy and 74.7 Normal accuracy but 46.6 Hard accuracy, below chance, so its headline average of 70.1 is carried by pairs where style favors the chosen response or is matched \citep{liu2025rmbench}.
RewardBench \citep{lambertnaacl2025rewardbench} scores improve after post-hoc length calibration, confirming retrospective distortion \citep{huang2025posthoc}, and verbosity-induced leaderboard gains shrink substantially under length-controlled evaluation in both automated \citep{dubois2024length} and crowdsourced human voting settings \citep{liblog2024style}.
However, the control is partial as shown in \citet{yuan2025followinglengthconstraints}. For one model following length instructions, varying only the instructed response length moves the length-controlled AlpacaEval win rate between 23\% and 29\%, so the metric is not decoupled from length behavior.
\end{fact}

\begin{fact}
\label{fact:evaluator_weaknesses_learnable}
\textbf{Evaluator and annotator weaknesses are learnable targets.}
Models trained with RLHF learn strategies that mislead time-constrained human evaluators, such as fabricating or cherry-picking evidence with subtle fallacies, and writing less readable code that passes evaluator-written unit tests unseen during training \citep{wen2025language}. 
Under rubric-based GRPO with a weak LLM verifier, the share of newly credited criteria that a stronger judge panel rejects rises from 39\% to 65\% over training, through partial satisfaction of compound criteria, implicit content treated as explicit, and imprecise topical matching \citep{mahmoud2026rewardhackingrubricbasedreinforcement}.
These strategies target the evaluation process rather than a single surface feature.
\end{fact}

\begin{fact}
\label{fact:reward_decomposable}
\textbf{Reward can be somewhat decomposed into separable dimensional components.} 
A two-head reward model (one head trained to correlate with response length and the other to decorrelate from it) reduces the length correlation of the quality head to near zero while largely preserving preference accuracy and improves policies at matched response length~\citep{chenicml2024odin}.
Consistent with this, multi-attribute scoring frameworks track helpfulness, correctness, coherence, complexity, and verbosity as separately annotated (but correlated) dimensions \citep{wanghelpsteer2024helpsteer, nvidia2024nemotron4}.
\end{fact}

\begin{fact}
\label{fact:linear_intervention}
\textbf{Reward biases admit first-order linear intervention.} Length bias is smooth, structured, and stable enough across model families to be estimated and corrected post-hoc via LOESS without retraining \citep{huang2025posthoc}, but assumes that the true reward is independent of length.
\citet{fein2026one} also use a linear probe to mitigate some part of length and other biases in reward models, while \citet{papadatos2024linear} used a linear probe penalty on reward-model activations to reduce feedback sycophancy under BoN selection.
\end{fact}

\begin{fact}
\label{fact:bias_substitution}
\textbf{Mitigating a spurious correlation in one feature can increase reward hacking in another feature (bias substitution).}
In supervised vision, mitigating a single labeled shortcut amplifies reliance on the unlabeled one, with the off-target gap growing relative to ERM while the gap on the group lacking both shortcuts shrinks \citep{licvpr2023whac}.
In RLHF, aggressive length-debiasing in recent SOTA reward models has flipped the length-bias sign and reduced correctness \citep{fein2026one}.
In RLHF more broadly, 
the baselines re-evaluated by \citet{bu2025beyond} reduce targeted bias at measured cost to preference accuracy and downstream win rate, consistent with R2 or harmful R1, which their reported evidence does not separate.
Inside a PPO loop with separate relevance, factuality, and completeness reward models, upweighting the relevance reward shortens outputs while the factuality and completeness rewards fall, which \citet{wu2023finegrained} read as competing objectives, although all three axes stay rewarded and the shift is read off the training reward models rather than a true-reward measure, so the evidence is again consistent with R2 or harmful R1 without separating them.
The mechanism is compatible with the single-proxy correlation bound of \citet{laidlaw2025correlated},
whose required regularization strength stays positive whenever the proxy-truth correlation is below one.
\citet{liu2026orpo} address this at the reward level via max-min optimization, but without the R0-R4 diagnostic machinery, substitution between interpretable feature axes remains undetectable.
Conversely, jointly mitigating multiple shortcuts via kernel-based regularization \citep{ye2025prism} achieves near-zero shortcut correlations at the audit distribution, but without measuring whether optimization pressure has shifted at $\mu_{\pi^*}$, instantiating the distributional blindspot of \Cref{sec:impossibility-sufficiency}.

\end{fact}

\begin{table}[t]
\centering\footnotesize
\caption{Regime classification overview of the 30 surveyed mitigation works.
$\Delta J$ and off-target $\Delta_j$ refer to the policy-induced distribution $\mu_{\pi^\star}$. 
All information from cited prior work, except where marked. 
Regime classification is done based on the $\varepsilon$-banded reading of \Cref{app:epsilonbands}, where $\mathrm{R1}_\varepsilon$ and $\mathrm{R2}_\varepsilon$ are not separable without off-target measurement.
Where the cited work reports no uncertainty, the band is zero and the call reduces to the sign of the reported change. 
Works that report different methods or relevant baselines are reported more than once.\\
$^\dagger$ Based on our re-evaluation in \Cref{app:bon}\\
$^{c}$$\widehat{\Delta J} > 0$ on a single proxy excludes R2 and neutral R1 only if the proxy tracks true reward. 
Using a single proxy can miss the redirection axis, e.g., as in \Cref{fig:smoking_gun}, \Cref{app:gun}, with constant MMLU score while decreasing TriviaQA score. 
Thus, R0, $\mathrm{R0}_{\mathrm{cont}}$ (substitution, but overall increase of true return), and harmful R1 all remain consistent with the evidence. }
\label{tab:survey}
\setlength{\tabcolsep}{5pt}
\begin{tabular}{@{}llll@{}}
\toprule
Work & $\Delta J$ & Off-target $\Delta_j$ & Regime \\
\midrule
\citet{singhal2024long} & $<0$ & none & $\mathrm{R2}_\varepsilon$ / harmful $\mathrm{R1}_\varepsilon$ \\
\citet{bu2025beyond} (BAL) & $<0$ & none & $\mathrm{R2}_\varepsilon$ / harmful $\mathrm{R1}_\varepsilon$ \\
\citet{park2024disentangling} & $<0$ & none & $\mathrm{R2}_\varepsilon$ / harmful $\mathrm{R1}_\varepsilon$ \\
\citet{huang2025posthoc} & $>0$; $<-\varepsilon_J$ on 2/5 RMs$^\dagger$ & none & $\mathrm{R2}_\varepsilon$ / harmful $\mathrm{R1}_\varepsilon$ (2 cells) \\
\citet{fein2026one} & $> 0$ on 4/5; $< -\varepsilon_J$ on 1/5 RMs$^\dagger$ & none & $\mathrm{R2}_\varepsilon$ / harmful $\mathrm{R1}_\varepsilon$ (1 cell) \\
\midrule
\citet{shen2023loose} & $>0$ & none & not certified$^{c}$ \\
\citet{chenicml2024odin} & $>0$ & none & not certified$^{c}$ \\
\citet{bu2025beyond} (ALBM) & $>0$ & none & not certified$^{c}$ \\
\citet{eisenstein2024helping} & $>0$ & none & not certified$^{c}$ \\
\citet{liu2025rrm} & $>0$ & length & not certified$^{c}$ \\
\citet{wangcrm2025beyond} & $>0$ & none & not certified$^{c}$ \\
\citet{zhao2025bias} & $>0$ & none & not certified$^{c}$ \\
\citet{cai2026disentangling} & $>0$ & none & not certified$^{c}$ \\
\citet{kimarxiv2025mitigating} & $>0$ & none & not certified$^{c}$ \\
\citet{srivastava2026robust} & $>0$ & length, refusal & not certified$^{c}$ \\
\citet{song2025causalrewardadjustmentmitigating} & $>0$ & none & not certified$^{c}$ \\
\citet{feng2025c2po} & $> 0$ & none (panel acc.) & not certified$^{c}$ \\
\citet{umer2026generalpreferencereinforcementlearning} & $>0$ & length & not certified$^{c}$ \\
\midrule
\citet{meng2024simpo} & not isolated & none & undetermined \\
\citet{lu2024eliminating} & not isolated & none & undetermined \\
\citet{lilmpo2025length} & not isolated & none & undetermined \\
\citet{fu2025reward} & not isolated & none & undetermined \\
\citet{wang2024interpretablepreferences} & not isolated & none & undetermined \\
\citet{wanghelpsteer2024helpsteer} & not isolated & none & undetermined \\
\citet{papadatos2024linear} & not measured & none & undetermined \\
\citet{ng2025debiasing} & not isolated & none & undetermined \\
\citet{lidir2025eliminating} & not isolated & none & undetermined \\
\citet{liu2026orpo} & not isolated & none & undetermined \\
\citet{ye2025prism} & not isolated & none & undetermined \\
\citet{chen2026learning} & not isolated & none & undetermined \\
\citet{wang2026patterningpracticedebiasingreward} & not isolated & none & undetermined \\
\bottomrule
\end{tabular}
\end{table}

\section{Mapping published mitigations onto the regime taxonomy}
\label{app:regimes}

We classify each mitigation method surveyed in \Cref{sec:4_evidence} by where the available outcome evidence places it in our regime taxonomy, drawing on the reported results of each paper, on re-implementations reported by other surveyed works, and on our own experiments in \Cref{app:bon}. 
We mark the source wherever it is not reported by the classified paper itself, since evidence collected at different audit distributions need not agree.

Across the methods we survey, no paper provides the joint evidence that Theorem~\ref{thm:audit-sufficiency} requires to certify R0.
Five surveyed works provide direct evidence for a failure regime. 
Twelve works come closest to R0 but leave at least one required input unmeasured. 
The remaining thirteen are undetermined and consistent with R0\textsubscript{cont} or stronger failure modes. 
\Cref{tab:survey} summarizes the placements.
For direct preference optimization methods, the operator-side versus loss-side distinction follows \Cref{app:dpo} and Definition~\ref{def:dpo_mitigation_placement}.

\paragraph{Direct evidence for a failure regime.}
\begin{itemize}
    \item \citet{singhal2024long} (data-side length balancing, BAL).
    R2$_\varepsilon$ or harmful R1$_\varepsilon$ on RLCD, R0 versus R0\textsubscript{cont} undisambiguated on Stack. 
    Balancing drives reward model accuracy on WebGPT from 61.5\% to 52.6\% against a 50\% random baseline, where \citet{bu2025beyond} measure a 3.2 point drop for the same operator on the same dataset, and costs 8.1 and 6.9 points on Stack and RLCD. 
    Within-batch reward-length correlation falls from 0.72 to $-0.13$ on WebGPT and 0.55 to $-0.09$ on Stack, but only 0.67 to 0.62 on RLCD. 
    Downstream simulated preference against standard PPO is 0.44 on RLCD and 0.57 on Stack, with the WebGPT cell unreported under the convention of that paper for runs where reward optimization failed. 
    On RLCD the targeted reliance is barely reduced while $\widehat{\Delta J} < 0$, so the scale-overshoot origin of R2 is implausible there and only target misspecification or redirection remain. 
    Off-target $\Delta_j$ unmeasured throughout.
    \item \citet{bu2025beyond} (BAL baseline, data-side length balancing following \citet{singhal2024long}).
    $\mathrm{R2}_\varepsilon$ or harmful $\mathrm{R1}_\varepsilon$, undisambiguated. 
    Data-side length balancing costs $3.2$ points of reward model preference-accuracy and the resulting policy is dispreferred by the GPT-4o judge after PPO ($0.25$ win against $0.45$ lose), a policy-distribution $\widehat{\Delta J} < 0$, with generated length falling below the SFT initialization (125 against 198 tokens). 
    Their proposed ALBM method is placed separately below. 
    \item \citet{huang2025posthoc} (post-hoc reward calibration, RC-LWR). 
    Failure evidence from our re-evaluation in \Cref{app:bon}, not from the paper. 
    All three settings reported in the paper are favorable.
    On RewardBench the calibration gains 2.93 points averaged over 33 RMs and, when used as evaluators, calibrated RMs rank LLMs in closer agreement with AlpacaEval2 and with human preferences. 
    For alignment, DPO on calibrated rewards raises the length-controlled win rate by about 9.5 points on the two Llama-3-8B configurations and about 7 points on the two gemma2-9b ones relative to DPO on the uncalibrated reward, and recovers part of the eight-benchmark drop that uncalibrated DPO causes. 
    Off-target $\Delta_j$ across $\Phi_{\mathrm{sp}}$ is not measured for the aligned policies, and the eight-benchmark panel is reported only as an average, so a per-task regression of the kind \Cref{app:bon} finds on GSM8K would not surface in it. 
    At the within-prompt $\muaud$ of \Cref{app:bon} under BoN top-1 at $N = 64$, two of five cells satisfy $\Delta J < -\varepsilon_J$ and fall in $\mathrm{R2}_\varepsilon$ or harmful $\mathrm{R1}_\varepsilon$, two satisfy $\Delta J > \varepsilon_J$, and one is within band. 
    Therefore, no cell is certifiable as R0.
    \item \citet{park2024disentangling} (R-DPO, loss-side length penalty). R2$_\varepsilon$ or harmful R1$_{\varepsilon}$ on TLDR (at $\beta = 0.05$, undetermined elsewhere). 
    With GPT-4 judging policy samples against dataset-preferred responses, both penalty strengths lower the TLDR win rate at $\beta = 0.05$ from about $0.56$ to about $0.45$ and $0.44$. 
    By our estimation, this decrease is significant, although \citet{park2024disentangling} describe it as not statistically significant. 
    Each setting is a single training run, so seed variance is not captured. 
    Changes on HH and at other $\beta$ stay within two standard errors, the headline gains compare runs at different $\beta$, and off-target $\Delta j$ and sensitivity to the length preference of the judge are unmeasured.
    An independent re-implementation on Llama-2-70B shows the same pattern, shorter outputs with a flat or lower length-controlled AlpacaEval win rate and a lower MT-Bench score under a length-unadjusted judge, again without seed variance or off-target $\Delta j$ \citep{yuan2025followinglengthconstraints}.
    \item \citet{fein2026one} (mechanistic reward shaping). $\mathrm{R2}_\varepsilon$ or harmful $\mathrm{R1}_\varepsilon$ on the Allen cell, where the AlpacaEval LC win rate falls significantly below 50\% (46.2\%, \Cref{app:bon}). 
    $\Delta J>0$ on the other four RMs at within-prompt $\muaud$. 
    Off-target $\Delta_j$ at the policy distribution unmeasured, although several other bias axes are measured at the reward model.
\end{itemize}

\paragraph{Closest to R0, but not certified.}
These methods report $\widehat{\Delta J} > 0$ at a held-out evaluation but do not report off-target $\Delta_j$ across $\Phi_{\mathrm{sp}}$. 
Since $\widehat{\Delta J}$ is a proxy that can miss a true-reward loss caused by redirection onto an unmeasured axis (\Cref{fig:smoking_gun}), the reported evidence is consistent with R0, with R0\textsubscript{cont}, and with harmful R1, and none of the three can be ruled out.
\begin{itemize}
    \item \citet{shen2023loose} (Loose Lips). 
    Product-of-Experts reward modeling with a small bias-only expert absorbing length, discarded at PPO. 
    On HH-RLHF with the Alpaca base, human, AlpacaFarm and GPT-4 judges all prefer the mitigated policy over the vanilla-reward-model PPO policy ($57.47$, $61.43$, $59.56$, ties scored as half), at shorter output length in all six configurations. 
    Off-target $\Delta_j$ unmeasured, both evaluation sets HH-derived (no uncertainty reported).
    \item \citet{chenicml2024odin} (ODIN). 
    Two-head disentanglement reduces the length correlation of the quality head to near zero on held-out preference data.
    Under PPO and ReMax, ODIN policies outscore vanilla-reward-model policies at the same response length, a length-controlled $\widehat{\Delta J} > 0$.
    But scoring is against the SFT model rather than head-to-head, the baseline fronts already include reward-clipped and length-penalized runs, and the single direct comparison is against a length penalty rather than the unmitigated policy. 
    Off-target $\Delta_j$ and $\pi_{\mathrm{ref}}$-sensitivity (single Vicuna-7B reference) unmeasured. 
    \citet{bu2025beyond} re-implement ODIN on WebGPT and find no advantage over the unmitigated policy (0.28 win against 0.31 lose), but at a fixed training step rather than at matched length (206 against 261 tokens), so the two evaluations differ in protocol as well as in preference data.
    On OpenLLaMA-3B, 
    \citet{kimarxiv2025mitigating} report a lower length-controlled AlpacaEval win rate than the unmitigated policy (12.19 against 18.97), but that ODIN policy also falls below its SFT initialization (16.97), so a failed re-implementation cannot be ruled out.
    \item \citet{bu2025beyond} (ALBM). 
    ODIN-style two-head decomposition in which the length head is reintegrated with a learned prompt-dependent weight rather than discarded, so length reliance is reduced on average but deliberately retained where the prompt calls for it, treating length as structurally relevant on part of the prompt distribution (\Cref{app:A2}). 
    After PPO from the same Vicuna-7B SFT policy at a matched training step, the GPT-4o judge prefers the ALBM policy to the vanilla-reward-model policy (0.37 win against 0.28 lose) at shorter output (228 against 261 tokens), a policy-distribution $\widehat{\Delta J} > 0$ from a single run without reported uncertainty. Off-target $\Delta_j$ and $\pi_{\mathrm{ref}}$-sensitivity unmeasured.
    \item \citet{eisenstein2024helping} (RM ensembles). 
    Not a single-axis operator, so no targeted $\phi_i$ is defined and the placement is loose. 
    Under BoN, a zero-shot PaLM-2 judge rates single-RM policies significantly worse than ensemble policies, a policy-distribution $\widehat{\Delta J} > 0$. 
    Shared hacks such as longer TLDR summaries persist under the ensemble, and off-target $\Delta_j$ relative to single-RM policies is not reported.
    \item \citet{liu2025rrm} (RRM, causal data augmentation). 
    Prompt-permutation augmentation targets all prompt-independent artifacts at once, so $\Phi_{\mathrm{sp}}$ is implicit and the placement is loose. 
    DPO on the augmented reward raises the AlpacaEval-2 length-controlled win rate from 33.46 to 52.49 at shorter output, a policy-distribution $\widehat{\Delta J} > 0$ on a single Gemma-2-9b-it lineage without reported uncertainty. 
    Length is the only feature tracked at the policy distribution and is itself a target, so off-target $\Delta_j$ is unmeasured. 
    RewardBench Reasoning falls from 94.58 to 90.62, which the authors attribute to the domain rather than to the operator.
    \item \citet{wangcrm2025beyond} (CRM, MMD-regularized causal reward model). 
    No identifiability guarantee, as the MMD penalty enforces only a necessary condition for counterfactual invariance. 
    After PPO, unconditional CRM lowers concept bias and raises sentiment accuracy over the vanilla reward model policy on Yelp and IMDB, a policy-distribution $\widehat{\Delta J} > 0$ without reported uncertainty. 
    Conditional CRM lowers IMDB accuracy by up to 21 points, but with no SFT reference and some vanilla accuracies below 50\%, a failed run cannot be ruled out. Off-target $\Delta_j$ unmeasured.
    \item \citet{zhao2025bias} (FiMi-RM). 
    Nonlinear length-reward fit decorrelated from the reward model. 
    Length-controlled AlpacaEval win rate rises at shorter length in all seven reported settings across BoN, DPO, SimPO, and GRPO, and MT-Bench and IFEval rise after DPO, a policy-distribution $\widehat{\Delta J} > 0$ without reported uncertainty. 
    Preference accuracy changes by less than one point. 
    Off-target $\Delta_j$ unmeasured.
    \item \citet{cai2026disentangling} (Rc-BT). PPO $\Delta J > 0$ versus baseline. An RM-level format-bias ablation is reported but not propagated to policy outcomes.
    \item \citet{kimarxiv2025mitigating} (CDA). 
    Counterfactual preference-pair augmentation for reward model training, followed by PPO. 
    AlpacaEval length-controlled win rate rises from 18.97 to 37.18 over the unmitigated reward model policy at shorter output length, a policy-distribution $\widehat{\Delta J} > 0$ on a single OpenLLaMA-3B lineage without reported uncertainty. 
    Its reward-level length-controlled accuracy scores only pairs in which the shorter response was preferred, so it cannot detect overshoot.
    Off-target $\Delta_j$ unmeasured.
    \item \citet{srivastava2026robust} (CROME). 
    Spurious-unaware reward model augmentation from oracle-LLM causal rubrics, with no designer-specified $\Phi_{\mathrm{sp}}$. 
    Best-of-$N$ on RewardBench prompts, GSM8K, and WildGuardTest and three-round on-policy DPO give a policy-distribution $\widehat{\Delta J} > 0$, including an AlpacaEval 2.0 LC win rate of 59.9 against 52.4 for the unmitigated reward model, without reported uncertainty. 
    All policy-side results use a single Gemma-2-9B-IT policy. 
    The only behaviors beyond task metrics measured at the policy distribution are average length under DPO (3445 against 3572) and benign-prompt refusal rate under Best-of-$N$, which is 0.1 to 0.7 points higher than for the unmitigated reward model at every $N$. 
    Robustness on reWordBench and the length and sycophancy checks on augmented pairs are audit-side. 
    A second oracle LLM is compared only through RewardBench, reWordBench, and rubric overlap, so R4 sensitivity is not measured.
    \item \citet{song2025causalrewardadjustmentmitigating} (SAE causal adjustment). $\Delta J > 0$ across every evaluated combination of policy model, dataset, and reward models. The $\Phi_{\mathrm{sp}}$ panel is whatever the SAE learns, not a designer-specified set.
    \item \citet{feng2025c2po} (C2PO). A reference-free DPO-variant loss with length-normalized implicit reward, targeting stereotypical and structural shortcuts. 
    Reports accuracy or judge-agreement gains on a bias benchmark panel alongside preserved MMLU and GSM8K, with the operator isolated by adding it to four baseline losses. 
    The panel entries are classification accuracies of the trained model rather than feature expectations on its generations, so no off-target $\Delta_j$ is reported and $\pi_{\mathrm{ref}}$-sensitivity is unmeasured.
    \item \citet{umer2026generalpreferencereinforcementlearning} (GPRL). 
    Multi-dimensional preference reward in GRPO-style online RL with a drift controller, and a sufficient condition under which the aggregate advantage disfavors single-axis exploitation. 
    $\Delta J > 0$ over GRPO with a scalar BT reward model on four benchmarks at a single base policy (Llama-3-8B-Instruct), without reported uncertainty and with the baseline evaluated after its stated peak. 
    Gains on three further base policies are reported against the base models, not against GRPO with a scalar reward model. 
    $\Phi_{\mathrm{sp}}$ implicit in the underlying preference model, length the only reported off-target feature, off-panel substitution unaddressed.
\end{itemize}

\paragraph{Undetermined direct preference optimization methods (loss-side).}
The constructive single-axis guarantees of \Cref{app:dpo} do not apply. 
These methods are classifiable only through the policy pair, and the surveyed evidence sits inside the impossibility class of Theorem~\ref{thm:audit-insufficiency}.
\begin{itemize}
    \item \citet{meng2024simpo} (SimPO). Reference-free with length-normalized reward $\frac{\beta}{|y|}\log\pi_\theta$, a length-dependent reweighting interacting with the cardinal/ordinal blindspot of \Cref{cor:functional-blindspot} rather than its scalar-rescaling case.
    \item \citet{lu2024eliminating} (SamPO). Algorithmic length dependence in DPO addressed analytically. Empirical validation does not isolate $\Delta J$ from length reduction.
    \item \citet{lilmpo2025length} (LMPO). Reference-free with a policy-dependent margin term. Win-rate gains over DPO compare against baselines copied from \citet{meng2024simpo}, whose SimPO results LMPO reports it could not reproduce. Arena-Hard is not length-controlled, and the largest gain coincides with the longest reported outputs. Off-target $\Delta_j$ unmeasured.
\end{itemize}

\paragraph{Undetermined RM-side methods without a policy-side $\widehat{\Delta J}$.}
\begin{itemize}
    \item \citet{fu2025reward} (PAR). Single Gemma-2B lineage. Targets training stability rather than length bias directly.
    \item \citet{wang2024interpretablepreferences} (ArmoRM). 
    Each reward objective is projected as $r_i' \leftarrow r_i - \lambda_i r_{\mathrm{verbose}}$ with $\lambda_i$ set so that the Spearman correlation between $r_i'$ and a learned verbosity objective vanishes on UltraFeedback.
    Their method is a rank-correlation counterpart of \Cref{def:single_axis_mitigation} acting on a predicted verbosity rating rather than on length. 
    The zero-correlation condition is imposed per objective before a prompt-conditioned gating recombines them, so the final scalar reward is not itself decorrelated. 
    Certified on RewardBench accuracy alone, so all reported evidence lies in $\mathcal{B}$. 
    No $\lambda_i$ is reported and no policy is trained, so $\Delta J$ and off-target $\Delta_j$ are unmeasured.
    \item \citet{wanghelpsteer2024helpsteer} (HelpSteer). 
    Data-side multi-attribute annotation motivated by reward models learning length as a proxy for helpfulness. 
    Validation (MT-Bench, TruthfulQA, response length, human win rates) is of attribute-conditioned SFT against RLHF and DPO baselines, not of a reward model or mitigation operator against its unmitigated counterpart, so no $\Delta_j$ or $\Delta J$ in the sense of our definition is reported.
    \item \citet{papadatos2024linear} (sycophancy probe penalty). 
    A linear probe on reward model activations scores sycophancy and is subtracted from the reward as $\hat R = R - \lambda S$ with $\lambda$ fixed, such that the penalty spread is 0.75 of the reward.
    This operator has the form of the canonical $M_i$ from \Cref{def:single_axis_mitigation} with a fixed coefficient in place of the estimated reliance and a probe direction in place of a surface feature.
    Validation is at the policy distribution but on the targeted axis only. 
    Under BoN up to $N = 32$ on OpenChat-3.5, GPT-4-judged feedback sycophancy rises with $N$ under UltraRM and falls under the surrogate, with 95\% bands on 300 poems, while under Starling-RM the base reward already lowers it and the surrogate lowers it further, so the marginal effect of the operator is reward-model dependent. 
    No quality or true-reward measure of the selected responses is reported, so $\widehat{\Delta J}$ is unmeasured, and no off-target feature or second reference policy is tracked. 
    The reward-level reliance is not reported either.
    \item \citet{ng2025debiasing} (debiasing with guarantees). Identifiability theorem is conditional on assumed latent-variable DAG, which preference data does not pin down. Empirical validation small or synthetic.
    \item \citet{lidir2025eliminating} (DIR). Information-bottleneck framing handles non-linearity in principle. Empirical evidence conflates bias mitigation with downstream task improvement.
    \item \citet{liu2026orpo} (robust correlated proxies). Improves a worst-case proxy bound. The worst-case proxy is not the true reward, and the Linear Worst$^\ast$ probe shows the linear variant degrades on unseen features.
    \item \citet{ye2025prism} (PRISM). Near-zero correlation on three measured shortcuts. Off-panel substitution unmeasured.
    \item \citet{chen2026learning} (OPRM with RgFT). Calibration framing via ECE is the closest engagement with the cardinal versus ordinal distinction in the surveyed set, but the protocol does not exploit it for substitution detection.
    \item \citet{wang2026patterningpracticedebiasingreward} (patterning, preference-data reweighting). 
    Data-side, so no targeted $\phi_i$ is defined and the placement is loose. 
    Skywork pairs are reweighted by their estimated effect on RM-Bench split losses and a Gemma 2 9B reward model is retrained, so the audit benchmark defines the target and scores the result, and all reported evidence lies in $\mathcal{B}$. 
    Safety-response Easy falls $44.1 \pm 4.0$ pp, traced by the authors to downweighted pairs that prefer a long compliant response over a short refusal, so reliance shifts toward short refusals on benign prompts, consistent with target misspecification (\Cref{fact:length_penalties}) or with redirection onto a refusal axis, undisambiguated. 
    No policy is trained or selected, so $\Delta J$ and off-target $\Delta_j$ are unmeasured.
\end{itemize}

\paragraph{Takeaway.}
The five works with direct failure-regime evidence and the twelve works closest to R0 together cover 17 of 30 surveyed works. 
The remaining thirteen are undetermined under R0, R0\textsubscript{cont}, R1, and R2 from their reported evidence alone. 
Closing this gap requires either policy-induced $(\Delta_j, \Delta J)$ measurement on the methods side or benchmark infrastructure that captures it on the evaluation side, as set out in \Cref{app:a8-takeaways}.
Overall, the most common recommendation would be to measure off-target feature changes $\Delta_j$, especially changes in verbosity and confidence (and ideally sycophancy, too).

\newpage
\section*{NeurIPS Paper Checklist}



\begin{enumerate}

\item {\bf Claims}
    \item[] Question: Do the main claims made in the abstract and introduction accurately reflect the paper's contributions and scope?
    \item[] Answer: \answerYes{} 
    \item[] Justification: Yes, all our listed claims are backed up either as theoretical contribution (\Cref{sec:bias_sub} and \Cref{app:more_math}) or empirical (\Cref{sec:4_evidence,app:experiments,app:survey,app:regimes}).
    \item[] Guidelines:
    \begin{itemize}
        \item The answer \answerNA{} means that the abstract and introduction do not include the claims made in the paper.
        \item The abstract and/or introduction should clearly state the claims made, including the contributions made in the paper and important assumptions and limitations. A \answerNo{} or \answerNA{} answer to this question will not be perceived well by the reviewers. 
        \item The claims made should match theoretical and experimental results, and reflect how much the results can be expected to generalize to other settings. 
        \item It is fine to include aspirational goals as motivation as long as it is clear that these goals are not attained by the paper. 
    \end{itemize}

\item {\bf Limitations}
    \item[] Question: Does the paper discuss the limitations of the work performed by the authors?
    \item[] Answer: \answerYes{} 
    \item[] Justification: Yes, we discuss general limitations in \Cref{sec:6_discussion}. For theoretical contributions, we discuss our assumptions in \Cref{app:more_math}, for experiments, we provide additional limitations in \Cref{app:experiments}.
    \item[] Guidelines:
    \begin{itemize}
        \item The answer \answerNA{} means that the paper has no limitation while the answer \answerNo{} means that the paper has limitations, but those are not discussed in the paper. 
        \item The authors are encouraged to create a separate ``Limitations'' section in their paper.
        \item The paper should point out any strong assumptions and how robust the results are to violations of these assumptions (e.g., independence assumptions, noiseless settings, model well-specification, asymptotic approximations only holding locally). The authors should reflect on how these assumptions might be violated in practice and what the implications would be.
        \item The authors should reflect on the scope of the claims made, e.g., if the approach was only tested on a few datasets or with a few runs. In general, empirical results often depend on implicit assumptions, which should be articulated.
        \item The authors should reflect on the factors that influence the performance of the approach. For example, a facial recognition algorithm may perform poorly when image resolution is low or images are taken in low lighting. Or a speech-to-text system might not be used reliably to provide closed captions for online lectures because it fails to handle technical jargon.
        \item The authors should discuss the computational efficiency of the proposed algorithms and how they scale with dataset size.
        \item If applicable, the authors should discuss possible limitations of their approach to address problems of privacy and fairness.
        \item While the authors might fear that complete honesty about limitations might be used by reviewers as grounds for rejection, a worse outcome might be that reviewers discover limitations that aren't acknowledged in the paper. The authors should use their best judgment and recognize that individual actions in favor of transparency play an important role in developing norms that preserve the integrity of the community. Reviewers will be specifically instructed to not penalize honesty concerning limitations.
    \end{itemize}

\item {\bf Theory assumptions and proofs}
    \item[] Question: For each theoretical result, does the paper provide the full set of assumptions and a complete (and correct) proof?
    \item[] Answer: \answerYes{} 
    \item[] Justification: Yes, we provide all assumptions and proofs either in \Cref{sec:bias_sub} or \Cref{app:more_math}.
    \item[] Guidelines:
    \begin{itemize}
        \item The answer \answerNA{} means that the paper does not include theoretical results. 
        \item All the theorems, formulas, and proofs in the paper should be numbered and cross-referenced.
        \item All assumptions should be clearly stated or referenced in the statement of any theorems.
        \item The proofs can either appear in the main paper or the supplemental material, but if they appear in the supplemental material, the authors are encouraged to provide a short proof sketch to provide intuition. 
        \item Inversely, any informal proof provided in the core of the paper should be complemented by formal proofs provided in appendix or supplemental material.
        \item Theorems and Lemmas that the proof relies upon should be properly referenced. 
    \end{itemize}

    \item {\bf Experimental result reproducibility}
    \item[] Question: Does the paper fully disclose all the information needed to reproduce the main experimental results of the paper to the extent that it affects the main claims and/or conclusions of the paper (regardless of whether the code and data are provided or not)?
    \item[] Answer: \answerYes{} 
    \item[] Justification: Yes, we state all experiment details in \Cref{app:experiments} and all code for full replication are available on Github (see link in \Cref{sec:intro}).
    \item[] Guidelines:
    \begin{itemize}
        \item The answer \answerNA{} means that the paper does not include experiments.
        \item If the paper includes experiments, a \answerNo{} answer to this question will not be perceived well by the reviewers: Making the paper reproducible is important, regardless of whether the code and data are provided or not.
        \item If the contribution is a dataset and\slash or model, the authors should describe the steps taken to make their results reproducible or verifiable. 
        \item Depending on the contribution, reproducibility can be accomplished in various ways. For example, if the contribution is a novel architecture, describing the architecture fully might suffice, or if the contribution is a specific model and empirical evaluation, it may be necessary to either make it possible for others to replicate the model with the same dataset, or provide access to the model. In general. releasing code and data is often one good way to accomplish this, but reproducibility can also be provided via detailed instructions for how to replicate the results, access to a hosted model (e.g., in the case of a large language model), releasing of a model checkpoint, or other means that are appropriate to the research performed.
        \item While NeurIPS does not require releasing code, the conference does require all submissions to provide some reasonable avenue for reproducibility, which may depend on the nature of the contribution. For example
        \begin{enumerate}
            \item If the contribution is primarily a new algorithm, the paper should make it clear how to reproduce that algorithm.
            \item If the contribution is primarily a new model architecture, the paper should describe the architecture clearly and fully.
            \item If the contribution is a new model (e.g., a large language model), then there should either be a way to access this model for reproducing the results or a way to reproduce the model (e.g., with an open-source dataset or instructions for how to construct the dataset).
            \item We recognize that reproducibility may be tricky in some cases, in which case authors are welcome to describe the particular way they provide for reproducibility. In the case of closed-source models, it may be that access to the model is limited in some way (e.g., to registered users), but it should be possible for other researchers to have some path to reproducing or verifying the results.
        \end{enumerate}
    \end{itemize}

\item {\bf Open access to data and code}
    \item[] Question: Does the paper provide open access to the data and code, with sufficient instructions to faithfully reproduce the main experimental results, as described in supplemental material?
    \item[] Answer: \answerYes{} 
    \item[] Justification: The GitHub repository contains the relevant code to run the experiments, do the analysis, and generate the plots in this paper.
    \item[] Guidelines:
    \begin{itemize}
        \item The answer \answerNA{} means that paper does not include experiments requiring code.
        \item Please see the NeurIPS code and data submission guidelines (\url{https://neurips.cc/public/guides/CodeSubmissionPolicy}) for more details.
        \item While we encourage the release of code and data, we understand that this might not be possible, so \answerNo{} is an acceptable answer. Papers cannot be rejected simply for not including code, unless this is central to the contribution (e.g., for a new open-source benchmark).
        \item The instructions should contain the exact command and environment needed to run to reproduce the results. See the NeurIPS code and data submission guidelines (\url{https://neurips.cc/public/guides/CodeSubmissionPolicy}) for more details.
        \item The authors should provide instructions on data access and preparation, including how to access the raw data, preprocessed data, intermediate data, and generated data, etc.
        \item The authors should provide scripts to reproduce all experimental results for the new proposed method and baselines. If only a subset of experiments are reproducible, they should state which ones are omitted from the script and why.
        \item At submission time, to preserve anonymity, the authors should release anonymized versions (if applicable).
        \item Providing as much information as possible in supplemental material (appended to the paper) is recommended, but including URLs to data and code is permitted.
    \end{itemize}

\item {\bf Experimental setting/details}
    \item[] Question: Does the paper specify all the training and test details (e.g., data splits, hyperparameters, how they were chosen, type of optimizer) necessary to understand the results?
    \item[] Answer: \answerYes{} 
    \item[] Justification: Yes, see \Cref{app:experiments}.
    \item[] Guidelines:
    \begin{itemize}
        \item The answer \answerNA{} means that the paper does not include experiments.
        \item The experimental setting should be presented in the core of the paper to a level of detail that is necessary to appreciate the results and make sense of them.
        \item The full details can be provided either with the code, in appendix, or as supplemental material.
    \end{itemize}

\item {\bf Experiment statistical significance}
    \item[] Question: Does the paper report error bars suitably and correctly defined or other appropriate information about the statistical significance of the experiments?
    \item[] Answer: \answerYes{} 
    \item[] Justification: Yes, see \Cref{app:experiments}. We report 95\% confidence intervals and p-values where applicable.
    \item[] Guidelines:
    \begin{itemize}
        \item The answer \answerNA{} means that the paper does not include experiments.
        \item The authors should answer \answerYes{} if the results are accompanied by error bars, confidence intervals, or statistical significance tests, at least for the experiments that support the main claims of the paper.
        \item The factors of variability that the error bars are capturing should be clearly stated (for example, train/test split, initialization, random drawing of some parameter, or overall run with given experimental conditions).
        \item The method for calculating the error bars should be explained (closed form formula, call to a library function, bootstrap, etc.)
        \item The assumptions made should be given (e.g., Normally distributed errors).
        \item It should be clear whether the error bar is the standard deviation or the standard error of the mean.
        \item It is OK to report 1-sigma error bars, but one should state it. The authors should preferably report a 2-sigma error bar than state that they have a 96\% CI, if the hypothesis of Normality of errors is not verified.
        \item For asymmetric distributions, the authors should be careful not to show in tables or figures symmetric error bars that would yield results that are out of range (e.g., negative error rates).
        \item If error bars are reported in tables or plots, the authors should explain in the text how they were calculated and reference the corresponding figures or tables in the text.
    \end{itemize}

\item {\bf Experiments compute resources}
    \item[] Question: For each experiment, does the paper provide sufficient information on the computer resources (type of compute workers, memory, time of execution) needed to reproduce the experiments?
    \item[] Answer: \answerYes{} 
    \item[] Justification: Yes, see \Cref{app:experiments}.
    \item[] Guidelines:
    \begin{itemize}
        \item The answer \answerNA{} means that the paper does not include experiments.
        \item The paper should indicate the type of compute workers CPU or GPU, internal cluster, or cloud provider, including relevant memory and storage.
        \item The paper should provide the amount of compute required for each of the individual experimental runs as well as estimate the total compute. 
        \item The paper should disclose whether the full research project required more compute than the experiments reported in the paper (e.g., preliminary or failed experiments that didn't make it into the paper). 
    \end{itemize}
    
\item {\bf Code of ethics}
    \item[] Question: Does the research conducted in the paper conform, in every respect, with the NeurIPS Code of Ethics \url{https://neurips.cc/public/EthicsGuidelines}?
    \item[] Answer: \answerYes{} 
    \item[] Justification: Yes, we follow the NeurIPS Code of Ethics guidelines.
    \item[] Guidelines:
    \begin{itemize}
        \item The answer \answerNA{} means that the authors have not reviewed the NeurIPS Code of Ethics.
        \item If the authors answer \answerNo, they should explain the special circumstances that require a deviation from the Code of Ethics.
        \item The authors should make sure to preserve anonymity (e.g., if there is a special consideration due to laws or regulations in their jurisdiction).
    \end{itemize}

\item {\bf Broader impacts}
    \item[] Question: Does the paper discuss both potential positive societal impacts and negative societal impacts of the work performed?
    \item[] Answer: \answerYes{} 
    \item[] Justification: Yes, see \Cref{sec:6_discussion}.
    \item[] Guidelines:
    \begin{itemize}
        \item The answer \answerNA{} means that there is no societal impact of the work performed.
        \item If the authors answer \answerNA{} or \answerNo, they should explain why their work has no societal impact or why the paper does not address societal impact.
        \item Examples of negative societal impacts include potential malicious or unintended uses (e.g., disinformation, generating fake profiles, surveillance), fairness considerations (e.g., deployment of technologies that could make decisions that unfairly impact specific groups), privacy considerations, and security considerations.
        \item The conference expects that many papers will be foundational research and not tied to particular applications, let alone deployments. However, if there is a direct path to any negative applications, the authors should point it out. For example, it is legitimate to point out that an improvement in the quality of generative models could be used to generate Deepfakes for disinformation. On the other hand, it is not needed to point out that a generic algorithm for optimizing neural networks could enable people to train models that generate Deepfakes faster.
        \item The authors should consider possible harms that could arise when the technology is being used as intended and functioning correctly, harms that could arise when the technology is being used as intended but gives incorrect results, and harms following from (intentional or unintentional) misuse of the technology.
        \item If there are negative societal impacts, the authors could also discuss possible mitigation strategies (e.g., gated release of models, providing defenses in addition to attacks, mechanisms for monitoring misuse, mechanisms to monitor how a system learns from feedback over time, improving the efficiency and accessibility of ML).
    \end{itemize}
    
\item {\bf Safeguards}
    \item[] Question: Does the paper describe safeguards that have been put in place for responsible release of data or models that have a high risk for misuse (e.g., pre-trained language models, image generators, or scraped datasets)?
    \item[] Answer: \answerNA{} 
    \item[] Justification: Our work is trying to avoid safety-failures and is unambiguously doing the opposite through enabling robuster evals of bias mitigation methods.
    \item[] Guidelines:
    \begin{itemize}
        \item The answer \answerNA{} means that the paper poses no such risks.
        \item Released models that have a high risk for misuse or dual-use should be released with necessary safeguards to allow for controlled use of the model, for example by requiring that users adhere to usage guidelines or restrictions to access the model or implementing safety filters. 
        \item Datasets that have been scraped from the Internet could pose safety risks. The authors should describe how they avoided releasing unsafe images.
        \item We recognize that providing effective safeguards is challenging, and many papers do not require this, but we encourage authors to take this into account and make a best faith effort.
    \end{itemize}

\item {\bf Licenses for existing assets}
    \item[] Question: Are the creators or original owners of assets (e.g., code, data, models), used in the paper, properly credited and are the license and terms of use explicitly mentioned and properly respected?
    \item[] Answer: \answerYes{} 
    \item[] Justification: Yes, we properly cite any model or data set used in \Cref{app:experiments}.
    \item[] Guidelines:
    \begin{itemize}
        \item The answer \answerNA{} means that the paper does not use existing assets.
        \item The authors should cite the original paper that produced the code package or dataset.
        \item The authors should state which version of the asset is used and, if possible, include a URL.
        \item The name of the license (e.g., CC-BY 4.0) should be included for each asset.
        \item For scraped data from a particular source (e.g., website), the copyright and terms of service of that source should be provided.
        \item If assets are released, the license, copyright information, and terms of use in the package should be provided. For popular datasets, \url{paperswithcode.com/datasets} has curated licenses for some datasets. Their licensing guide can help determine the license of a dataset.
        \item For existing datasets that are re-packaged, both the original license and the license of the derived asset (if it has changed) should be provided.
        \item If this information is not available online, the authors are encouraged to reach out to the asset's creators.
    \end{itemize}

\item {\bf New assets}
    \item[] Question: Are new assets introduced in the paper well documented and is the documentation provided alongside the assets?
    \item[] Answer: \answerNA{} 
    \item[] Justification: Our main contributions are theoretical and we do not release new assets.
    \item[] Guidelines:
    \begin{itemize}
        \item The answer \answerNA{} means that the paper does not release new assets.
        \item Researchers should communicate the details of the dataset\slash code\slash model as part of their submissions via structured templates. This includes details about training, license, limitations, etc. 
        \item The paper should discuss whether and how consent was obtained from people whose asset is used.
        \item At submission time, remember to anonymize your assets (if applicable). You can either create an anonymized URL or include an anonymized zip file.
    \end{itemize}

\item {\bf Crowdsourcing and research with human subjects}
    \item[] Question: For crowdsourcing experiments and research with human subjects, does the paper include the full text of instructions given to participants and screenshots, if applicable, as well as details about compensation (if any)? 
    \item[] Answer: \answerNA{} 
    \item[] Justification: Our main contributions are theoretical and supporting experiments are without human participants.
    \item[] Guidelines:
    \begin{itemize}
        \item The answer \answerNA{} means that the paper does not involve crowdsourcing nor research with human subjects.
        \item Including this information in the supplemental material is fine, but if the main contribution of the paper involves human subjects, then as much detail as possible should be included in the main paper. 
        \item According to the NeurIPS Code of Ethics, workers involved in data collection, curation, or other labor should be paid at least the minimum wage in the country of the data collector. 
    \end{itemize}

\item {\bf Institutional review board (IRB) approvals or equivalent for research with human subjects}
    \item[] Question: Does the paper describe potential risks incurred by study participants, whether such risks were disclosed to the subjects, and whether Institutional Review Board (IRB) approvals (or an equivalent approval/review based on the requirements of your country or institution) were obtained?
    \item[] Answer: \answerNA{} 
    \item[] Justification: Our main contributions are theoretical and supporting experiments are without human participants.
    \item[] Guidelines:
    \begin{itemize}
        \item The answer \answerNA{} means that the paper does not involve crowdsourcing nor research with human subjects.
        \item Depending on the country in which research is conducted, IRB approval (or equivalent) may be required for any human subjects research. If you obtained IRB approval, you should clearly state this in the paper. 
        \item We recognize that the procedures for this may vary significantly between institutions and locations, and we expect authors to adhere to the NeurIPS Code of Ethics and the guidelines for their institution. 
        \item For initial submissions, do not include any information that would break anonymity (if applicable), such as the institution conducting the review.
    \end{itemize}

\item {\bf Declaration of LLM usage}
    \item[] Question: Does the paper describe the usage of LLMs if it is an important, original, or non-standard component of the core methods in this research? Note that if the LLM is used only for writing, editing, or formatting purposes and does \emph{not} impact the core methodology, scientific rigor, or originality of the research, declaration is not required.
    \item[] Answer: \answerYes{} 
    \item[] Justification: All research input and scientific judgment are from the authors, including the framework, theorems, and experimental design. We used LLM as active component in our experiments as LLM-judges, LM-based reward models, and as RL policy, and the usage is described in \Cref{app:experiments}. In addition, we used LLM support for writing (e.g., turning sentence fragments and notes into draft sentences to edit ourselves), formatting (e.g., latex tables), and editing (e.g., grammar, writing clarity) throughout. We also used commercial LLM tools to extend the scope of our literature review to include papers we would have otherwise missed (sources and claims based on paper content all manually verified). We used LLM support to help with code generation for experiments in \Cref{app:experiments}, with all code reviewed and run by the authors, and sanity check proofs for gaps or errors, with all theoretical results derived and written by the authors.
    
    \item[] Guidelines:
    \begin{itemize}
        \item The answer \answerNA{} means that the core method development in this research does not involve LLMs as any important, original, or non-standard components.
        \item Please refer to our LLM policy in the NeurIPS handbook for what should or should not be described.
    \end{itemize}

\end{enumerate}

\end{document}